\documentclass{article}

\usepackage{iclr2021_conference,times}
\iclrfinalcopy %

\setcitestyle{numbers,square}

\usepackage[utf8x]{inputenc}
\usepackage[T1]{fontenc}    %
\usepackage{hyperref}       %
\usepackage{url}            %
\usepackage[inline]{enumitem}
\usepackage{siunitx}
\usepackage[symbol]{footmisc}
\usepackage{stringstrings}  %

\usepackage{chngcntr}
\usepackage{microtype}
\usepackage{graphicx}
\usepackage{booktabs} %
\usepackage{multirow}
\usepackage{xcolor}
\usepackage{xintexpr}
\usepackage{subcaption}
\usepackage{placeins}

\newcommand\Bstrut{\rule[-1.1ex]{0pt}{0pt}}
\newcommand\Tstrut{\rule{0pt}{2.2ex}}

\newcommand{\compactheatmap}[3]{
	\begin{subfigure}{1.90in}
		\includegraphics{figs/hand_designed/epic/#1/bootstrap_#2_all}
		\captionsetup{margin={3.3em,0em}}
		\subcaption{EPIC#3}
	\end{subfigure}%
	\hspace{0.045in}%
	\begin{subfigure}{1.40in}
		\includegraphics{figs/hand_designed/npec/#1/bootstrap_#2_all}
		\subcaption{NPEC#3}
	\end{subfigure}%
	\hspace{0.045in}%
	\begin{subfigure}{2.10in}
		\includegraphics{figs/hand_designed/erc/#1/#2_all}
		\captionsetup{margin={-5.4em,0em}}
		\subcaption{ERC#3}
	\end{subfigure}%
}

\usepackage{pifont}
\definecolor{darkred}{HTML}{990000}
\definecolor{darkgreen}{HTML}{38761d}
\newcommand{\yessymbol}{{\color{darkgreen} \ding{51}}}
\newcommand{\nosymbol}{{\color{darkred} \ding{55}}}

\usepackage{amsmath}
\usepackage{amsfonts}
\usepackage{amssymb}
\usepackage{amsthm}
\usepackage{bbm}
\usepackage{thmtools}
\usepackage{nicefrac}       %

\newtheorem{defn}{Definition}[section]
\newtheorem{theorem}[defn]{Theorem}
\newtheorem{lemma}[defn]{Lemma}

\newtheorem{prop}[defn]{Proposition}
\declaretheoremstyle[%
	spaceabove=0pt,%
	spacebelow=6pt,%
	headfont=\normalfont\itshape,%
	postheadspace=1em,%
	qed=\qedsymbol%
]{compact}
\declaretheorem[name={Proof},style=compact,unnumbered]{compact-prf}
\usepackage{environ}
\NewEnviron{optional-prf}{}

\usepackage{figsymbols}
\usepackage{mlsymbols}

\usepackage[disable]{todonotes}
\newcommand{\ag}[1]{\todo{AG: #1}}

\newcommand{\envheatmapcaption}[1]{Approximate distances between hand-designed reward functions in \env{#1}.
	The \transitiondatasetname{} $\transitiondataset{}$ is sampled from rollouts of a policy $\randompolicy{}$ taking actions uniformly at random.}
\newcommand{\pointmasscaption}{\textbf{Key}: The agent has position $x \in \mathbb{R}$, velocity $\dot{x} \in \mathbb{R}$ and can accelerate $\ddot{x} \in \mathbb{R}$, producing future position $x' \in \mathbb{R}$.
	\controlpenalty{}~quadratic penalty on control $\ddot{x}^2$, \nocontrolpenalty{}~no control penalty.
	\sparse{}~is $\mathrm{Sparse}(x) = \mathbbm{1}[|x| < 0.05]$, \dense{}~is shaped $\mathrm{Dense}(x, x') = \mathrm{Sparse}(x) + |x'| - |x|$, while \magnitude{}~is $\mathrm{Magnitude}(x) = -|x|$.}
\newcommand{\cicaption}{\textbf{Confidence Interval (CI)}: 95\% CI computed by bootstraping over \num{10000} samples.}
\newcommand{\checkpointscaption}{Distance of reward checkpoints from the ground-truth in \pointmaze{} and policy regret for reward checkpoints during reward model training. Each point evaluates a reward function checkpoint from a single seed. \canonicalizeddistanceabbrevonly{}, \nearestpointabbrevonly{} and \returncorrelationabbrevonly{} distance use the \mixturepolicy{}ture distribution. Regret is computed by running RL on the checkpoint. The shaded region represents the bootstrapped 95\% confidence interval for the distance or regret at that checkpoint, calculated following Section~\ref{sec:supp:approx:ci}.}

\ificlrfinal
\newcommand{\projectsource}{\url{https://github.com/HumanCompatibleAI/evaluating-rewards}}
\else
\newcommand{\projectsource}{--- double blind: supplementary material ---}
\fi

\title{Quantifying Differences in Reward Functions}

\author{
	Adam Gleave\textsuperscript{1,2}\quad%
    Michael Dennis\textsuperscript{1}\quad%
    Shane Legg\textsuperscript{2}\quad%
    Stuart Russell\textsuperscript{1}\quad%
    Jan Leike\textsuperscript{3}\thanks{Work partially conducted while at DeepMind.}
}

\begin{document}

\begin{center}
\maketitle
\ificlrfinal
\vspace{-2.8em}
\textsuperscript{1}UC Berkeley \quad \textsuperscript{2}DeepMind \quad \textsuperscript{3}OpenAI \\
\texttt{gleave@berkeley.edu}
\fi
\end{center}

\begin{abstract}
For many tasks, the reward function is inaccessible to introspection or too complex to be specified procedurally, and must instead be learned from user data.
Prior work has evaluated learned reward functions by evaluating policies optimized for the learned reward.
However, this method cannot distinguish between the learned reward function failing to reflect user preferences and the policy optimization process failing to optimize the learned reward.
Moreover, this method can only tell us about behavior in the evaluation environment, but the reward may incentivize very different behavior in even a slightly different deployment environment.
To address these problems, we introduce the \emph{\canonicalizeddistancelongnameonly{} (\canonicalizeddistanceabbrevonly{})} distance to quantify the difference between two reward functions directly, without a policy optimization step.
We prove \canonicalizeddistanceabbrevonly{} is invariant on an equivalence class of reward functions that always induce the same optimal policy.
Furthermore, we find \canonicalizeddistanceabbrevonly{} can be efficiently approximated and is more robust than baselines to the choice of \transitiondatasetname{}.
Finally, we show that \canonicalizeddistanceabbrevonly{} distance bounds the regret of optimal policies even under different transition dynamics, and we confirm empirically that it predicts policy training success.
Our source code is available at \projectsource{}.
\end{abstract}

\section{Introduction}

Reinforcement learning~(RL) has reached or surpassed human performance in many domains with clearly defined reward functions, such as games~\citep{silver:2016,openai:2018,vinyals:2019} and narrowly scoped robotic manipulation tasks~\citep{openai:2019}.
Unfortunately, the reward functions for most real-world tasks are difficult or impossible to specify procedurally.
Even a task as simple as peg insertion from pixels has a non-trivial reward function that must usually be learned~\citep[IV.A]{vecerik:2019}.
Tasks involving human interaction can have far more complex reward functions that users may not even be able to introspect on.
These challenges have inspired work on learning a reward function, whether from demonstrations~\citep{ng:2000,ramachandran:2007,ziebart:2008,fu:2018,bahdanau:2019}, preferences~\citep{akrour:2011,wilson:2012,christiano:2017,sadigh:2017,ziegler:2019} or both~\citep{ibarz:2018,brown:2019}.

Prior work has usually evaluated the learned reward function $\learnedreward{}$ using the ``rollout method'': training a policy $\pi_{\learnedreward{}}$ to optimize $\learnedreward{}$ and then examining rollouts from $\pi_{\learnedreward{}}$.
Unfortunately, using RL to compute $\pi_{\learnedreward{}}$ is often computationally expensive.
Furthermore, the method produces \emph{false negatives} when the reward $\learnedreward{}$ matches user preferences but the RL algorithm 
fails to optimize with respect to $\learnedreward{}$.

The rollout method also produces \emph{false positives}.
Of the many reward functions that induce the desired rollout in a given environment, only a small subset align with the user's preferences.
For example, suppose the agent can reach states $\{A,B,C\}$.
If the user prefers $A > B > C$, but the agent instead learns $A > C > B$, the agent will still go to the correct state $A$.
However, if the initial state distribution or transition dynamics change, misaligned rewards may induce undesirable policies.
For example, if $A$ is no longer reachable at deployment, the previously reliable agent would misbehave by going to the least-favoured state $C$.

We propose instead to evaluate learned rewards via their distance from other reward functions, and summarize our desiderata for reward function distances in Table~\ref{table:desiderata}.
For benchmarks, it is usually possible to directly compare a learned reward $\learnedreward{}$ to the true reward function $\gtreward{}$.
Alternatively, benchmark creators can train a ``proxy'' reward function from a large human data set.
This proxy can then be used as a stand-in for the true reward $\gtreward$ when evaluating algorithms trained on a different or smaller data set.

Comparison with a ground-truth reward function is rarely possible outside of benchmarks.
However, even in this challenging case, comparisons can at least be used to cluster reward models trained using different techniques or data.
Larger clusters are more likely to be correct, since multiple methods arrived at a similar result.
Moreover, our regret bound (Theorem~\ref{thm:epic-regret-bound-discrete}) suggests we could use interpretability methods~\cite{michaud:2020} on one model and get some guarantees for models in the same cluster.

\begin{table}
    \caption{Summary of the desiderata satisfied by each reward function distance. \textbf{Key} -- the distance is:
        a \emph{pseudometric} (section~\ref{sec:background});
        \emph{invariant}
        to potential shaping~\citep{ng:1999} and positive rescaling (section~\ref{sec:background});
        a computationally \emph{efficient} approximation achieving low error (section~\ref{sec:experiments:hand-designed});
        \emph{robust} to the choice of \transitiondatasetname{} (section~\ref{sec:experiments:visitation-sensitivity});
        and \emph{predictive} of the similarity of the trained policies (section~\ref{sec:experiments:ground-truth}).}
    \label{table:desiderata}
    \centering
    \begin{tabular}{@{}lccccc@{}}
        \toprule
        \textbf{Distance} & \textbf{Pseudometric} & \textbf{Invariant} & \textbf{Efficient} & \textbf{Robust} & \textbf{Predictive} \\
        \midrule
        \canonicalizeddistanceabbrevonly{} & \yessymbol{} & \yessymbol{} & \yessymbol{} & \yessymbol{} & \yessymbol \\
        \nearestpointabbrevonly{} & \nosymbol{} & \yessymbol{} & \nosymbol{} & \nosymbol{} & \yessymbol{} \\
        \returncorrelationabbrevonly{} & \yessymbol{} & \nosymbol{} & \yessymbol{} & \nosymbol{} & \yessymbol{} \\
        \bottomrule
    \end{tabular}
\end{table}

We introduce the \emph{\canonicalizeddistancelongnameonly{} (\canonicalizeddistanceabbrevonly{})} distance that meets all the criteria in Table~\ref{table:desiderata}.
We believe \canonicalizeddistanceabbrevonly{} is the first method to quantitatively evaluate reward functions without training a policy.
\canonicalizeddistanceabbrevonly{} (section~\ref{sec:comparing}) canonicalizes the reward functions' potential-based shaping~\citep{ng:1999}, then takes the correlation between the canonical rewards over a \emph{\transitiondatasetname{}} $\transitiondataset$ of transitions.
We also introduce baselines \emph{\nearestpointabbrevonly{}} and \emph{\returncorrelationabbrevonly{}}~(section~\ref{sec:baselines}) which partially satisfy the criteria.

\canonicalizeddistanceabbrevonly{} works best when $\transitiondataset$ has support on all realistic transitions.
We achieve this in our experiments by using uninformative priors, such as rollouts of a policy taking random actions.
Moreover, we find \canonicalizeddistanceabbrevonly{} is robust to the exact choice of distribution $\transitiondataset{}$, producing similar results across a range of distributions, whereas \returncorrelationabbrevonly{} and especially \nearestpointabbrevonly{} are highly sensitive to the choice of $\transitiondataset{}$~(section~\ref{sec:experiments:visitation-sensitivity}).

Moreover, low \canonicalizeddistanceabbrevonly{} distance between a learned reward $\learnedreward$ and the true reward $\gtreward{}$ predicts low regret.
That is, the policies $\policy_{\learnedreward}$ and $\policy_{\reward}$ optimized for $\learnedreward{}$ and $\gtreward{}$ obtain similar returns under $\gtreward{}$.
Theorem~\ref{thm:epic-regret-bound-discrete} bounds the regret even in unseen environments; by contrast, the rollout method can only determine regret in the evaluation environment.
We also confirm this result empirically~(section~\ref{sec:experiments:ground-truth}).

\section{Related work}
There exist a variety of methods to learn reward functions.
\ag{This paragraph could be condensed a little; I go into much more detail here than for preference comparison.}
Inverse reinforcement learning (IRL)~\citep{ng:2000} is a common approach that works by inferring a reward function from demonstrations.
The IRL problem is inherently underconstrained: many different reward functions lead to the same demonstrations.
Bayesian IRL~\citep{ramachandran:2007} handles this ambiguity by inferring a posterior over reward functions.
By contrast, Maximum Entropy IRL~\citep{ziebart:2008} selects the highest entropy reward function consistent with the demonstrations; this method has scaled to high-dimensional environments~\citep{finn:2016,fu:2018}.

An alternative approach is to learn from \textit{preference comparisons} between two trajectories~\citep{akrour:2011,wilson:2012,christiano:2017,sadigh:2017}.
\mbox{T-REX}~\citep{brown:2019} is a hybrid approach, learning from a \textit{ranked} set of demonstrations.
More directly, \citet{cabi:2019} learn from ``sketches'' of cumulative reward over an episode.

To the best of our knowledge, there is no prior work that focuses on evaluating reward functions directly.
The most closely related work is \citet{ng:1999}, identifying reward transformations guaranteed to not change the optimal policy.
However, a variety of ad-hoc methods have been developed to evaluate reward functions.
The rollout method -- evaluating rollouts of a policy trained on the learned reward -- is evident in the earliest work on IRL~\citep{ng:2000}.
\citet{fu:2018} refined the rollout method by testing on a transfer environment, inspiring our experiment in section~\ref{sec:experiments:ground-truth}.
Recent work has compared reward functions by scatterplotting returns~\citep{ibarz:2018,brown:2019}, inspiring our \returncorrelationabbrevonly{} baseline (section~\ref{sec:baselines:return-correlation}).

\section{Background}
\label{sec:background}

This section introduces material needed for the distances defined in subsequent sections.
We start by introducing the \emph{Markov Decision Process (MDP)} formalism, then describe when reward functions induce the same optimal policies in an MDP, and finally define the notion of a distance \emph{metric}.

\begin{defn}
A \emph{Markov Decision Process (MDP)} $M = \mdp$ consists of
a set of states $\statespace$ and a set of actions $\actionspace$;
a discount factor $\discount \in [0,1]$;
an initial state distribution $\initialstatedist(\state)$;
a transition distribution $\transitiondist(\nextstate \mid \state,\action)$ specifying the probability of transitioning to $\nextstate{}$ from $\state{}$ after taking action $\action{}$;
and a reward function $\reward(\state, \action, \nextstate)$ specifying the reward upon taking action $\action$ in state $\state$ and transitioning to state $\nextstate$.
\end{defn}

A trajectory $\tau = (\state_0, \action_0, \state_1, \action_1, \cdots)$ consists of a sequence of states $\state_i \in \statespace$ and actions $\action_i \in \actionspace$. The \emph{return} on a trajectory is defined as the sum of discounted rewards, $\episodereturn{}(\tau; \reward{}) = \sum_{t=0}^{|\tau|} \discount^t \reward{}(\state_t, \action_t, \state_{t+1})$, where the length of the trajectory $|\tau|$ may be infinite.

In the following, we assume a discounted ($\gamma < 1$) infinite-horizon MDP.
The results can be generalized to undiscounted ($\gamma = 1$) MDPs subject to regularity conditions needed for convergence.

A \emph{stochastic policy} $\policy(\action \mid \state)$ assigns probabilities to taking action $\action \in \actionspace$ in state $\state \in \statespace$. The objective of an MDP is to find a policy $\policy$ that maximizes the expected return $\policyreturn(\policy) = \expectation_{\tau(\policy)}\left[\episodereturn{}(\tau; \reward{})\right]$, where $\tau(\policy)$ is a trajectory generated by sampling the initial state $s_0$ from $\initialstatedist$, each action $a_t$ from the policy $\policy(\action_t \mid \state_t)$ and successor states $\state_{t+1}$ from the transition distribution $\transitiondist(\state_{t+1} \mid \state_t, \action_t)$. An MDP $M$ has a set of optimal policies $\policy^*(M)$ that maximize the expected return, $\policy^*(M) = \argmax_{\policy} \policyreturn(\policy)$.

In this paper, we consider the case where we only have access to an \mdpnorliteral{}, $M^{-} = \mdpnor$.
The unknown reward function $\reward$ must be learned from human data.
Typically, only the state space $\statespace$, action space $\actionspace$ and discount factor $\discount$ are known exactly, with the initial state distribution $\initialstatedist$ and transition dynamics $\transitiondist$ only observable from interacting with the environment $M^{-}$.
Below, we describe an equivalence class whose members are guaranteed to have the same optimal policy set in \emph{any} \mdpnorliteral{} $M^{-}$ with fixed $\statespace$, $\actionspace$ and $\discount$ (allowing the unknown $\transitiondist$ and $\initialstatedist$ to take arbitrary values).

\begin{defn}
	\label{defn:potential-shaping}
	Let $\discount \in [0,1]$ be the discount factor, and $\potential:\statespace \to \mathbb{R}$ a real-valued function. Then $\reward(\state,\action,\nextstate) = \discount \potential(\nextstate) - \potential(\state)$ is a \emph{potential shaping} reward, with \emph{potential} $\potential$~\citep{ng:1999}.
\end{defn}

\begin{defn}[Reward Equivalence]
	\label{defn:equivalence}
	We define two bounded reward functions $\firstreward{}$ and $\secondreward{}$ to be \emph{equivalent}, $\firstreward{} \equivreward{} \secondreward{}$, for a fixed $(\statespace, \actionspace, \discount)$ if and only if there exists a constant $\lambda > 0$ and a bounded potential function $\potential:\statespace \to \mathbb{R}$ such that for all $\state,\nextstate \in \statespace$ and $\action \in \actionspace$:
	\begin{equation}
	\secondreward{}(\state,\action,\nextstate) = \lambda\firstreward{}(\state,\action,\nextstate) + \discount \potential(\nextstate) - \potential(\state).
	\end{equation}
\end{defn}

\begin{restatable}{prop}{rewardequivisequivalence}
The binary relation $\equiv$ is an equivalence relation. Let $\firstreward{}, \secondreward{}, \thirdreward{}:\statespace \times \actionspace \times \statespace \to \mathbb{R}$ be bounded reward functions. Then $\equiv$ is reflexive, $\firstreward{} \equivreward \firstreward{}$; symmetric, $\firstreward{} \equivreward \secondreward{}$ implies $\secondreward{} \equivreward \firstreward{}$; and transitive, $\left(\firstreward{} \equivreward \secondreward{}\right) \land \left(\secondreward{} \equivreward \thirdreward{}\right)$ implies $\firstreward{} \equivreward \thirdreward{}$.
\end{restatable}
\begin{compact-prf}
See section~\ref{sec:supp:proofs:background} in supplementary material.
\end{compact-prf}

The expected return of potential shaping $\gamma \potential(s') - \potential(s)$ on a trajectory segment $(s_0, \cdots, s_T)$ is $\gamma^T \potential(s_T) - \potential(s_0)$.
The first term $\gamma^T \potential(s_T) \to 0$ as $T \to \infty$, while the second term $\potential(s_0)$ only depends on the initial state, and so potential shaping does not change the set of optimal policies.
Moreover, any additive transformation that is not potential shaping will, for some reward $\reward$ and transition distribution $\transitiondist$, produce a set of optimal policies that is disjoint from the original~\citep{ng:1999}. %

The set of optimal policies is invariant to constant shifts $c \in \mathbb{R}$ in the reward, however this can already be obtained by shifting $\potential$ by $\frac{c}{\gamma - 1}$.\footnote{Note constant shifts in the reward of an undiscounted MDP would cause the value function to diverge. Fortunately, the shaping $\gamma\potential(\nextstate) - \potential(\state)$ is unchanged by constant shifts to $\potential$ when $\discount = 1$.}
Scaling a reward function by a positive factor $\lambda > 0$ scales the expected return of all trajectories by $\lambda$, also leaving the set of optimal policies unchanged.

If $\firstreward{} \equivreward{} \secondreward{}$ for some fixed $(\statespace, \actionspace, \discount)$, then for any \mdpnorliteral{} \mbox{$M^{-} = \mdpnor$} we have
$\policy^*\left(\left(M^{-}, \firstreward{}\right)\right) = \policy^*\left(\left(M^{-}, \secondreward{}\right)\right)$,
where $\left(M^{-}, R{}\right)$ denotes the MDP specified by $M^{-}$ with reward function $R$.
In other words, $\firstreward{}$ and $\secondreward{}$ induce the same optimal policies for all initial state distributions $\initialstatedist$ and transition dynamics $\transitiondist$.

\begin{defn}
Let $X$ be a set and $d:X \times X \to [0, \infty)$ a function. $d$ is a \emph{premetric} if $d(x,x) = 0$ for all $x \in X$. $d$ is a \emph{pseudometric} if, furthermore, it is symmetric, $d(x,y) = d(y,x)$ for all $x, y \in X$; and satisfies the triangle inequality, $d(x,z) \leq d(x,y) + d(y,z)$ for all $x, y, z \in X$. $d$ is a \emph{metric} if, furthermore, for all $x, y \in X$, $d(x,y) = 0 \implies x = y$.
\end{defn}

We wish for $d(\firstreward{},\secondreward{}) = 0$ whenever the rewards are equivalent, $\firstreward{} \equivreward{} \secondreward{}$, even if they are not identical, $\firstreward{} \neq \secondreward{}$.
This is forbidden in a metric but permitted in a pseudometric, while retaining other guarantees such as symmetry and triangle inequality that a metric provides.
Accordingly, a pseudometric is usually the best choice for a distance $d$ over reward functions.

\section{Comparing reward functions with \canonicalizeddistanceabbrevonly{}}
\label{sec:comparing}

In this section we introduce the \emph{\canonicalizeddistancelongnameonly{} (\canonicalizeddistanceabbrevonly{})} pseudometric.
This novel distance canonicalizes the reward functions' potential-based shaping, then compares the canonical representatives using Pearson correlation, which is invariant to scale.
Together, this construction makes \canonicalizeddistanceabbrevonly{} invariant on reward equivalence classes.
See section~\ref{sec:supp:proofs:canonicalizeddistance} for proofs.

We define the \emph{canonically shaped reward} $\canonical{\reward{}}$ as an expectation over some arbitrary distributions $\statedistribution$ and $\actiondistribution$ over states $\statespace$ and actions $\actionspace$ respectively.
This construction means that $\canonical{\reward{}}$ does not depend on the MDP's initial state distribution $\initialstatedist$ or transition dynamics $\transitiondist$.
In particular, we may evaluate $\reward{}$ on transitions that are impossible in the training environment, since these may become possible in a deployment environment with a different $\initialstatedist$ or $\transitiondist$.

\begin{defn}[Canonically Shaped Reward]
\label{defn:canonically-shaped-reward}
Let $\reward{}:\statespace \times \actionspace \times \statespace \to \mathbb{R}$ be a reward function.
Given distributions $\statedistribution \in \Delta(\statespace)$ and $\actiondistribution \in \Delta(\actionspace)$ over states and actions, let $\staterv{}$ and $\nextstaterv{}$ be random variables independently sampled from $\statedistribution$ and $\actionrv{}$ sampled from $\actiondistribution$.
We define the \emph{canonically shaped} $\reward{}$ to be:
\ag{Strictly speaking need some conditions for expectation to converge -- bounded would be sufficient but not necessary.}
\begin{equation}
\canonical{\reward{}}(\state, \action, \nextstate) = \reward{}(\state,\action,\nextstate) + \expectation\left[\discount\reward{}(\nextstate, \actionrv{}, \nextstaterv{}) - \reward{}(\state, \actionrv{}, \nextstaterv{}) - \discount\reward{}(\staterv{}, \actionrv{}, \nextstaterv{})\right].
\end{equation}
\end{defn}

Informally, if $\reward'$ is shaped by potential $\potential$, then increasing $\potential(\state)$ decreases $\reward'(\state,\action,\nextstate)$ but increases $\expectation\left[-\reward'(\state,\actionrv{},\nextstaterv{})\right]$, canceling.
Similarly, increasing $\potential(\nextstate)$ increases $\reward'(\state,\action,\nextstate)$ but decreases $\expectation\left[\discount \reward'(\nextstate,\actionrv{},\nextstaterv{})\right]$.
Finally, $\expectation[\discount R(\staterv{},\actionrv{},\nextstaterv{})]$ centers the reward, canceling constant shift.

\begin{restatable}[The Canonically Shaped Reward is Invariant to Shaping]{prop}{canonicallyshapedrewardinvariant}
\label{prop:canonically-shaped-reward-invariant-shaping}
Let $\reward{}:\statespace \times \actionspace \times \statespace \to \mathbb{R}$ be a reward function and $\potential:\statespace \to \mathbb{R}$ a potential function.
Let $\discount \in [0,1]$ be a discount rate, and $\statedistribution \in \Delta(\statespace)$ and $\actiondistribution \in \Delta(\actionspace)$ be distributions over states and actions.
Let $\reward{}'$ denote $\reward{}$ shaped by $\potential$: $\reward{}'(\state,\action,\nextstate) = \reward{}(\state, \action, \nextstate) + \discount \potential(\nextstate) - \potential(\state)$.
Then the canonically shaped $\reward{}'$ and $\reward{}$ are equal:
\ag{Sam pointed out the converse is also true; if the canonicalized versions are equal at all points then they must be up to shaping. Probably worth proving this and rephrasing proposition to be an if-and-only-if condition.}
$\canonical{\reward'} = \canonical{\reward{}}$.
\end{restatable}
\begin{optional-prf}
See section~\ref{sec:supp:proofs:canonicalizeddistance}.
\end{optional-prf}

Proposition~\ref{prop:canonically-shaped-reward-invariant-shaping} holds for arbitrary distributions $\statedistribution$ and $\actiondistribution$.
However, in the following Proposition we show that the potential shaping applied by the canonicalization $\canonical{\reward}$ is more influenced by perturbations to $\reward$ of transitions $(\state, \action, \nextstate)$ with high joint probability.
This suggests choosing $\statedistribution$ and $\actiondistribution$ to have broad support, making $\canonical{\reward}$ more robust to perturbations of any given transition.

\begin{restatable}{prop}{smoothnesscanonicalization}
\label{prop:smoothness-canonicalization}
Let $\statespace$ and $\actionspace$ be finite, with $|\statespace| \geq 2$.
Let $\statedistribution \in \Delta(\statespace)$ and $\actiondistribution \in \Delta(\actionspace)$.
Let $\reward,\:\nu:\statespace \times \actionspace \times \statespace \to \mathbb{R}$ be reward functions, with $\nu(\state, \action,\nextstate) = \lambda \mathbb{I}[(\state,\action,\nextstate) = (x, u, x')]$, $\lambda \in \mathbb{R}$, $x,x' \in \statespace$ and $u \in \actionspace$.
Let $\potential_{\statedistribution, \actiondistribution}(\reward)(\state, \action, \nextstate) = \canonical{\reward}(\state, \action, \nextstate) - \reward(\state, \action, \nextstate)$.
Then:
\begin{equation}
\left\Vert \potential_{\statedistribution, \actiondistribution}(\reward + \nu) - \potential_{\statedistribution, \actiondistribution}(\reward) \right\Vert_{\infty} = \lambda\left(1 + \discount\statedistribution(x)\right)\actiondistribution(u)\statedistribution(x').
\end{equation}
\end{restatable}

We have canonicalized potential shaping; next, we compare the rewards in a scale-invariant manner.

\begin{defn}
The \emph{Pearson distance} between random variables $\firstrv{}$ and $\secondrv{}$ is defined by the expression
$\pearsondistance{}(\firstrv{},\secondrv{}) = \sqrt{1 - \rho(\firstrv{},\secondrv{})}/\sqrt{2}$,
where $\rho(\firstrv{},\secondrv{})$ is the Pearson correlation between $\firstrv{}$ and $\secondrv{}$.
\end{defn}

\begin{restatable}{lemma}{pearsondistanceproperties}
\label{lemma:pearson-distance-properties}
The Pearson distance $\pearsondistance{}$ is a pseudometric.
Moreover, let $a,b \in (0,\infty)$, $c, d \in \mathbb{R}$ and $\firstrv{}, \secondrv{}$ be random variables.
Then it follows that $0 \leq \pearsondistance{}(a\firstrv{} + c, b\secondrv{} + d) = \pearsondistance{}(\firstrv{}, \secondrv{}) \leq 1$.
\end{restatable}

We can now define \canonicalizeddistanceabbrevonly{} in terms of the Pearson distance between canonically shaped rewards.

\begin{defn}[\canonicalizeddistancewithabbrev{}]
\label{defn:canonicalized-distance}
Let $\transitiondataset{}$ be some \transitiondatasetname{} over transitions $\state \overset{a}{\to} \nextstate$.
Let $\staterv{},\actionrv{},\nextstaterv{}$ be random variables jointly sampled from $\transitiondataset{}$.
Let $\statedistribution$ and $\actiondistribution$ be some distributions over states $\statespace$ and $\actionspace$ respectively.
The \emph{\canonicalizeddistancelongnameonly{} (\canonicalizeddistanceabbrevonly{})} distance between reward functions $\firstreward{}$ and $\secondreward{}$ is:%
\begin{equation}
\label{eq:canonicalized-distance}
\canonicalizeddistance{}(\firstreward{}, \secondreward{}) = \pearsondistance{}\left(\canonical{\firstreward{}}(\staterv{},\actionrv{},\nextstaterv{}),\canonical{\secondreward{}}(\staterv{},\actionrv{},\nextstaterv{})\right).
\end{equation}
\end{defn}

\begin{restatable}{theorem}{canonicalizeddistancepseudometric}
\label{thm:canonicalized-distance-pseudometric}
The \canonicalizeddistancelongnameonly{} distance is a pseudometric.
\end{restatable}

Since \canonicalizeddistanceabbrevonly{} is a pseudometric, it satisfies the triangle inequality.
To see why this is useful, consider an environment with an expensive to evaluate ground-truth reward $\gtreward{}$.
Directly comparing many learned rewards $\learnedreward{}$ to $\gtreward{}$ might be prohibitively expensive.
We can instead pay a one-off cost: query $\gtreward{}$ a finite number of times and infer a proxy reward $\reward_P$ with $\canonicalizeddistance{}(\gtreward{}, \reward_P) \leq \epsilon$.
The triangle inequality allows us to evaluate $\learnedreward{}$ via comparison to $\reward_P$, since $\canonicalizeddistance(\learnedreward{}, \gtreward{}) \leq \canonicalizeddistance(\learnedreward{}, \reward_P) + \epsilon$.
This is particularly useful for benchmarks, which can be expensive to build but should be cheap to use.

\begin{restatable}{theorem}{canonicalizeddistanceproperties}
\label{thm:canonicalized-distance-properties}
Let $\firstreward{}$, $\firstreward'$, $\secondreward{}, \secondreward' : \statespace \times \actionspace \times \statespace \to \mathbb{R}$ be reward functions such that $\firstreward' \equivreward \firstreward{}$ and $\secondreward' \equivreward \secondreward{}$.
Then $0 \leq \canonicalizeddistance{}(\firstreward', \secondreward') = \canonicalizeddistance{}(\firstreward{}, \secondreward{}) \leq 1$.
\end{restatable}

The following is our main theoretical result, showing that $\canonicalizeddistance{}(\firstreward{}, \secondreward{})$ distance gives an upper bound on the difference in returns under \emph{either} $\firstreward{}$ or $\secondreward{}$ between optimal policies $\policy^*_{\firstreward}$ and $\policy^*_{\secondreward}$.
In other words, \canonicalizeddistanceabbrevonly{} bounds the regret under $\firstreward$  of using $\policy^*_{\secondreward}$ instead of $\policy^*_{\firstreward}$.
Moreover, by symmetry $\canonicalizeddistance{}(\firstreward{}, \secondreward{})$ also bounds the regret under $\secondreward$ of using $\policy^*_{\firstreward}$ instead of $\policy^*_{\secondreward}$.
\begin{restatable}{theorem}{epicregretbounddiscrete}
\label{thm:epic-regret-bound-discrete}
Let $M$ be a $\discount$-discounted \mdpnorliteral{} with finite state and action spaces $\statespace$ and $\actionspace$.
Let $\firstreward{}, \secondreward{}:\statespace \times \actionspace \times \statespace \to \mathbb{R}$ be rewards, and $\policy_A^*, \policy_B^*$ be respective optimal policies.
Let $\transitiondataset_{\policy}(t, \state[t], \action[t], \state[t+1])$ denote the distribution over transitions $\statespace \times \actionspace \times \statespace$ induced by policy $\policy$ at time $t$, and $\transitiondataset{}(\state,\action,\nextstate)$ be the \transitiondatasetname{} used to compute $\canonicalizeddistance$.
Suppose there exists $K > 0$ such that $K\transitiondataset(\state[t], \action[t], \state[t+1]) \geq \transitiondataset_{\policy}(t, \state[t], \action[t], \state[t+1])$ for all times $t \in \mathbb{N}$, triples $(\state[t],\action[t],\state[t+1]) \in \statespace \times \actionspace \times \statespace$ and policies $\pi \in \{\policy_A^*, \policy_B^*\}$.
Then the regret under $\firstreward$ from executing $\policy_B^*$ instead of $\policy_A^*$ is at most
\begin{equation*}
\policyreturn{}_{\firstreward{}}(\policy_A^*) - \policyreturn{}_{\firstreward{}}(\policy_B^*) \leq 16K\lVert\firstreward{}\rVert_2\left(1 - \discount\right)^{-1} \canonicalizeddistance{}(\firstreward{}, \secondreward{}),
\end{equation*}
where $\policyreturn{}_{\reward}(\policy)$ is the return of policy $\policy$ under reward $\reward$.
\end{restatable}

We generalize the regret bound to continuous spaces in theorem~\ref{thm:epic-regret-bound-lipschitz} via a Lipschitz assumption, with Wasserstein distance replacing $K$.
Importantly, the returns of $\policy_A^*$ and $\policy_B^*$ \textbf{converge} as $\canonicalizeddistance{}(\firstreward{}, \secondreward{}) \to 0$ in both cases, no matter which reward function you evaluate on.%

The key assumption is that the \transitiondatasetname{} $\transitiondataset$ has adequate support for transitions occurring in rollouts of $\policy_A^*$ and $\policy_B^*$.
The bound is tightest when $\transitiondataset{}$ is similar to $\transitiondataset_{\policy_A^*}$ and $\transitiondataset_{\policy_B^*}$.
However, computing $\policy_A^*$ and $\policy_B^*$ is often intractable.
The MDP $M$ may be unknown, such as when making predictions about an unseen deployment environment.
Even when $M$ is known, RL is computationally expensive and may fail to converge in non-trivial environments.

In finite cases, a uniform $\transitiondataset$ satisfies the requirements with $K \leq |\statespace|^2 |\actionspace|$.
In general, it is best to choose $\transitiondataset$ to have broad coverage over plausible transitions.
Broad coverage ensures adequate support for $\transitiondataset_{\policy_A^*}$ and $\transitiondataset_{\policy_B^*}$.
But excluding transitions that are unlikely or impossible to occur leads to tighter regret bounds due to a smaller $K$ (finite case) or Wasserstein distance (continuous case).

While \canonicalizeddistanceabbrevonly{} upper bounds policy regret, it does not lower bound it.
In fact, no reward distance can lower bound regret in arbitrary environments.
For example, suppose the deployment environment transitions to a randomly chosen state independent of the action taken.
In this case, all policies obtain the same expected return, so the policy regret is always zero, regardless of the reward functions.

\begin{figure*}
	\centering
	\includegraphics{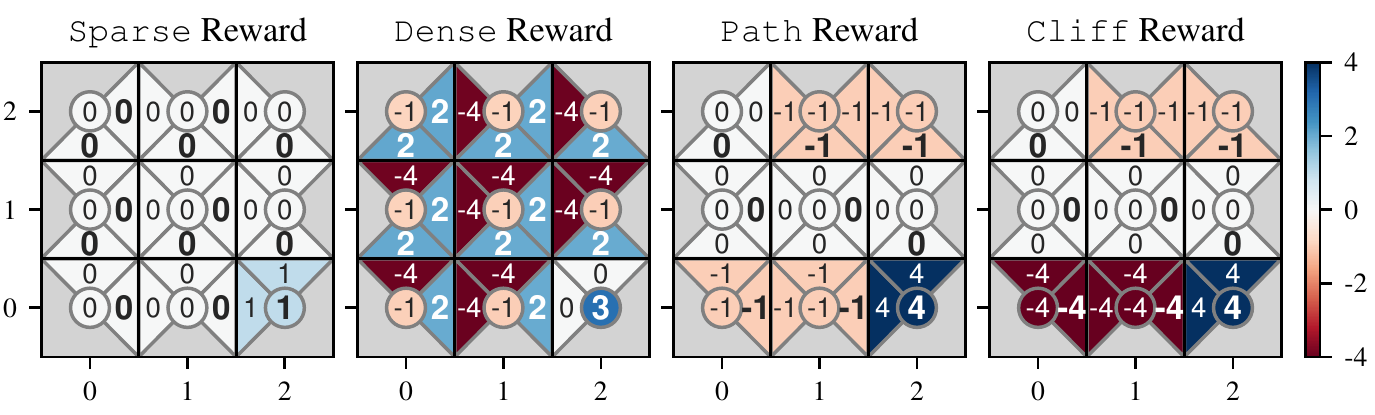}
	\caption{Heatmaps of four reward functions for a $3\times 3$ gridworld. \sparsegoal{} and \densegoal{} look different but are actually equivalent with $\canonicalizeddistance{}\left(\sparsegoal{}, \densegoal{}\right)=0$.
	By contrast, the optimal policies for \dirtpath{} and \cliffwalk{} are the same if the gridworld is deterministic but different if it is ``slippery''.
	\canonicalizeddistanceabbrevonly{} recognizes this difference with $\canonicalizeddistance{}\left(\dirtpath, \cliffwalk{}\right)=0.27$.
	\textbf{Key}:
		Reward $R(s, s')$ for moving from $s$ to $s'$ is given by the \textbf{triangular wedge} in cell $s$ that is adjacent to cell $s'$.
		$R(s, s)$ is given by the \textbf{central circle} in cell $s$.
		Optimal action(s) (deterministic, infinite horizon, discount $\gamma = 0.99$) have \textbf{bold} labels.
		See Figure~\ref{fig:supp:hand-designed-gridworld} for the distances between all reward pairs.}
	\label{fig:comparing:illustrative}
\end{figure*}

To demonstrate \canonicalizeddistanceabbrevonly{}'s properties, we compare the gridworld reward functions from Figure~\ref{fig:comparing:illustrative}, reporting the distances between all reward pairs in Figure~\ref{fig:supp:hand-designed-gridworld}.
\densegoal{} is a rescaled and shaped version of \sparsegoal{}, despite looking dissimilar at first glance, so $\canonicalizeddistance{}\left(\sparsegoal{}, \densegoal{}\right)=0$.
By contrast, $\canonicalizeddistance{}\left(\dirtpath, \cliffwalk{}\right)=0.27$.
In \emph{deterministic} gridworlds, \dirtpath{} and \cliffwalk{} have the same optimal policy, so the rollout method could wrongly conclude they are equivalent.
But in fact the rewards are fundamentally different: when there is a significant risk of ``slipping'' in the wrong direction the optimal policy for \cliffwalk{} walks along the top instead of the middle row, incurring a $-1$ penalty to avoid the risk of falling into the $-4$ ``cliff''.

For this example, we used state and action distributions $\statedistribution$ and $\actiondistribution$ uniform over $\statespace$ and $\actionspace$, and \transitiondatasetname{} $\transitiondataset{}$ uniform over state-action pairs $(\state,\action)$, with $\nextstate$ deterministically computed.
It is important these distributions have adequate support.
As an extreme example, if $\statedistribution$ and $\transitiondataset$ have no support for a particular state then the reward of that state has no effect on the distance.
We can compute \canonicalizeddistanceabbrevonly{} exactly in a tabular setting, but in general use a sample-based approximation~(section~\ref{sec:supp:approx:canonicalized-distance}).

\ag{TODO: mention as aside can derive this from a variant of advantage}
\ag{Cite appendix: comparison of different methods in the gridworlds.}

\section{Baseline approaches for comparing reward functions}
\label{sec:baselines}

Given the lack of established methods, we develop two alternatives as baselines: \returncorrelationlongnameonly{} (\returncorrelationabbrevonly) and  \nearestpointlongnameonly{} (\nearestpointabbrevonly).

\subsection{\returncorrelationlongnameonly{} (\returncorrelationabbrevonly{})}
\label{sec:baselines:return-correlation}

The goal of an MDP is to maximize expected episode return, so it is natural to compare reward functions by the returns they induce.
If the return of a reward function $\firstreward{}$ is a positive affine transformation of another reward $\secondreward{}$, then $\firstreward{}$ and $\secondreward{}$ have the same set of optimal policies.
This suggests using Pearson distance, which is invariant to positive affine transformations.

\begin{defn}[\returncorrelationlongnamewithabbrev]
\ag{Show this is actually a pseudometric.}
Let $\transitiondataset$ be some distribution over trajectories.
Let $E$ be a random variable sampled from $\transitiondataset$.
The \emph{\returncorrelationlongnameonly{}} distance between reward functions $\firstreward{}$ and $\secondreward{}$ is the Pearson distance between their episode returns on $\transitiondataset{}$,
$\returncorrelation{}(\firstreward{}, \secondreward{}) = \pearsondistance{}(g(E; \firstreward{}), g(E; \secondreward{}))$.
\end{defn}

Prior work has produced scatter plots of the return of $\firstreward{}$ against $\secondreward{}$ over episodes~\citep[Figure~3]{brown:2019} and fixed-length segments~\citep[section~D]{ibarz:2018}.
\returncorrelationabbrevonly{} is the Pearson distance of such plots, so is a natural baseline.
We approximate \returncorrelationabbrevonly{} by the correlation of episode returns on a finite collection of rollouts.

\returncorrelationabbrevonly{} is invariant to shaping when the initial state $\state_0$ and terminal state $\state_T$ are fixed.
Let $\reward{}$ be a reward function and $\potential{}$ a potential function, and define the shaped reward $\reward{}'(\state,\action,\nextstate) = \reward{}(\state,\action,\nextstate) + \discount \potential(\nextstate) - \potential(\state)$.
The return under the shaped reward on a trajectory $\tau = (\state_0, \action_0, \cdots, \state_T)$ is $g(\tau; \reward{}') = g(\tau; \reward{}) + \gamma^{T} \potential(\state_T) - \potential(\state_0)$.
Since $\state_0$ and $\state_T$ are fixed, $\gamma^{T} \potential(\state_T) - \potential(\state_0)$ is constant.
It follows that \returncorrelationabbrevonly{} is invariant to shaping, as Pearson distance is invariant to constant shifts.
In fact, for infinite-horizon discounted MDPs only $\state_0$ needs to be fixed, since $\gamma^{T} \potential(\state_T) \to 0$ as $T \to \infty$.

However, if the initial state $\state_0$ is stochastic, then the \returncorrelationabbrevonly{} distance can take on arbitrary values under shaping.
Let $\firstreward{}$ and $\secondreward{}$ be two arbitrary reward functions.
Suppose that there are at least two distinct initial states, $s_X$ and $s_Y$, with non-zero measure in $\transitiondataset{}$.
Choose potential $\potential(\state) = 0$ everywhere except $\potential(s_X) = \potential(s_Y) = c$, and let $\firstreward'$ and $\secondreward'$ denote $\firstreward{}$ and $\secondreward{}$ shaped by $\potential$.
As $c \to \infty$, the correlation $\rho\left(g(E; \firstreward'), g(E; \secondreward')\right) \to 1$.
This is since the relative difference tends to zero, even though $g(E; \firstreward')$ and $g(E; \secondreward')$ continue to have the same absolute difference as $c$ varies.
Consequently, the \returncorrelationabbrevonly{} \returncorrelationkind{} $\returncorrelation{}(\firstreward', \secondreward') \to 0$ as $c \to \infty$.
By an analogous argument, setting $\potential(s_X) = c$ and $\potential(s_Y) = -c$ gives $\returncorrelation{}(\firstreward', \secondreward') \to 1$ as $c \to \infty$.

\subsection{\nearestpointlongnameonly{} (\nearestpointabbrevonly{})}
\label{sec:baselines:nearest-point}

\nearestpointabbrevonly{} takes the minimum \directdistancename{} distance between equivalence classes. See section~\ref{sec:supp:proofs:nearestpoint} for proofs.

\begin{defn}[\directdistancename{} distance]
Let $\transitiondataset{}$ be a \transitiondatasetname{} over transitions $\state \overset{a}{\to} \nextstate$ and let $p \geq 1$ be a power.
The \emph{\directdistancename{} distance} between reward functions $\firstreward{}$ and $\secondreward{}$ is the $L^p$ norm of their difference:
$\directdistance(\firstreward{}, \secondreward{}) = \left(\expectation{}_{s,a,s' \sim \transitiondataset{}} \left[\left\lvert\firstreward{}(s,a,s') -  \secondreward{}(s,a,s')\right\rvert^p\right]\right)^{1/p}.$
\end{defn}

The \directdistancename{} distance is affected by potential shaping and positive rescaling that do not change the optimal policy.
A natural solution is to take the distance from the \emph{nearest point} in the equivalence class:
$\nearestpoint^U(\firstreward{}, \secondreward{}) = \inf_{\firstreward' \equivreward{} \firstreward{}} \directdistance(\firstreward', \secondreward{})$.
Unfortunately, $\nearestpoint^U$ is sensitive to $\secondreward{}$'s scale.

It is tempting to instead take the infimum over both arguments of $\directdistance$.
However, $\inf_{\firstreward' \equivreward{} \firstreward, \secondreward' \equivreward{} \secondreward} \directdistance(\firstreward', \secondreward{}) = 0$ since all equivalence classes come arbitrarily close to the origin in $L^p$ space.
Instead, we fix this by normalizing $\nearestpoint^U$.

\begin{defn}
\emph{\nearestpointabbrevonly{}} is defined by $\nearestpoint(\firstreward{}, \secondreward{}) =
\nearestpoint^U(\firstreward{}, \secondreward{})/\nearestpoint^U(\zeroreward{}, \secondreward{})$ when $\nearestpoint^U(\zeroreward{}, \secondreward{}) \neq 0$,
and is otherwise given by $\nearestpoint(\firstreward{}, \secondreward{}) = 0$.
\end{defn}

If $\nearestpoint^U(\zeroreward{}, \secondreward{}) = 0$ then $\nearestpoint^U(\firstreward{}, \secondreward{}) = 0$ since $\firstreward{}$ can be scaled arbitrarily close to $\zeroreward{}$.
Since all policies are optimal for $\reward{} \equivreward \zeroreward{}$, we choose $\nearestpoint(\firstreward{}, \secondreward{}) = 0$ in this case.

\begin{restatable}{theorem}{nearestpointproperties}
$\nearestpoint{}$ is a premetric on the space of bounded reward functions. Moreover, let $\firstreward{}, \firstreward{}', \secondreward{}, \secondreward{}':\statespace \times \actionspace \times \statespace \to \mathbb{R}$ be bounded reward functions such that $\firstreward{} \equivreward \firstreward{}'$ and $\secondreward{} \equivreward \secondreward{}'$.
Then $0 \leq \nearestpoint{}(\firstreward{}', \secondreward{}') = \nearestpoint{}(\firstreward{}, \secondreward{}) \leq 1$.
\end{restatable}
\begin{optional-prf}
Pseudometric follows from $\directdistance$ a pseudometric; see section~\ref{sec:supp:proofs:nearestpoint} for details.
Invariance to $\firstreward' \equivreward \firstreward$ immediate from the infimum being over $R \equivreward \firstreward{}$.
Invariance to $\secondreward' \equivreward \secondreward$ is due to translational invariance of $\directdistance$, and $\nearestpoint^U(\firstreward{}, \lambda \secondreward{}) = \lambda \nearestpoint^U(\firstreward{}, \secondreward{})$ for $\lambda > 0$.
Upper bound of $1$ is due to $\nearestpoint^U(\firstreward{}, \secondreward{}) \leq \nearestpoint^U(\zeroreward{}, \secondreward{})$, while lower bound immediate from $\directdistance$ non-negative.
See section~\ref{sec:supp:proofs:nearestpoint} for details.
\end{optional-prf}

Note that $\nearestpoint{}$ may not be symmetric and so is not, in general, a pseudometric: see proposition~\ref{prop:nearest-point-not-pseudometric}.
The infimum in $\nearestpoint^U$ can be computed exactly in a tabular setting, but in general we must approximate it using gradient descent.
This gives an upper bound for $\nearestpoint^U$, but the quotient of upper bounds $\nearestpoint$ may be too low or too high.
See section~\ref{sec:supp:approx:nearest-point} for details of the approximation.

\section{Experiments}

We evaluate \canonicalizeddistanceabbrevonly{} and the baselines \returncorrelationabbrevonly{} and \nearestpointabbrevonly{} in a variety of continuous control tasks.
In section~\ref{sec:experiments:hand-designed}, we compute the distance between hand-designed reward functions, finding \canonicalizeddistanceabbrevonly{} to be the most reliable.
\nearestpointabbrevonly{} has substantial approximation error, and \returncorrelationabbrevonly{} sometimes erroneously assigns high distance to equivalent rewards.
Next, in section~\ref{sec:experiments:visitation-sensitivity} we show \canonicalizeddistanceabbrevonly{} is robust to the exact choice of \transitiondatasetname{} $\transitiondataset{}$, whereas \returncorrelationabbrevonly{} and especially \nearestpointabbrevonly{} are highly sensitive to the choice of $\transitiondataset{}$.
Finally, in section~\ref{sec:experiments:ground-truth} we find that the distance of learned reward functions to a ground-truth reward predicts the return obtained by policy training, even in an unseen test environment.

\ag{Mention setup of different algorithms?}
\subsection{Comparing hand-designed reward functions}
\label{sec:experiments:hand-designed}

\begin{figure}
	\centering
	\compactheatmap{point_mass}{middle}{}
	\caption[]{Approximate distances between hand-designed reward functions in \env{PointMass}, where the agent moves on a line trying to reach the origin. \canonicalizeddistanceabbrevonly{} correctly assigns $0$ distance between equivalent rewards such as $(\dense{}\nocontrolpenalty{},\sparse{}\nocontrolpenalty{})$ while $\nearestpoint{}(\dense{}\nocontrolpenalty{},\sparse{}\nocontrolpenalty{}) = 0.58$ and $\returncorrelation{}(\dense{}\nocontrolpenalty{},\sparse{}\nocontrolpenalty{}) = 0.56$. The \transitiondatasetname{} $\transitiondataset{}$ is sampled from rollouts of a policy $\randompolicy{}$ taking actions uniformly at random. \pointmasscaption{}}
	\label{fig:experiments:hand-designed}
\end{figure}

We compare procedurally specified reward functions in four tasks, finding that \canonicalizeddistanceabbrevonly{} is more reliable than the baselines \nearestpointabbrevonly{} and \returncorrelationabbrevonly{}, and more computationally efficient than \nearestpointabbrevonly{}.
Figure~\ref{fig:experiments:hand-designed} presents results in the proof-of-concept \env{PointMass} task.
The results for \env{Gridworld}, \env{HalfCheetah} and \env{Hopper}, in section~\ref{sec:supp:hand-designed}, are qualitatively similar.

In \env{PointMass} the agent can accelerate $\ddot{x}$ left or right on a line.
The reward functions include (\controlpenalty{}) or exclude (\nocontrolpenalty{}) a quadratic penalty $\ddot{x}^2$.
The sparse reward (\sparse{}) gives a reward of $1$ in the region $\pm 0.05$ from the origin.
The dense reward (\dense{}) is a shaped version of the sparse reward.
The magnitude reward (\magnitude{}) is the negative distance of the agent from the origin.

We find that \canonicalizeddistanceabbrevonly{} correctly identifies the equivalent reward pairs (\sparse{}\controlpenalty{}-\dense{}\controlpenalty{} and \sparse{}\nocontrolpenalty{}-\dense{}\nocontrolpenalty{}) with estimated distance $<\num{1e-3}$.
By contrast, \nearestpointabbrevonly{} has substantial approximation error: $\nearestpoint{}(\dense{}\nocontrolpenalty{},\sparse{}\nocontrolpenalty{}) = 0.58$.
Similarly, $\returncorrelation{}(\dense{}\nocontrolpenalty{},\sparse{}\nocontrolpenalty{}) = 0.56$ due to \returncorrelationabbrevonly{}'s erroneous handling of stochastic initial states.
Moreover, \nearestpointabbrevonly{} is computationally inefficient: Figure~\ref{fig:experiments:hand-designed}(b) took 31 hours to compute.
By contrast, the figures for \canonicalizeddistanceabbrevonly{} and \returncorrelationabbrevonly{} were generated in less than two hours, and a lower precision approximation of \canonicalizeddistanceabbrevonly{} finishes in just $17$ seconds (see section~\ref{sec:supp:experiments:runtime}).

\subsection{Sensitivity of reward distance to \transitiondatasetname{}}
\label{sec:experiments:visitation-sensitivity}

Reward distances should be robust to the choice of \transitiondatasetname{} $\transitiondataset$.
In Table~\ref{tab:learned-reward} (center), we report distances from the ground-truth reward (\groundtruthmethod{}) to reward functions (rows) across \transitiondatasetname{}s $\transitiondataset \in \{\randompolicy{}, \expertpolicy{}, \mixturepolicy{}\}$ (columns).
We find \canonicalizeddistanceabbrevonly{} is fairly robust to the choice of $\transitiondataset{}$ with a similar ratio between rows in each column $\transitiondataset$.
By contrast, \returncorrelationabbrevonly{} and especially \nearestpointabbrevonly{} are substantially more sensitive to the choice of $\transitiondataset$.

We evaluate in the \pointmaze{} MuJoCo task from \citet{fu:2018}, where a point mass agent must navigate around a wall to reach a goal.
The \transitiondatasetname{}s $\transitiondataset{}$ are induced by rollouts from three different policies: $\randompolicy{}$ takes actions uniformly at random, producing broad support over transitions; $\expertpolicy{}$ is an expert policy, yielding a distribution concentrated around the goal; and \mixturepolicy{} is a mixture of the two.
In \canonicalizeddistanceabbrevonly{}, $\statedistribution$ and $\actiondistribution$ are marginalized from $\transitiondataset$ and so also vary with $\transitiondataset$.

\newcommand{\numformat}[1]{\num[round-mode=figures,round-precision=3,scientific-notation=false]{#1}}
\begin{table}
	\centering
	\caption{Low reward distance from the ground-truth (\groundtruthmethod{}) in \pointmazetrain{} predicts high policy return even in unseen task \pointmazetest{}. \canonicalizeddistanceabbrevonly{} distance is robust to the choice of \transitiondatasetname{} $\transitiondataset{}$, with similar values across columns, while \returncorrelationabbrevonly{} and especially \nearestpointabbrevonly{} are sensitive to $\transitiondataset{}$. \textbf{Center}: approximate distances ($1000\times$ scale) of reward functions from \groundtruthmethod{}. The \transitiondatasetname{} $\transitiondataset{}$ is computed from rollouts in \pointmazetrain{} of: a uniform random policy $\boldsymbol{\randompolicy{}}$, an expert $\boldsymbol{\expertpolicy{}}$ and a \textbf{\mixturepolicy{}}ture of these policies. $\statedistribution{}$ and $\actiondistribution{}$ are computed by marginalizing $\transitiondataset{}$.
	\textbf{Right}: mean \groundtruthmethod{} return over $9$ seeds of RL training on the reward in \env{PointMaze-\{Train,Test\}}, and returns for AIRL's \emph{generator} policy. %
	\textbf{Confidence Intervals}: see Table~\ref{tab:supp:learned-reward}.}
	\setlength{\tabcolsep}{3pt}
	\begin{tabular}{@{}ll|@{\hspace{0.75em}}llll@{\hspace{0.4em}}llll@{\hspace{0.4em}}llll|@{\hspace{0.75em}}ccc@{}}
		\toprule
		\makeatletter
		\textbf{Reward} & & \multicolumn{3}{c}{$\boldsymbol{1000 \times \canonicalizeddistance{}}$} & & \multicolumn{3}{c}{$\boldsymbol{1000 \times \nearestpoint{}}$} & & \multicolumn{3}{c}{$\boldsymbol{1000 \times \returncorrelation{}}$} & & \multicolumn{3}{c}{\textbf{Episode Return}} \\
		\textbf{Function} & & $\boldsymbol{\randompolicy{}}$ & $\boldsymbol{\expertpolicy{}}$ & \textbf{\mixturepolicy{}} & & $\boldsymbol{\randompolicy{}}$ & $\boldsymbol{\expertpolicy{}}$ & \textbf{\mixturepolicy{}} & & $\boldsymbol{\randompolicy{}}$ & $\boldsymbol{\expertpolicy{}}$ & \textbf{\mixturepolicy{}} & & \textbf{Gen.} & \textbf{Train} & \textbf{Test}\Bstrut \\
		\hline
		\groundtruthmethod{} & & \num{0.06} & \num{0.05} & \num{0.04} &  & \num{0.04} & \num{3.17} & \num{0.01} &  & \num{0.00} & \num{0.00} & \num{0.00} &  & --- & \num{-5.19} & \num{-6.59}\Tstrut \\
		\regressionmethod{} & & \num{35.8} & \num{33.7} & \num{26.1} &  & \num{1.42} & \num{38.9} & \num{0.35} &  & \num{9.99} & \num{90.7} & \num{2.43} &  & --- & \num{-5.47} & \num{-6.30} \\
		\preferencesmethod{} & & \num{68.7} & \num{100} & \num{56.8} &  & \num{8.51} & \num{1333} & \num{9.74} &  & \num{24.9} & \num{360} & \num{19.6} &  & --- & \num{-5.57} & \num{-5.04} \\
		\airlstateonlymethod{} & & \num{572} & \num{520} & \num{404} &  & \num{817} & \num{2706} & \num{488} &  & \num{549} & \num{523} & \num{240} &  & \num{-5.43} & \num{-27.3} & \num{-22.7} \\
		\airlstateactionmethod{} & & \num{776} & \num{930} & \num{894} &  & \num{1067} & \num{2040} & \num{1039} &  & \num{803} & \num{722} & \num{964} &  & \num{-5.05} & \num{-30.7} & \num{-29.0}\Bstrut \\
		\hline
		\bettergoalmethod{} & & \num{17.0} & \num{0.05} & \num{397} &  & \num{0.68} & \num{6.30} & \num{597} &  & \num{35.3} & <0.01 & \num{166} &  & --- & \num{-30.4} & \num{-29.1}\Tstrut \\
		\bottomrule
	\end{tabular}
	\label{tab:learned-reward}
\end{table}

We evaluate four reward learning algorithms:
\regressionmethod{}ion onto reward labels~\citep[\emph{target} method from][section~3.3]{christiano:2017},
\preferencesmethod{}erence comparisons on trajectories~\cite{christiano:2017},
and adversarial IRL with a state-only~(\airlstateonlymethod{}) and state-action~(\airlstateactionmethod{}) reward model~\citep{fu:2018}.
All models are trained using synthetic data from an oracle with access to the ground-truth; see section~\ref{sec:supp:learned-reward-models} for details.

We find \canonicalizeddistanceabbrevonly{} is robust to varying $\transitiondataset{}$ when comparing the learned reward models: the distance varies by less than $2\times$, and the ranking between the reward models is the same across \transitiondatasetname{}s.
By contrast, \nearestpointabbrevonly{} is highly sensitive to $\transitiondataset{}$: the ratio of \airlstateonlymethod{} ($817$) to \preferencesmethod{} ($8.51$) is $96:1$ under $\randompolicy{}$ but only $2:1$ ($2706:1333$) under $\expertpolicy{}$.
\returncorrelationabbrevonly{} lies somewhere in the middle: the ratio is $22:1$ ($549:24.9$) under $\randompolicy{}$ and $3:2$ ($523:360$) under $\expertpolicy{}$.

We evaluate the effect of pathological choices of \transitiondatasetname{} $\transitiondataset{}$ in Table~\ref{tab:supp:learned-reward-pathological}.
For example, \textbf{Ind} independently samples states and next states, giving physically impossible transitions, while \textbf{Jail} constrains rollouts to a tiny region excluding the goal.
We find that the ranking of \canonicalizeddistanceabbrevonly{} changes in only one distribution, whilst the ranking of \nearestpointabbrevonly{} changes in two cases and \returncorrelationabbrevonly{} changes in all cases.

However, we do find that \canonicalizeddistanceabbrevonly{} is sensitive to $\transitiondataset{}$ on $\bettergoalmethod{}$, a reward function we explicitly designed to break these methods.
\bettergoalmethod{} assigns a larger reward when close to a ``mirage'' state than when at the true goal, but is identical to \groundtruthmethod{} at all other points.
The ``mirage'' state is rarely visited by random exploration $\randompolicy{}$ as it is far away and on the opposite side of the wall from the agent.
The expert policy $\expertpolicy{}$ is even less likely to visit it, as it is not on or close to the optimal path to the goal.
As a result, the \canonicalizeddistanceabbrevonly{} distance from \bettergoalmethod{} to \groundtruthmethod{} (Table~\ref{tab:learned-reward}, bottom row) is small under $\randompolicy{}$ and $\expertpolicy{}$.

In general, any black-box method for assessing reward models -- including the rollout method -- only has predictive power on transitions visited during testing.
Fortunately, we can achieve a broad support over states with \mixturepolicy{}: it often navigates around the wall due to $\expertpolicy{}$, but strays from the goal thanks to $\randompolicy{}$.
As a result, \canonicalizeddistanceabbrevonly{} under \mixturepolicy{} correctly infers that \bettergoalmethod{} is far from the ground-truth \groundtruthmethod{}.

These empirical results support our theoretically inspired recommendation from section~\ref{sec:comparing}: ``in general, it is best to choose $\transitiondataset$ to have broad coverage over plausible transitions.''
Distributions such as $\expertpolicy{}$ are too narrow, assigning coverage only on a direct path from the initial state to the goal.
Very broad distributions such as \textbf{Ind} waste probability mass on impossible transitions like teleporting.
Distributions like \mixturepolicy{} strike the right balance between these extremes.

\subsection{Predicting policy performance from reward distance}
\label{sec:experiments:ground-truth}

We find that low distance from the ground-truth reward \groundtruthmethod{} (Table~\ref{tab:learned-reward}, center) predicts high \groundtruthmethod{} return (Table~\ref{tab:learned-reward}, right) of policies optimized for that reward.
Moreover, the distance is predictive of return not just in \pointmazetrain{} where the reward functions were trained and evaluated in, but also in the unseen variant \pointmazetest{}.
This is despite the two variants differing in the position of the wall, such that policies for \pointmazetrain{} run directly into the wall in \pointmazetest{}.

Both \regressionmethod{} and \preferencesmethod{} achieve very low distances at convergence, producing near-expert policy performance.
The \airlstateonlymethod{} and \airlstateactionmethod{} models have reward distances an order of magnitude higher and poor policy performance.
Yet intriguingly, the \emph{generator} policies for \airlstateonlymethod{} and \airlstateactionmethod{} -- trained simultaneously with the reward -- perform reasonably in \pointmazetrain{}.
This suggests the learned rewards are reasonable on the subset of transitions taken by the generator policy, yet fail to transfer to the different transitions taken by a policy being trained from scratch.

Figure~\ref{fig:supp:checkpoints-rewards} shows reward distance and policy regret during reward model training.
The lines all closely track each other, showing that the distance to \groundtruthmethod{} is highly correlated with policy regret for intermediate reward checkpoints as well as at convergence.
\regressionmethod{} and \preferencesmethod{} converge quickly to low distance and low regret, while \airlstateonlymethod{} and \airlstateactionmethod{} are slower and more unstable.

\section{Conclusion}

Our novel \canonicalizeddistanceabbrevonly{} distance compares reward functions directly, without training a policy.
We have proved it is a pseudometric, is bounded and invariant to equivalent rewards, and bounds the regret of optimal policies (Theorems~\ref{thm:canonicalized-distance-pseudometric},~\ref{thm:canonicalized-distance-properties} and~\ref{thm:epic-regret-bound-discrete}).
Empirically, we find \canonicalizeddistanceabbrevonly{} correctly infers zero distance between equivalent reward functions that the \nearestpointabbrevonly{} and \returncorrelationabbrevonly{} baselines wrongly consider dissimilar.
Furthermore, we find the distance of learned reward functions to the ground-truth reward predicts the return of policies optimized for the learned reward, even in unseen environments.

Standardized metrics are an important driver of progress in machine learning.
Unfortunately, traditional policy-based metrics do not provide any guarantees as to the fidelity of the learned reward function.
We believe the \canonicalizeddistanceabbrevonly{} distance will be a highly informative addition to the evaluation toolbox, and would encourage researchers to report \canonicalizeddistanceabbrevonly{} distance in addition to policy-based metrics.
Our implementation of \canonicalizeddistanceabbrevonly{} and our baselines, including a tutorial and documentation, are available at \projectsource{}.

\ifdefined\iclrfinaltrue
\subsection*{Acknowledgements}
Thanks to Sam Toyer, Rohin Shah, Eric Langlois, Siddharth Reddy and Stuart Armstrong for helpful discussions; to Miljan Martic for code-review; and to David Krueger, Matthew Rahtz, Rachel Freedman, Cody Wild, Alyssa Dayan, Adria Garriga, Jon Uesato, Zac Kenton and Alden Hung for feedback on drafts.
This work was supported by Open Philanthropy and the Leverhulme Trust.
\fi

\bibliography{refs}

\begin{thebibliography}{27}
\providecommand{\natexlab}[1]{#1}
\providecommand{\url}[1]{\texttt{#1}}
\expandafter\ifx\csname urlstyle\endcsname\relax
  \providecommand{\doi}[1]{doi: #1}\else
  \providecommand{\doi}{doi: \begingroup \urlstyle{rm}\Url}\fi

\bibitem[Akrour et~al.(2011)Akrour, Schoenauer, and Sebag]{akrour:2011}
Riad Akrour, Marc Schoenauer, and Michele Sebag.
\newblock Preference-based policy learning.
\newblock In \emph{Machine Learning and Knowledge Discovery in Databases},
  2011.

\bibitem[Amodei et~al.(2017)Amodei, Christiano, and Ray]{amodei:2017}
Dario Amodei, Paul Christiano, and Alex Ray.
\newblock Learning from human preferences, June 2017.
\newblock URL
  \url{https://openai.com/blog/deep-reinforcement-learning-from-human-preferences/}.

\bibitem[Bahdanau et~al.(2019)Bahdanau, Hill, Leike, Hughes, Hosseini, Kohli,
  and Grefenstette]{bahdanau:2019}
Dzmitry Bahdanau, Felix Hill, Jan Leike, Edward Hughes, Arian Hosseini,
  Pushmeet Kohli, and Edward Grefenstette.
\newblock Learning to understand goal specifications by modelling reward.
\newblock In \emph{ICLR}, 2019.

\bibitem[Brown et~al.(2019)Brown, Goo, Nagarajan, and Niekum]{brown:2019}
Daniel~S. Brown, Wonjoon Goo, Prabhat Nagarajan, and Scott Niekum.
\newblock Extrapolating beyond suboptimal demonstrations via inverse
  reinforcement learning from observations.
\newblock In \emph{ICML}, 2019.

\bibitem[Cabi et~al.(2019)Cabi, Colmenarejo, Novikov, Konyushkova, Reed, Jeong,
  Zolna, Aytar, Budden, Vecerik, Sushkov, Barker, Scholz, Denil, de~Freitas,
  and Wang]{cabi:2019}
Serkan Cabi, Sergio~Gómez Colmenarejo, Alexander Novikov, Ksenia Konyushkova,
  Scott Reed, Rae Jeong, Konrad Zolna, Yusuf Aytar, David Budden, Mel Vecerik,
  Oleg Sushkov, David Barker, Jonathan Scholz, Misha Denil, Nando de~Freitas,
  and Ziyu Wang.
\newblock Scaling data-driven robotics with reward sketching and batch
  reinforcement learning.
\newblock arXiv: 1909.12200v2 [cs.RO], 2019.

\bibitem[Christiano et~al.(2017)Christiano, Leike, Brown, Martic, Legg, and
  Amodei]{christiano:2017}
Paul~F Christiano, Jan Leike, Tom Brown, Miljan Martic, Shane Legg, and Dario
  Amodei.
\newblock Deep reinforcement learning from human preferences.
\newblock In \emph{NIPS}, pp.\  4299--4307, 2017.

\bibitem[Finn et~al.(2016)Finn, Levine, and Abbeel]{finn:2016}
Chelsea Finn, Sergey Levine, and Pieter Abbeel.
\newblock Guided cost learning: Deep inverse optimal control via policy
  optimization.
\newblock In \emph{ICML}, 2016.

\bibitem[Fu et~al.(2018)Fu, Luo, and Levine]{fu:2018}
Justin Fu, Katie Luo, and Sergey Levine.
\newblock Learning robust rewards with adverserial inverse reinforcement
  learning.
\newblock In \emph{ICLR}, 2018.

\bibitem[Hill et~al.(2018)Hill, Raffin, Ernestus, Gleave, Kanervisto, Traore,
  Dhariwal, Hesse, Klimov, Nichol, Plappert, Radford, Schulman, Sidor, and
  Wu]{stable-baselines:2018}
Ashley Hill, Antonin Raffin, Maximilian Ernestus, Adam Gleave, Anssi
  Kanervisto, Rene Traore, Prafulla Dhariwal, Christopher Hesse, Oleg Klimov,
  Alex Nichol, Matthias Plappert, Alec Radford, John Schulman, Szymon Sidor,
  and Yuhuai Wu.
\newblock {Stable Baselines}.
\newblock \url{https://github.com/hill-a/stable-baselines}, 2018.

\bibitem[Ibarz et~al.(2018)Ibarz, Leike, Pohlen, Irving, Legg, and
  Amodei]{ibarz:2018}
Borja Ibarz, Jan Leike, Tobias Pohlen, Geoffrey Irving, Shane Legg, and Dario
  Amodei.
\newblock Reward learning from human preferences and demonstrations in {Atari}.
\newblock In \emph{NeurIPS}, pp.\  8011--8023, 2018.

\bibitem[Lawson \& Hanson(1995)Lawson and Hanson]{lawson:1995}
Charles~L. Lawson and Richard~J. Hanson.
\newblock \emph{Solving Least Squares Problems}.
\newblock SIAM, 1995.

\bibitem[Michaud et~al.(2020)Michaud, Gleave, and Russell]{michaud:2020}
Eric~J. Michaud, Adam Gleave, and Stuart Russell.
\newblock Understanding learned reward functions.
\newblock In \emph{Proceedings of the Workshop on Deep Reinforcement Learning
  at NeurIPS}, 2020.

\bibitem[Ng \& Russell(2000)Ng and Russell]{ng:2000}
Andrew~Y. Ng and Stuart Russell.
\newblock Algorithms for inverse reinforcement learning.
\newblock In \emph{ICML}, 2000.

\bibitem[Ng et~al.(1999)Ng, Harada, and Russell]{ng:1999}
Andrew~Y. Ng, Daishi Harada, and Stuart Russell.
\newblock Policy invariance under reward transformations: theory and
  application to reward shaping.
\newblock In \emph{NIPS}, 1999.

\bibitem[OpenAI(2018)]{openai:2018}
OpenAI.
\newblock {OpenAI Five}.
\newblock \url{https://blog.openai.com/openai-five/}, 2018.

\bibitem[OpenAI et~al.(2019)OpenAI, Akkaya, Andrychowicz, Chociej, Litwin,
  McGrew, Petron, Paino, Plappert, Powell, Ribas, Schneider, Tezak, Tworek,
  Welinder, Weng, Yuan, Zaremba, and Zhang]{openai:2019}
OpenAI, Ilge Akkaya, Marcin Andrychowicz, Maciek Chociej, Mateusz Litwin, Bob
  McGrew, Arthur Petron, Alex Paino, Matthias Plappert, Glenn Powell, Raphael
  Ribas, Jonas Schneider, Nikolas Tezak, Jerry Tworek, Peter Welinder, Lilian
  Weng, Qiming Yuan, Wojciech Zaremba, and Lei Zhang.
\newblock Solving {Rubik's Cube} with a robot hand.
\newblock arXiv: 1910.07113v1 [cs.LG], 2019.

\bibitem[Ramachandran \& Amir(2007)Ramachandran and Amir]{ramachandran:2007}
Deepak Ramachandran and Eyal Amir.
\newblock Bayesian inverse reinforcement learning.
\newblock In \emph{IJCAI}, 2007.

\bibitem[Sadigh et~al.(2017)Sadigh, Dragan, Sastry, and Seshia]{sadigh:2017}
Dorsa Sadigh, Anca~D. Dragan, S.~Shankar Sastry, and Sanjit~A. Seshia.
\newblock Active preference-based learning of reward functions.
\newblock In \emph{RSS}, July 2017.

\bibitem[Schulman et~al.(2017)Schulman, Wolski, Dhariwal, Radford, and
  Klimov]{schulman:2017}
John Schulman, Filip Wolski, Prafulla Dhariwal, Alec Radford, and Oleg Klimov.
\newblock Proximal policy optimization algorithms.
\newblock arXiv:1707.06347v2 [cs.LG], 2017.

\bibitem[Silver et~al.(2016)Silver, Huang, Maddison, Guez, Sifre, van~den
  Driessche, Schrittwieser, Antonoglou, Panneershelvam, Lanctot, Dieleman,
  Grewe, Nham, Kalchbrenner, Sutskever, Lillicrap, Leach, Kavukcuoglu, Graepel,
  and Hassabis]{silver:2016}
David Silver, Aja Huang, Chris~J. Maddison, Arthur Guez, Laurent Sifre, George
  van~den Driessche, Julian Schrittwieser, Ioannis Antonoglou, Veda
  Panneershelvam, Marc Lanctot, Sander Dieleman, Dominik Grewe, John Nham, Nal
  Kalchbrenner, Ilya Sutskever, Timothy Lillicrap, Madeleine Leach, Koray
  Kavukcuoglu, Thore Graepel, and Demis Hassabis.
\newblock Mastering the game of {Go} with deep neural networks and tree search.
\newblock \emph{Nature}, 529\penalty0 (7587):\penalty0 484--489, 2016.

\bibitem[Singh \& Yee(1994)Singh and Yee]{singh:1994}
Satinder~P. Singh and Richard~C. Yee.
\newblock An upper bound on the loss from approximate optimal-value functions.
\newblock \emph{Machine Learning}, 16\penalty0 (3):\penalty0 227–233,
  September 1994.

\bibitem[Vecerik et~al.(2019)Vecerik, Sushkov, Barker, Roth\"{o}rl, Hester, and
  Scholz]{vecerik:2019}
Mel Vecerik, Oleg Sushkov, David Barker, Thomas Roth\"{o}rl, Todd Hester, and
  Jon Scholz.
\newblock A practical approach to insertion with variable socket position using
  deep reinforcement learning.
\newblock In \emph{ICRA}, 2019.

\bibitem[Vinyals et~al.(2019)Vinyals, Babuschkin, Czarnecki, Mathieu, Dudzik,
  Chung, Choi, Powell, Ewalds, Georgiev, Oh, Horgan, Kroiss, Danihelka, Huang,
  Sifre, Cai, Agapiou, Jaderberg, Vezhnevets, Leblond, Pohlen, Dalibard,
  Budden, Sulsky, Molloy, Paine, Gulcehre, Wang, Pfaff, Wu, Ring, Yogatama,
  W{\"u}nsch, McKinney, Smith, Schaul, Lillicrap, Kavukcuoglu, Hassabis, Apps,
  and Silver]{vinyals:2019}
Oriol Vinyals, Igor Babuschkin, Wojciech~M. Czarnecki, Micha{\"e}l Mathieu,
  Andrew Dudzik, Junyoung Chung, David~H. Choi, Richard Powell, Timo Ewalds,
  Petko Georgiev, Junhyuk Oh, Dan Horgan, Manuel Kroiss, Ivo Danihelka, Aja
  Huang, Laurent Sifre, Trevor Cai, John~P. Agapiou, Max Jaderberg,
  Alexander~S. Vezhnevets, R{\'e}mi Leblond, Tobias Pohlen, Valentin Dalibard,
  David Budden, Yury Sulsky, James Molloy, Tom~L. Paine, Caglar Gulcehre, Ziyu
  Wang, Tobias Pfaff, Yuhuai Wu, Roman Ring, Dani Yogatama, Dario W{\"u}nsch,
  Katrina McKinney, Oliver Smith, Tom Schaul, Timothy Lillicrap, Koray
  Kavukcuoglu, Demis Hassabis, Chris Apps, and David Silver.
\newblock Grandmaster level in {StarCraft II} using multi-agent reinforcement
  learning.
\newblock \emph{Nature}, 575\penalty0 (7782):\penalty0 350--354, 2019.

\bibitem[Wang et~al.(2020)Wang, Gleave, and Toyer]{imitation:2020}
Steven Wang, Adam Gleave, and Sam Toyer.
\newblock imitation: implementations of inverse reinforcement learning and
  imitation learning algorithms.
\newblock \url{https://github.com/humancompatibleai/imitation}, 2020.

\bibitem[Wilson et~al.(2012)Wilson, Fern, and Tadepalli]{wilson:2012}
Aaron Wilson, Alan Fern, and Prasad Tadepalli.
\newblock A {Bayesian} approach for policy learning from trajectory preference
  queries.
\newblock In \emph{NIPS}, 2012.

\bibitem[Ziebart et~al.(2008)Ziebart, Maas, Bagnell, and Dey]{ziebart:2008}
Brian~D. Ziebart, Andrew Maas, J.~Andrew Bagnell, and Anind~K. Dey.
\newblock Maximum entropy inverse reinforcement learning.
\newblock In \emph{AAAI}, 2008.

\bibitem[Ziegler et~al.(2019)Ziegler, Stiennon, Wu, Brown, Radford, Amodei,
  Christiano, and Irving]{ziegler:2019}
Daniel~M. Ziegler, Nisan Stiennon, Jeffrey Wu, Tom~B. Brown, Alec Radford,
  Dario Amodei, Paul Christiano, and Geoffrey Irving.
\newblock Fine-tuning language models from human preferences.
\newblock arXiv: 1909.08593v2 [cs.CL], 2019.

\end{thebibliography}
\bibliographystyle{iclr2021_conference}

\onecolumn
\appendix
\counterwithin{table}{section}
\counterwithin{figure}{section}

\section{Supplementary material}

\subsection{Approximation Procedures}

\subsubsection{Sample-based approximation for \canonicalizeddistanceabbrevonly{} distance}
\label{sec:supp:approx:canonicalized-distance}

We approximate \canonicalizeddistanceabbrevonly{} distance (definition~\ref{defn:canonicalized-distance}) by estimating Pearson distance on a set of samples, canonicalizing the reward on-demand.
Specifically, we sample a batch $B_V$ of $N_V$ samples from the \transitiondatasetname{} $\transitiondataset{}$, and a batch $B_M$ of $N_M$ samples from the joint state and action distributions $\statedistribution \times \actiondistribution$.
For each $(s,a,s') \in B_V$, we approximate the canonically shaped rewards (definition~\ref{defn:canonically-shaped-reward}) by taking the mean over $B_M$:

\begin{alignat}{2}
\canonical{\reward{}}(\state, \action, \nextstate) &= \reward{}(\state,\action,\nextstate) &&+ \expectation\left[\discount\reward{}(\nextstate, A, S') - \reward{}(\state, A, S') - \discount\reward{}(S, A, S')\right] \\
&\approx \reward{}(\state,\action,\nextstate) &&+ \frac{\gamma}{N_M} \sum_{(x, u) \in B_M} \reward{}(\nextstate, u, x) \\
& &&- \frac{1}{N_M} \sum_{(x,u) \in B_M} \reward{}(\state, u, x) - c.
\end{alignat}

We drop the constant $c$ from the approximation since it does not affect the Pearson distance; it can also be estimated in $O(N_M^2)$ time by $c = \frac{\gamma}{N_M^2} \sum_{(x, \cdot) \in B_M} \sum_{(x', u) \in B_M} \reward{}(x, u, x')$.
Finally, we compute the Pearson distance between the approximate canonically shaped rewards on the batch of samples $B_V$, yielding an $O(N_V N_M)$ time algorithm.

\subsubsection{Optimization-based approximation for \nearestpointabbrevonly{} distance}
\label{sec:supp:approx:nearest-point}

$\nearestpoint{}(\firstreward{}, \secondreward{})$ (section~\ref{sec:baselines:nearest-point}) is defined as the infimum of \directdistancename{} distance over an infinite set of equivalent reward functions $\reward{} \equivreward \firstreward{}$.
We approximate this using gradient descent on the reward model
\begin{equation}
R_{\nu,c,w}(\state,\action,\nextstate) = \exp(\nu) \firstreward(\state,\action,\nextstate) + c + \gamma \potential_w(s') - \potential_w(s),
\end{equation}
where $\nu, c \in \mathbb{R}$ are scalar weights and $w$ is a vector of weights parameterizing a deep neural network $\potential_w$.
The constant $c \in \mathbb{R}$ is unnecessary if $\potential_w$ has a bias term, but its inclusion simplifies the optimization problem.

We optimize $\nu,c,w$ to minimize the mean of the cost
\begin{equation}
J(\nu, c, w)(\state,\action,\nextstate) = \left\Vert R_{\nu,c,w}(\state,\action,\nextstate), \secondreward{}(\state,\action,\nextstate)\right\Vert^p
\end{equation}
on samples $(\state,\action,\nextstate)$ from a \transitiondatasetname{} $\transitiondataset{}$.
Note
\begin{equation}
\expectation_{(S,A,S') \sim \transitiondataset{}}\left[J(\nu, c, w)(S, A, S')\right]^{1/p} = \directdistance{}(R_{\nu,c,w}, \secondreward)
\end{equation}
upper bounds the true \nearestpointabbrevonly{} distance since $R_{\nu,c,w} \equivreward{} \firstreward{}$.

We found empirically that $\nu$ and $c$ need to be initialized close to their optimal values for gradient descent to reliably converge.
To resolve this problem, we initialize the affine parameters to $\nu \leftarrow \log \lambda$ and $c$ found by:
\begin{equation}
\argmin_{\lambda \geq 0, c \in \mathbb{R}} \underset{s,a,s' \sim \transitiondataset{}}{\expectation{}} \left(\lambda \firstreward(\state, \action, \nextstate) + c - \secondreward(\state, \action, \nextstate)\right)^2.
\end{equation}
We use the active set method of \citet{lawson:1995} to solve this constrained least-squares problem.
These initial affine parameters minimize the \directdistancename{} distance $\directdistance(R_{\nu,c,0}(\state,\action,\nextstate), \secondreward{}(\state,\action,\nextstate))$ when $p=2$ with the potential fixed at $\potential_0(s) = 0$.

\subsubsection{Confidence Intervals}
\label{sec:supp:approx:ci}
We report confidence intervals to help measure the degree of error introduced by the approximations.
Since approximate distances may not be normally distributed, we use bootstrapping to produce a distribution-free confidence interval.
For \canonicalizeddistanceabbrevonly{}, \nearestpointabbrevonly{} and Episode Return (sometimes reported as regret rather than return), we compute independent approximate distances or returns over different seeds, and then compute a bootstrapped confidence interval for each seed.
We use $30$ seeds for \canonicalizeddistanceabbrevonly{}, but only $9$ seeds for computing Episode Return and $3$ seeds for \nearestpointabbrevonly{} due to their greater computational requirements.
In \returncorrelationabbrevonly{}, computing the distance is very fast, so we instead apply bootstrapping to the collected \emph{episodes}, computing the \returncorrelationabbrevonly{} distance for each bootstrapped episode sample.

\subsection{Experiments}

\subsubsection{Hyperparameters for Approximate Distances}
\label{sec:supp:distances-hyperparameters}

\begin{table}
	\centering
	\caption{Summary of hyperparameters and distributions used in experiments. The uniform random \transitiondatasetname{} $\uniformtransitiondataset{}$ samples states and actions uniformly at random, and samples the next state from the transition dynamics. Random policy $\randompolicy{}$ takes uniform random actions. The synthetic expert policy $\expertpolicy{}$ was trained with PPO on the ground-truth reward. \textbf{Mix}ture samples actions from either $\randompolicy{}$ or $\expertpolicy{}$, switching between them at each time step with probability $0.05$. Warmstart Size is the size of the dataset used to compute initialization parameters described in section~\ref{sec:supp:approx:nearest-point}.}
	\capitalizetitle[q]{\transitiondatasetname}%
	\begin{tabular}{@{}lll@{}}
		\toprule
		\textbf{Parameter} & \textbf{Value} & \textbf{In experiment} \\
		\midrule
		\multirow{3}*{\thestring\ $\transitiondataset{}$} & Random transitions $\uniformtransitiondataset{}$ & \env{GridWorld} \\
		& Rollouts from $\randompolicy{}$ & \env{PointMass}, \env{HalfCheetah}, \env{Hopper} \\
		& $\randompolicy{}$, $\expertpolicy{}$ and \textbf{Mix}ture & \env{PointMaze} \\
		Bootstrap Samples & \num{10000} & All \\
		Discount $\discount{}$ & \num{0.99} & All \\
		\midrule
		\textbf{\canonicalizeddistanceabbrevonly{}} & & \\
		State Distribution $\statedistribution{}$ & Marginalized from $\transitiondataset{}$ & All \\
		Action Distribution $\actiondistribution{}$ & Marginalized from $\transitiondataset{}$ & All \\
		Seeds & \num{30} & All \\
		Samples $N_V$ & \num{32768} & All \\
		Mean Samples $N_M$ & \num{32768} & All \\
		\midrule
		\textbf{\nearestpointabbrevonly{}} & & \\
		Seeds & 3 & All \\
		Total Time Steps & $\num{1e6}$ & All \\
		Optimizer & Adam & All \\
		Learning Rate & \num{1e-2} & All \\
		Batch Size & \num{4096} & All\\
		Warmstart Size & \num{16386} & All \\
		Loss $\loss{}$ & $\loss(x,y) = (x - y)^2$ & All \\
		\midrule
		\textbf{\returncorrelationabbrevonly{}} & & \\
		Episodes & \num{131072} & All \\
		\midrule
		\textbf{Episode Return} & & \\
		Seeds & \num{9} & All \\
		\multicolumn{3}{@{}l}{See Table~\ref{tab:ppo-hyperparams} for the policy training hyperparameters} \\
		\bottomrule
	\end{tabular}
	\label{tab:experimental-config}
\end{table}

Table~\ref{tab:experimental-config} summarizes the hyperparameters and distributions used to compute the distances between reward functions.
Most parameters are the same across all environments.
We use a \transitiondatasetname{} of uniform random transitions $\uniformtransitiondataset{}$ in the simple \env{GridWorld} environment with known determinstic dynamics.
In other environments, the \transitiondatasetname{} is sampled from rollouts of a policy.
We use a random policy $\randompolicy{}$ for \env{PointMass}, \env{HalfCheetah} and \env{Hopper} in the hand-designed reward experiments (section~\ref{sec:experiments:hand-designed}).
In \env{PointMaze}, we compare three \transitiondatasetname{}s (section~\ref{sec:experiments:visitation-sensitivity}) induced by rollouts of $\randompolicy{}$, an expert policy $\expertpolicy{}$ and a \textbf{Mix}ture of the two policies, sampling actions from either $\randompolicy{}$ or $\expertpolicy{}$ and switching between them with probability $0.05$ per time step.

\subsubsection{Training Learned Reward Models}
\label{sec:supp:learned-reward-models}

\begin{table}[p]
	\centering
	\caption{Hyperparameters for proximal policy optimisation (PPO)~\citep{schulman:2017}. We used the implementation and default hyperparameters from \citet{stable-baselines:2018}. PPO was used to train expert policies on ground-truth reward and to optimize learned reward functions for evaluation.}
	\begin{tabular}{lll}
		\toprule
		\textbf{Parameter} & \textbf{Value} & \textbf{In environment} \\
		\midrule
		Total Time Steps & \num{1e6} & All \\
		Seeds & 9 & All \\
		Batch Size & \num{4096} & All \\
		Discount $\discount{}$ & \num{0.99} & All \\
		Entropy Coefficient & \num{0.01} & All \\
		Learning Rate & \num{3e-4} & All \\
		Value Function Coefficient & \num{0.5} & All \\
		Gradient Clipping Threshold & \num{0.5} & All \\
		Ratio Clipping Thrsehold & \num{0.2} & All \\
		Lambda (GAE) & \num{0.95} & All \\
		Minibatches & \num{4} & All \\
		Optimization Epochs & \num{4} & All \\
		Parallel Environments & \num{8} & All \\
		\bottomrule
	\end{tabular}
	\label{tab:ppo-hyperparams}
\end{table}

\begin{table}
	\centering
	\caption{Hyperparameters for adversarial inverse reinforcement learning (AIRL) used in \citet{imitation:2020}.}
	\begin{tabular}{@{}ll@{}}
		\toprule
		\textbf{Parameter} & \textbf{Value} \\
		\midrule
		RL Algorithm & PPO~\cite{schulman:2017} \\
		Total Time Steps & \num{1000000} \\
		Discount $\discount{}$ & \num{0.99} \\
		Demonstration Time Steps & \num{100000} \\
		Generator Batch Size & \num{2048} \\
		Discriminator Batch Size & \num{50} \\
		Entropy Weight & \num{1.0} \\
		Reward Function Architecture & MLP, two 32-unit hidden layers \\
		Potential Function Architecture & MLP, two 32-unit hidden layers \\
		\bottomrule
	\end{tabular}
	\label{tab:airl-hyperparams}
\end{table}

\begin{table}
	\centering
	\caption{Hyperparameters for preference comparison used in our implementation of \citet{christiano:2017}.}
	\begin{tabular}{@{}lll@{}}
		\toprule
		\textbf{Parameter} & \textbf{Value} & \textbf{Range Tested} \\
		\midrule
		Total Time Steps & \num{5e6} & $[\num{1}, \num{10e6}]$ \\
		Batch Size & \num{10000} & $[\num{500}, \num{250000}]$ \\
		Trajectory Length & \num{5} & $[\num{1},\num{100}]$ \\
		Learning Rate & \num{1e-2} & $[\num{1e-4}, \num{1e-1}]$ \\
		Discount $\discount{}$ & \num{0.99} & \\
		Reward Function Architecture & MLP, two 32-unit hidden layers \\
		Output L2 Regularization Weight & \num{1e-4} \\
		\bottomrule
	\end{tabular}
	\label{tab:drlhp-hyperparams}
\end{table}

\begin{table}
	\centering
	\caption{Hyperparameters for regression used in our implementation of \citet[\emph{target} method from section~3.3]{christiano:2017}.}
	\begin{tabular}{@{}lll@{}}
		\toprule
		\textbf{Parameter} & \textbf{Value} & \textbf{Range Tested} \\
		\midrule
		Total Time Steps & \num{10e6} & $[\num{1}, \num{20e6}]$ \\
		Batch Size & \num{4096} & $[\num{256}, \num{16384}]$ \\
		Learning Rate & \num{2e-2} & $[\num{1e-3}, \num{1e-1}]$ \\
		Discount $\discount{}$ & \num{0.99} & \\
		Reward Function Architecture & MLP, two 32-unit hidden layers \\
		\bottomrule
	\end{tabular}
	\label{tab:regression-hyperparams}
\end{table}

For the experiments on learned reward functions (sections~\ref{sec:experiments:ground-truth}~and~\ref{sec:experiments:visitation-sensitivity}), we trained reward models using adversarial inverse reinforcement learning~(AIRL;~\citealp{fu:2018}), preference comparison~\cite{christiano:2017} and by regression onto the ground-truth reward~\citep[\emph{target} method from][section~3.3]{christiano:2017}.
For AIRL, we use an existing open-source implementation~\cite{imitation:2020}.
We developed new implementations for preference comparison and regression, available at \projectsource{}.
We also use the RL algorithm proximal policy optimization~(PPO;~\citealp{schulman:2017}) on the ground-truth reward to train expert policies to provide demonstrations for AIRL.
We use 9 seeds, taking rollouts from the seed with the highest ground-truth return.

Our hyperparameters for PPO in Table~\ref{tab:ppo-hyperparams} were based on the defaults in Stable Baselines~\cite{stable-baselines:2018}.
We only modified the batch size and learning rate, and disabled value function clipping to match the original PPO implementation.

Our AIRL hyperparameters in Table~\ref{tab:airl-hyperparams} likewise match the defaults, except for increasing the total number of timesteps to $10^6$.
Due to the high variance of AIRL, we trained 5 seeds, selecting the one with the highest ground-truth return.
While this does introduce a positive bias for our AIRL results, in spite of this AIRL performed worse in our tests than other algorithms.
Moreover, the goal in this paper is to evaluate distance metrics, not reward learning algorithms.

For preference comparison we performed a sweep over batch size, trajectory length and learning rate to decide on the hyperparameters in Table~\ref{tab:drlhp-hyperparams}.
Total time steps was selected once diminishing returns were observed in loss curves.
The exact value of the regularization weight was found to be unimportant, largely controlling the scale of the output at convergence.

Finally, for regression we performed a sweep over batch size, learning rate and total time steps to decide on the hyperparameters in Table~\ref{tab:regression-hyperparams}.
We found batch size and learning rate to be relatively unimportant with many combinations performing well, but regression was found to converge slowly but steadily requiring a relatively large $\num{10e6}$ time steps for good performance in our environments.

\ag{Expand on this section a little since ICML R2 wanted more details (although probably never read this section.)}
All algorithms are trained on synthetic data generated from the ground-truth reward function.
AIRL is provided with a large demonstration dataset of \num{100000} time steps from an expert policy trained on the ground-truth reward (see Table~\ref{tab:airl-hyperparams}).
In preference comparison and regression, each batch is sampled afresh from the \transitiondatasetname{} specified in Table~\ref{tab:experimental-config} and labeled according to the ground-truth reward.

\subsubsection{Computing infrastructure}
Experiments were conducted on a workstation (Intel i9-7920X CPU with 64 GB of RAM), and a small number of \texttt{r5.24xlarge} AWS VM instances, with 48 CPU cores on an Intel Skylake processor and 768 GB of RAM.
It takes less than three weeks of compute on a single \texttt{r5.24xlarge} instance to run all the experiments described in this paper.

\subsubsection{Comparing hand-designed reward functions}
\label{sec:supp:hand-designed}

We compute distances between hand-designed reward functions in four environments: \env{GridWorld}, \env{PointMass}, \env{HalfCheetah} and \env{Hopper}.
The reward functions for \env{GridWorld} are described in Figure~\ref{fig:supp:gridworld-definition}, and the distances are reported in Figure~\ref{fig:supp:hand-designed-gridworld}.
We report the approximate distances and confidence intervals between reward functions in the other environments in Figures~\ref{fig:supp:hand-designed-point-mass},~\ref{fig:supp:hand-designed-half-cheetah}~and~\ref{fig:supp:hand-designed-hopper}.

We find the (approximate) \canonicalizeddistanceabbrevonly{} distance closely matches our intuitions for similarity between the reward functions.
\nearestpointabbrevonly{} often produces similar results to \canonicalizeddistanceabbrevonly{}, but unfortunately is dogged by optimization error.
This is particularly notable in higher-dimensional environments like \env{HalfCheetah} and \env{Hopper}, where the \nearestpointabbrevonly{} distance often exceeds the theoretical upper bound of $1.0$ and the confidence interval width is frequently larger than $0.2$.

By contrast, \returncorrelationabbrevonly{} distance generally has a tight confidence interval, but systematically fails in the presence of shaping.
For example, it confidently assigns large distances between \emph{equivalent} reward pairs in \env{PointMass} such as \sparse\nocontrolpenalty{}-\dense\nocontrolpenalty{}.
However, \returncorrelationabbrevonly{} produces reasonable results in \env{HalfCheetah} and \env{Hopper} where rewards are all similarly shaped.
In fact, \returncorrelationabbrevonly{} picks up on a detail in \env{Hopper} that \canonicalizeddistanceabbrevonly{} misses: whereas \canonicalizeddistanceabbrevonly{} assigns a distance of around $0.71$ between all rewards of different types (running vs backflipping), \returncorrelationabbrevonly{} assigns lower distances when the rewards are in the same direction (forward or backward).
Given this, \returncorrelationabbrevonly{} may be attractive in some circumstances, especially given the ease of implementation.
However, we would caution against using it in isolation due to the likelihood of misleading results in the presence of shaping.

\subsubsection{Comparing learned reward functions}

Previously, we reported the mean approximate distance from a ground-truth reward of four learned reward models in \env{PointMaze} (Table~\ref{tab:learned-reward}).
Since these distances are approximate, we report 95\% lower and upper bounds computed via bootstrapping in Table~\ref{tab:supp:learned-reward}.
We also include the relative difference of the upper and lower bounds from the mean, finding the relative difference to be fairly consistent across reward models for a given algorithm and \transitiondatasetname{} pair.
The relative difference is less than 1\% for all \canonicalizeddistanceabbrevonly{} and \returncorrelationabbrevonly{} distances.
However, \nearestpointabbrevonly{} confidence intervals can be as wide as 50\%: this is due to the method's high variance, and the small number of seeds we were able to run because of the method's computational expense.

\subsubsection{Runtime of Distance Metrics}
\label{sec:supp:experiments:runtime}

\begin{table}
	\centering
	\begin{tabular}{lllllll}
	\toprule
	\multicolumn{1}{c}{\textbf{Distance}} & \multicolumn{1}{c}{\textbf{Wall-Clock}} & \multicolumn{1}{c}{\textbf{Environment}} & \multicolumn{1}{c}{\textbf{Reward}} & \multicolumn{1}{c}{\textbf{\# of }} & \multicolumn{2}{c}{\textbf{95\% CI Width}} \\
	\multicolumn{1}{c}{\textbf{Metric}} & \multicolumn{1}{c}{\textbf{Time}} & \multicolumn{1}{c}{\textbf{Time Steps}} & \multicolumn{1}{c}{\textbf{Queries}} & \multicolumn{1}{c}{\textbf{Seeds}} & \multicolumn{1}{c}{\textbf{Max}} & \multicolumn{1}{c}{\textbf{Mean}} \\
	\midrule
	\canonicalizeddistanceabbrevonly{} Quick & \SI{17}{s} & \num{8192} & $\numformat{1.67e7}$ & $3$ & $\num{0.02304}$ & $\num{0.00860}$ \\
	\canonicalizeddistanceabbrevonly{} & \SI{6738}{s} & \num{65536} & $\numformat{1.07e9}$ & $30$ & $\num{0.00558}$ & $\num{0.00234}$ \\
	\nearestpointabbrevonly{} & \SI{29769}{s} & $\numformat{7.5e7}$ & $\numformat{7.5e7}$ & $3$ & $\num{0.31591}$ & $\num{0.06620}$ \\
	\returncorrelationabbrevonly{} & \SI{1376}{s} & $\numformat{6.55e6}$ & $\numformat{6.55e6}$ & --- & $\num{0.01581}$ & $\num{0.00533}$ \\
	RL (PPO) & \SI{14745}{s} & $\numformat{7.5e7}$ & $\numformat{7.5e7}$ & $3$ & \quad\ --- & \quad\ --- \\
	\bottomrule
	\end{tabular}
	\caption{Time and resources taken by different metrics to perform $25$ distance comparisons on \env{PointMass}, and the confidence interval widths obtained (smaller is better). Methods \emph{\canonicalizeddistanceabbrevonly{}}, \emph{\nearestpointabbrevonly{}} and \emph{\returncorrelationabbrevonly{}} correspond to Figures~\ref{fig:experiments:hand-designed}(a), (b) and (c) respectively. \emph{\canonicalizeddistanceabbrevonly{} Quick} is an abbreviated version with fewer samples. \emph{RL (PPO)} is estimated from the time taken using PPO to train a single policy (16m:23s) until convergence ($10^6$ time steps). \canonicalizeddistanceabbrevonly{} samples $N_M + N_V$ time steps from the environment and performs $N_M N_V$ reward queries. In \emph{\canonicalizeddistanceabbrevonly{} Quick}, $N_M = N_V = 4096$; in $\emph{\canonicalizeddistanceabbrevonly{}}$, $N_M = N_V = 302768$. Other methods query the reward once per environment time step.}
	\label{table:empirical-runtimes}
\end{table}

We report the empirical runtime for \canonicalizeddistanceabbrevonly{} and baselines in Table~\ref{table:empirical-runtimes}, performing $25$ pairwise comparisons across $5$ reward functions in \env{PointMass}.
These comparisons were run on an unloaded machine running Ubuntu 20.04 (kernel 5.4.0-52) with an Intel i9-7920X CPU and 64 GB of RAM.
We report sequential runtimes: runtimes for all methods could be decreased further by parallelizing across seeds.
The algorithms were configured to use $8$ parallel environments for sampling.
Inference and training took place on CPU.
All methods used the same TensorFlow configuration, parallelizing operations across threads both within and between operations.
We found GPUs offered no performance benefit in this setting, and in some cases even increased runtime.
This is due to the fixed cost of CPU-GPU communication, and the relatively small size of the observations.

We find that in just $17$ seconds \canonicalizeddistanceabbrevonly{} can provide results with a 95\% confidence interval $<0.023$, an order of magnitude tighter than \nearestpointabbrevonly{} running for over $8$ hours.
Training policies for all learned rewards in this environment using PPO is comparatively slow, taking over $4$ hours even with only $3$ seeds.
While \returncorrelationabbrevonly{} is relatively fast, it takes a large number of samples to achieve tight confidence intervals.
Moreover, since \env{PointMass} has stochastic initial states, \returncorrelationabbrevonly{} can take on arbitrary values under shaping, as discussed in sections~\ref{sec:baselines:return-correlation} and~\ref{sec:experiments:ground-truth}.

\begin{figure*}
	\vspace*{-0.4em}
	\centering
	\includegraphics{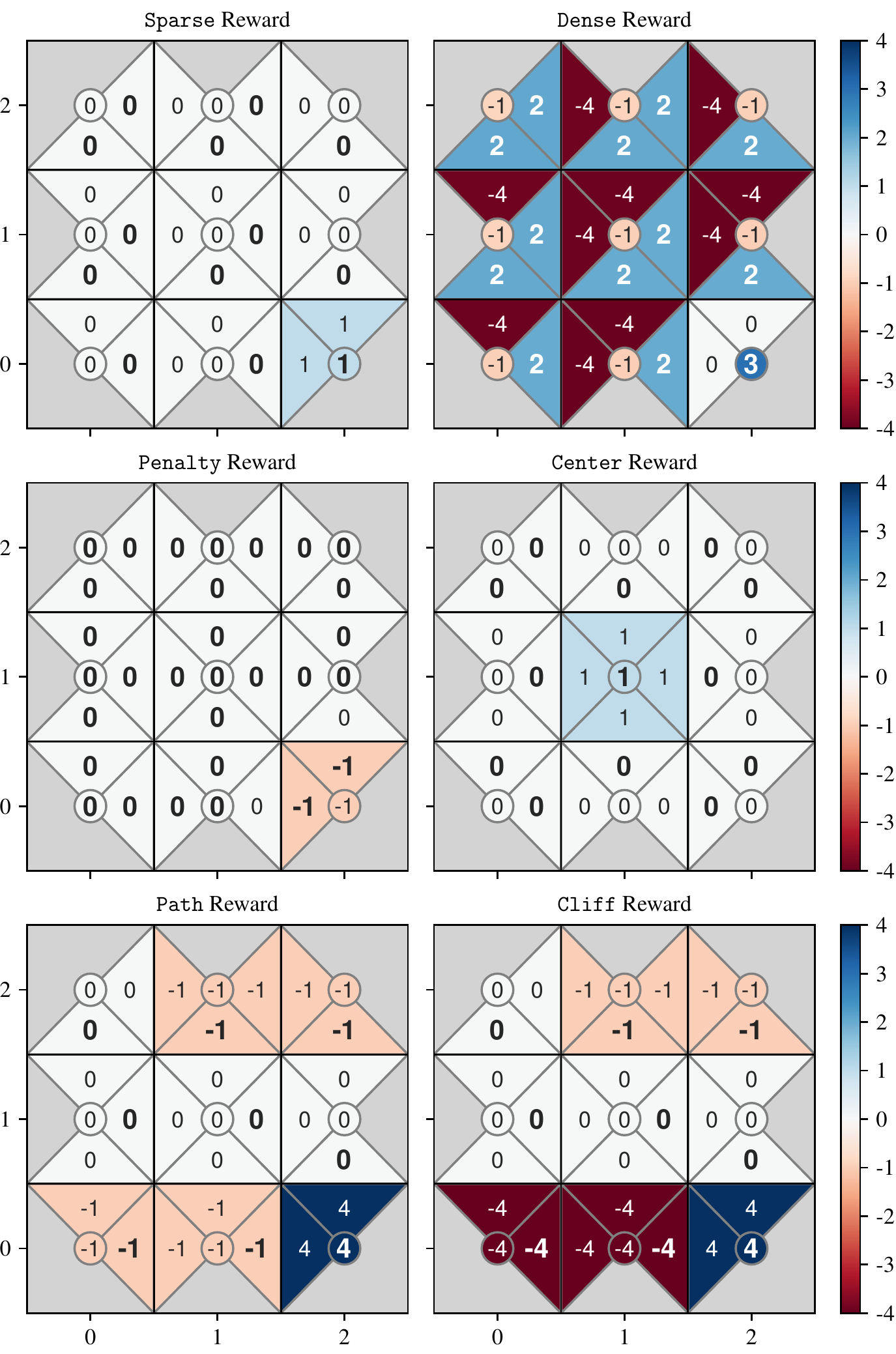}
	\caption{Heatmaps of reward functions $R(s,a,s')$ for a $3\times 3$ deterministic gridworld.
		$R(s, \mathrm{stay}, s)$ is given by the central circle in cell $s$.
		$R(s, a, s')$ is given by the triangular wedge in cell $s$ adjacent to cell $s'$ in direction $a$.
		Optimal action(s) (for infinite horizon, discount $\gamma = 0.99$) have bold labels against a hatched background.
		See Figure~\ref{fig:supp:hand-designed-gridworld} for the distance between all reward pairs.}
	\label{fig:supp:gridworld-definition}
\end{figure*}

\begin{figure*}
	\ag{TODO: add results for ERC too? also reformat so NPEC actually fits, may want to use compact method similar to non-gridworlds.}
	\centering
	\vspace*{-2em}
	\begin{subfigure}{\textwidth}
		\includegraphics{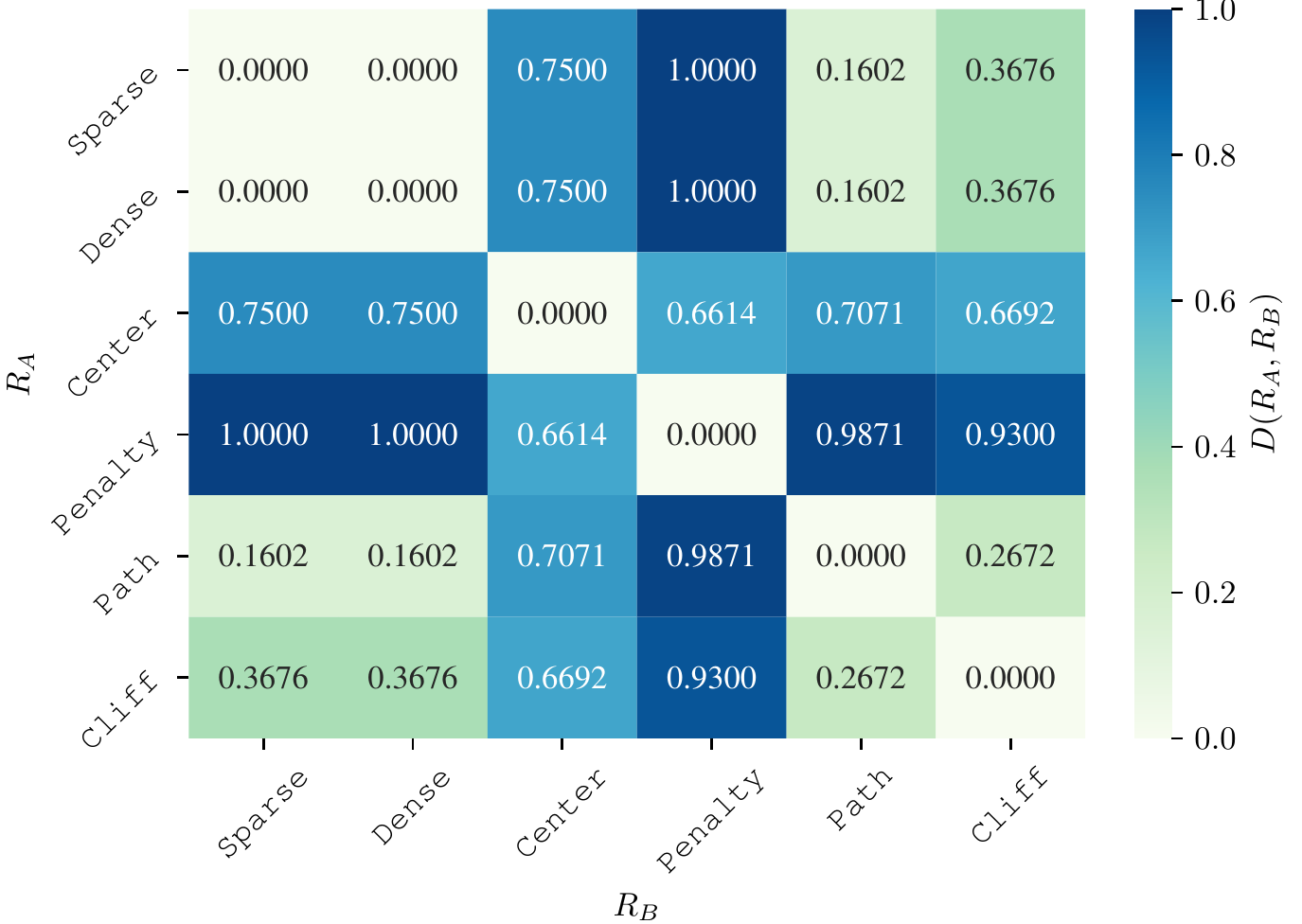}
		\caption{\canonicalizeddistanceabbrevonly{}}
	\end{subfigure}
	\begin{subfigure}{\textwidth}
		\includegraphics{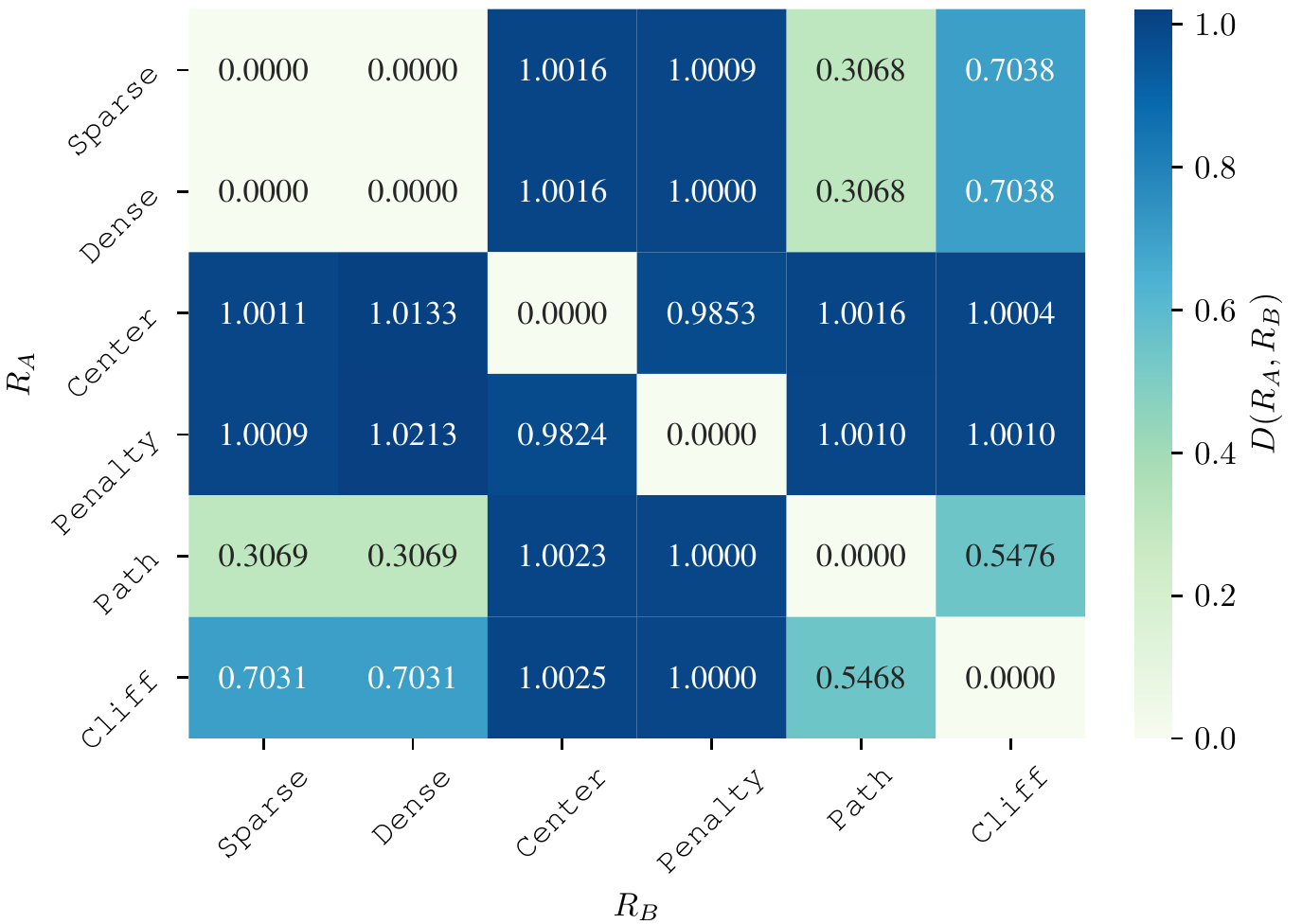}
		\caption{\nearestpointabbrevonly{}}
	\end{subfigure}
	\caption{Distances (\canonicalizeddistanceabbrevonly{}, top; \nearestpointabbrevonly{}, bottom) between hand-designed reward functions for the $3\times 3$ deterministic \env{Gridworld} environment. \canonicalizeddistanceabbrevonly{} and \nearestpointabbrevonly{} produce similar results, but \canonicalizeddistanceabbrevonly{} more clearly discriminates between rewards whereas \nearestpointabbrevonly{} distance tends to saturate. For example, the \nearestpointabbrevonly{} distance from \sparsepenalty{} to other rewards lies in the very narrow $[0.98,1.0]$ range, whereas \canonicalizeddistanceabbrevonly{} uses the wider $[0.66, 1.0]$ range. See Figure~\ref{fig:supp:gridworld-definition} for definitions of each reward. Distances are computed using tabular algorithms. We do not report confidence intervals since these algorithms are deterministic and exact up to floating point error.}
	\label{fig:supp:hand-designed-gridworld}
\end{figure*}

\begin{figure*}
	\centering
	\compactheatmap{point_mass}{middle}{ Median}
	\compactheatmap{point_mass}{width}{ CI Width}
	\caption[]{\envheatmapcaption{PointMass} \pointmasscaption{} \cicaption{}}
	\label{fig:supp:hand-designed-point-mass}
\end{figure*}

\begin{figure*}
	\centering
	\compactheatmap{half_cheetah}{middle}{ Median}
	\compactheatmap{half_cheetah}{width}{ CI Width}
	\caption[]{\envheatmapcaption{HalfCheetah}
	\textbf{Key}: \running{} is a reward proportional to the change in center of mass and moving \emph{forward} is rewarded when \running{} to the right, and moving \emph{backward} is rewarded when \backward{\running{}} to the left.
	\controlpenalty{} quadratic control penalty, \nocontrolpenalty{} no control penalty.
	\cicaption{}}
	\label{fig:supp:hand-designed-half-cheetah}
\end{figure*}

\begin{figure*}
	\centering
	\compactheatmap{hopper}{middle}{ Median}
	\compactheatmap{hopper}{width}{ CI Width}
	\caption[]{\envheatmapcaption{Hopper}
	\textbf{Key}: \running{} is a reward proportional to the change in center of mass and \backflipping{} is the backflip reward defined in \citet[footnote]{amodei:2017}.
	Moving \emph{forward} is rewarded when \running{} or \backflipping{} is to the right, and moving \emph{backward} is rewarded when \backward{\running{}} or \backward{\backflipping{}} is to the left.
	\controlpenalty{} quadratic control penalty, \nocontrolpenalty{} no control penalty.
	\cicaption{}}
	\label{fig:supp:hand-designed-hopper}
\end{figure*}

\FloatBarrier
	\newcommand*\xbar[1]{%
	   \hbox{%
		 \vbox{%
		   \hrule width 0.8em height 0.5pt %
		   \kern0.5ex%
		   \hbox{%
			 \kern-0.1em%
			 \ensuremath{#1}%
			 \kern-0.1em%
		   }%
		 }%
	   }%
	}

\begin{table}
	\caption{Approximate distances of reward functions from the ground-truth (\groundtruthmethod{}). We report the 95\% bootstrapped lower and upper bounds, the mean, and a 95\% bound on the relative error from the mean. Distances ($1000\times$ scale) use \transitiondatasetname{} $\transitiondataset{}$ from rollouts in the \pointmazetrain{} environment of: a uniform random policy $\boldsymbol{\randompolicy{}}$, an expert $\boldsymbol{\expertpolicy{}}$ and a \textbf{Mix}ture of these policies. $\statedistribution{}$ and $\actiondistribution{}$ are computed by marginalizing $\transitiondataset{}$.}
	\label{tab:supp:learned-reward}
	\begin{subtable}{\textwidth}
		\caption{95\% lower bound $\distance^{\mathrm{LOW}}$ of approximate distance.}
		\centering
		\setlength{\tabcolsep}{3pt}
		\begin{tabular}{@{}ll@{\hspace{1em}}llll@{\hspace{0.4em}}llll@{\hspace{0.4em}}llll@{\hspace{0.75em}}cc@{}}
			\toprule
			\textbf{Reward} & & \multicolumn{3}{c}{$\boldsymbol{1000 \times \canonicalizeddistance^\mathrm{LOW}}$} & & \multicolumn{3}{c}{$\boldsymbol{1000 \times \nearestpoint^\mathrm{LOW}}$} & & \multicolumn{3}{c}{$\boldsymbol{1000 \times \returncorrelation^\mathrm{LOW}}$} & & \multicolumn{2}{c}{\textbf{Episode Return}} \\
			\cmidrule{3-5}\cmidrule{7-9}\cmidrule{11-13}\cmidrule{15-16}
			\textbf{Function} & & $\randompolicy{}$ & $\expertpolicy{}$ & \textbf{Mix} & & $\randompolicy{}$ & $\expertpolicy{}$ & \textbf{Mix} & & $\randompolicy{}$ & $\expertpolicy{}$ & \textbf{Mix} & & \textbf{Train} & \textbf{Test} \\
			\midrule
			\groundtruthmethod{} & & \num{0.03} & \num{0.02} & <0.01 &  & \num{0.02} & \num{1.43} & <0.01 &  & \num{0.00} & \num{0.00} & \num{0.00} &  & \num{-4.46} & \num{-5.82} \\
			\regressionmethod{} & & \num{35.5} & \num{33.6} & \num{26.0} &  & \num{1.22} & \num{38.8} & \num{0.33} &  & \num{9.94} & \num{90.2} & \num{2.42} &  & \num{-4.7} & \num{-5.63} \\
			\preferencesmethod{} & & \num{68.3} & \num{100} & \num{56.6} &  & \num{7.02} & \num{1239} & \num{9.25} &  & \num{24.7} & \num{358} & \num{19.5} &  & \num{-5.26} & \num{-4.88} \\
			\airlstateonlymethod{} & & \num{570} & \num{519} & \num{402} &  & \num{734} & \num{1645} & \num{424} &  & \num{547} & \num{521} & \num{238} &  & \num{-27.3} & \num{-22.7} \\
			\airlstateactionmethod{} & & \num{774} & \num{930} & \num{894} &  & \num{956} & \num{723} & \num{952} &  & \num{802} & \num{720} & \num{963} &  & \num{-29.9} & \num{-28} \\
			\midrule
			\bettergoalmethod{} & & \num{3.49} & \num{0.02} & \num{381} &  & \num{0.17} & \num{4.03} & \num{481} &  & \num{25.8} & <0.01 & \num{162} &  & \num{-28.4} & \num{-26.2} \\
			\bottomrule
		\end{tabular}
		\label{tab:supp:learned-reward:lower}
	\end{subtable}
	\begin{subtable}{\textwidth}
		\caption{Mean approximate distance $\xbar{\distance}$. Results are the same as Table~\ref{tab:learned-reward}.}
		\centering
		\setlength{\tabcolsep}{3pt}
		\begin{tabular}{@{}ll@{\hspace{1em}}llll@{\hspace{0.4em}}llll@{\hspace{0.4em}}llll@{\hspace{0.75em}}cc@{}}
			\toprule
			\textbf{Reward} & & \multicolumn{3}{c}{$\boldsymbol{1000 \times \xbar{\canonicalizeddistance{}}}$} & & \multicolumn{3}{c}{$\boldsymbol{1000 \times \xbar{\nearestpoint{}}}$} & & \multicolumn{3}{c}{$\boldsymbol{1000 \times \xbar{\returncorrelation{}}}$} & & \multicolumn{2}{c}{\textbf{Episode Return}} \\
			\cmidrule{3-5}\cmidrule{7-9}\cmidrule{11-13}\cmidrule{15-16}
			\textbf{Function} & & $\randompolicy{}$ & $\expertpolicy{}$ & \textbf{Mix} & & $\randompolicy{}$ & $\expertpolicy{}$ & \textbf{Mix} & & $\randompolicy{}$ & $\expertpolicy{}$ & \textbf{Mix}  & & \textbf{Train} & \textbf{Test} \\
			\midrule
			\groundtruthmethod{} & & \num{0.06} & \num{0.05} & \num{0.04} &  & \num{0.04} & \num{3.17} & \num{0.01} &  & \num{0.00} & \num{0.00} & \num{0.00} &  & \num{-5.19} & \num{-6.59} \\
			\regressionmethod{} & & \num{35.8} & \num{33.7} & \num{26.1} &  & \num{1.42} & \num{38.9} & \num{0.35} &  & \num{9.99} & \num{90.7} & \num{2.43} &  & \num{-5.47} & \num{-6.3} \\
			\preferencesmethod{} & & \num{68.7} & \num{100} & \num{56.8} &  & \num{8.51} & \num{1333} & \num{9.74} &  & \num{24.9} & \num{360} & \num{19.6} &  & \num{-5.57} & \num{-5.04} \\
			\airlstateonlymethod{} & & \num{572} & \num{520} & \num{404} &  & \num{817} & \num{2706} & \num{488} &  & \num{549} & \num{523} & \num{240} &  & \num{-27.3} & \num{-22.7} \\
			\airlstateactionmethod{} & & \num{776} & \num{930} & \num{894} &  & \num{1067} & \num{2040} & \num{1039} &  & \num{803} & \num{722} & \num{964} &  & \num{-30.7} & \num{-29} \\
			\midrule
			\bettergoalmethod{} & & \num{17.0} & \num{0.05} & \num{397} &  & \num{0.68} & \num{6.30} & \num{597} &  & \num{35.3} & <0.01 & \num{166} &  & \num{-30.4} & \num{-29.1} \\
			\bottomrule
		\end{tabular}
		\label{tab:supp:learned-reward:mean}
	\end{subtable}
	\begin{subtable}{\textwidth}
		\caption{95\% upper bound $\distance^{\mathrm{UP}}$ of approximate distance.}
		\centering
		\setlength{\tabcolsep}{3pt}
		\begin{tabular}{@{}ll@{\hspace{1em}}llll@{\hspace{0.4em}}llll@{\hspace{0.4em}}llll@{\hspace{0.75em}}cc@{}}
			\toprule
			\textbf{Reward} & & \multicolumn{3}{c}{$\boldsymbol{1000 \times \canonicalizeddistance^{\mathrm{UP}}}$} & & \multicolumn{3}{c}{$\boldsymbol{1000 \times \nearestpoint^{\mathrm{UP}}}$} & & \multicolumn{3}{c}{$\boldsymbol{1000 \times \returncorrelation^{\mathrm{UP}}}$}  & & \multicolumn{2}{c}{\textbf{Episode Return}} \\
			\cmidrule{3-5}\cmidrule{7-9}\cmidrule{11-13}\cmidrule{15-16}
			\textbf{Function} & & $\randompolicy{}$ & $\expertpolicy{}$ & \textbf{Mix} & & $\randompolicy{}$ & $\expertpolicy{}$ & \textbf{Mix} & & $\randompolicy{}$ & $\expertpolicy{}$ & \textbf{Mix} & & \textbf{Train} & \textbf{Test} \\
			\midrule
			\groundtruthmethod{} & & \num{0.09} & \num{0.07} & \num{0.07} &  & \num{0.06} & \num{6.14} & \num{0.01} &  & <0.01 & <0.01 & <0.01 &  & \num{-6.04} & \num{-7.14} \\
			\regressionmethod{} & & \num{36.1} & \num{33.7} & \num{26.2} &  & \num{1.53} & \num{39.1} & \num{0.37} &  & \num{10.0} & \num{91.2} & \num{2.44} &  & \num{-6.26} & \num{-6.83} \\
			\preferencesmethod{} & & \num{69.1} & \num{101} & \num{57.1} &  & \num{10.0} & \num{1432} & \num{10.1} &  & \num{25.2} & \num{361} & \num{19.7} &  & \num{-5.9} & \num{-5.22} \\
			\airlstateonlymethod{} & & \num{574} & \num{520} & \num{407} &  & \num{982} & \num{3984} & \num{532} &  & \num{551} & \num{526} & \num{242} &  & \num{-27.3} & \num{-22.8} \\
			\airlstateactionmethod{} & & \num{779} & \num{930} & \num{895} &  & \num{1241} & \num{4378} & \num{1124} &  & \num{805} & \num{724} & \num{964} &  & \num{-31.7} & \num{-29.8} \\
			\midrule
			\bettergoalmethod{} & & \num{35.2} & \num{0.09} & \num{414} &  & \num{1.66} & \num{10.8} & \num{821} &  & \num{45.4} & <0.01 & \num{171} &  & \num{-32} & \num{-31.4} \\
			\bottomrule
		\end{tabular}
		\label{tab:supp:learned-reward:upper}
	\end{subtable}
	\begin{subtable}{\textwidth}
		\centering
		\caption{Relative 95\% confidence interval $\distance^{\mathrm{REL}} = \left\vert\max\left(\frac{\mathrm{Upper}}{\mathrm{Mean}} - 1,1 - \frac{\mathrm{Lower}}{\mathrm{Mean}}\right)\right\vert$ in percent. The population mean is contained within $\pm D^{\mathrm{REL}}\%$ of the sample mean in Table~\ref{tab:supp:learned-reward:mean} with 95\% probability.}
		\setlength{\tabcolsep}{3pt}
		\begin{tabular}{@{}ll@{\hspace{1em}}llll@{\hspace{0.4em}}llll@{\hspace{0.4em}}llll@{\hspace{0.75em}}cc@{}}
			\toprule
			\textbf{Reward} & & \multicolumn{3}{c}{$\boldsymbol{\canonicalizeddistance^{\mathrm{REL}}\%}$} & & \multicolumn{3}{c}{$\boldsymbol{\nearestpoint^{\mathrm{REL}}\%}$} & & \multicolumn{3}{c}{$\boldsymbol{\returncorrelation^{\mathrm{REL}}\%}$} & & \multicolumn{2}{c}{\textbf{Episode Return}} \\
			\cmidrule{3-5}\cmidrule{7-9}\cmidrule{11-13}\cmidrule{15-16}
			\textbf{Function} & & $\randompolicy{}$ & $\expertpolicy{}$ & \textbf{Mix} & & $\randompolicy{}$ & $\expertpolicy{}$ & \textbf{Mix} & & $\randompolicy{}$ & $\expertpolicy{}$ & \textbf{Mix} & & \textbf{Train} & \textbf{Test} \\
			\midrule
			\groundtruthmethod{} & & \num{50.0} & \num{62.5} & \num{80.0} &  & \num{61.8} & \num{94.0} & \num{29.7} &  & inf & inf & inf &  & \num{0.16} & \num{0.12} \\
			\regressionmethod{} & & \num{0.81} & \num{0.14} & \num{0.40} &  & \num{14.2} & \num{0.42} & \num{7.48} &  & \num{0.53} & \num{0.55} & \num{0.57} &  & \num{0.14} & \num{0.11} \\
			\preferencesmethod{} & & \num{0.61} & \num{0.14} & \num{0.44} &  & \num{17.5} & \num{7.49} & \num{5.02} &  & \num{0.90} & \num{0.48} & \num{0.48} &  & \num{0.06} & \num{0.04} \\
			\airlstateonlymethod{} & & \num{0.38} & \num{0.08} & \num{0.67} &  & \num{20.2} & \num{47.2} & \num{13.2} &  & \num{0.34} & \num{0.40} & \num{0.69} &  & <0.01 & <0.01 \\
			\airlstateactionmethod{} & & \num{0.35} & \num{0.02} & \num{0.08} &  & \num{16.3} & \num{115} & \num{8.42} &  & \num{0.23} & \num{0.26} & \num{0.04} &  & \num{0.03} & \num{0.04} \\
			\midrule
			\bettergoalmethod{} & & \num{108} & \num{65.5} & \num{4.17} &  & \num{142} & \num{70.9} & \num{37.5} &  & \num{28.5} & \num{0.55} & \num{2.66} &  & \num{0.07} & \num{0.10} \\
			\bottomrule
		\end{tabular}
		\label{tab:supp:learned-reward:relative-error}
	\end{subtable}
\end{table}

\begin{table}
	\caption{Approximate distances of reward functions from the ground-truth (\groundtruthmethod{}) under pathological \transitiondatasetname{}s. We report the 95\% bootstrapped lower and upper bounds, the mean, and a 95\% bound on the relative error from the mean. Distances ($1000\times$ scale) use four different \transitiondatasetname{}s $\transitiondataset{}$. $\boldsymbol{\permutedrandompolicy{}}$ independently samples states, actions and next states from the marginal distributions of rollouts from the uniform random policy $\randompolicy$ in the \pointmazetrain{} environment. \textbf{Ind} independently samples the components of states and next states from $\mathcal{N}(0,1)$, and actions from $U[-1, 1]$. \textbf{Jail} consists of rollouts of $\randompolicy$ restricted to a small $0.09 \times 0.09$ ``jail'' square that excludes the goal state $0.5$ distance away. $\boldsymbol{\wrongpolicy{}}$ are rollouts in \pointmazetrain{} of a policy that goes to the corner opposite the goal state. $\boldsymbol{\permutedrandompolicy{}}$ and \textbf{Ind} are not supported by \returncorrelationabbrevonly{} since they do not produce complete episodes.}
	\label{tab:supp:learned-reward-pathological}
	\begin{subtable}{\textwidth}
		\caption{95\% lower bound $\distance^{\mathrm{LOW}}$ of approximate distance.}
		\centering
		\setlength{\tabcolsep}{3pt}
		\begin{tabular}{@{}ll@{\hspace{0.75em}}lllll@{\hspace{0.4em}}lllll@{\hspace{0.4em}}lll@{}}
			\toprule
			\textbf{Reward} & & \multicolumn{4}{c}{$\boldsymbol{1000 \times \canonicalizeddistance^\mathrm{LOW}}$} & & \multicolumn{4}{c}{$\boldsymbol{1000\mkern-4mu\times\mkern-4mu\nearestpoint^\mathrm{LOW}}$} & & \multicolumn{2}{c}{$\boldsymbol{1000 \times \returncorrelation^\mathrm{LOW}}$} \\
			\cmidrule{3-6}\cmidrule{8-11}\cmidrule{13-14}
			\textbf{Function} & & $\boldsymbol{\permutedrandompolicy{}}$ & \textbf{Ind} & \textbf{Jail} & $\boldsymbol{\wrongpolicy{}}$ & & $\boldsymbol{\permutedrandompolicy{}}$ & \textbf{Ind} & \textbf{Jail} & $\boldsymbol{\wrongpolicy{}}$ & & \textbf{Jail} & $\boldsymbol{\wrongpolicy{}}$ \\
			\midrule
			\regressionmethod{} & & \num{127} & \num{398} & \num{705} & \num{205} &  & \num{87.6} & \num{590} & \num{2433} & \num{898} &  & \num{809} & \num{456} \\
\preferencesmethod{} & & \num{146} & \num{433} & \num{462} & \num{349} &  & \num{97.4} & \num{632} & \num{661} & \num{221} &  & \num{372} & \num{332} \\
\airlstateonlymethod{} & & \num{570} & \num{541} & \num{712} & \num{710} &  & \num{697} & \num{821} & \num{957} & \num{621} &  & \num{751} & \num{543} \\
\airlstateactionmethod{} & & \num{768} & \num{628} & \num{558} & \num{669} &  & \num{720} & \num{960} & \num{940} & \num{2355} &  & \num{428} & \num{753} \\
			\midrule
			\bettergoalmethod{} & & \num{9.22} & \num{0.03} & <0.01 & \num{0.02} &  & \num{0.41} & \num{0.05} & \num{11.2} & \num{31.3} &  & <0.01 & \num{0.02} \\
			\bottomrule
		\end{tabular}
		\label{tab:supp:learned-reward-pathological:lower}
	\end{subtable}
	\begin{subtable}{\textwidth}
		\caption{Mean approximate distance $\xbar{\distance}$. Results are the same as Table~\ref{tab:learned-reward}.}
		\centering
		\setlength{\tabcolsep}{3pt}
		\begin{tabular}{@{}ll@{\hspace{0.75em}}lllll@{\hspace{0.4em}}lllll@{\hspace{0.4em}}lll@{}}
			\toprule
			\textbf{Reward} & & \multicolumn{4}{c}{$\boldsymbol{1000 \times \xbar{\canonicalizeddistance{}}}$} & & \multicolumn{4}{c}{$\boldsymbol{1000 \times \xbar{\nearestpoint{}}}$} & & \multicolumn{2}{c}{$\boldsymbol{1000\mkern-4mu\times\mkern-4mu\xbar{\returncorrelation{}}}$} \\
			\cmidrule{3-6}\cmidrule{8-11}\cmidrule{13-14}
			\textbf{Function} & & $\boldsymbol{\permutedrandompolicy{}}$ & \textbf{Ind} & \textbf{Jail} & $\boldsymbol{\wrongpolicy{}}$ & & $\boldsymbol{\permutedrandompolicy{}}$ & \textbf{Ind} & \textbf{Jail} & $\boldsymbol{\wrongpolicy{}}$ & & \textbf{Jail} & $\boldsymbol{\wrongpolicy{}}$ \\
			\midrule
			\regressionmethod{} & & \num{128} & \num{398} & \num{705} & \num{206} &  & \num{97.2} & \num{591} & \num{2549} & \num{921} &  & \num{810} & \num{458} \\
			\preferencesmethod{} & & \num{147} & \num{433} & \num{463} & \num{349} &  & \num{117} & \num{633} & \num{683} & \num{237} &  & \num{374} & \num{333} \\
			\airlstateonlymethod{} & & \num{573} & \num{541} & \num{713} & \num{710} &  & \num{826} & \num{823} & \num{988} & \num{852} &  & \num{753} & \num{545} \\
			\airlstateactionmethod{} & & \num{771} & \num{628} & \num{558} & \num{669} &  & \num{859} & \num{962} & \num{964} & \num{2694} &  & \num{430} & \num{754} \\
			\midrule
			\bettergoalmethod{} & & \num{42.4} & \num{0.06} & \num{0.03} & \num{0.05} &  & \num{1.41} & \num{0.25} & \num{18.3} & \num{39} &  & <0.01 & \num{0.02} \\
			\bottomrule
		\end{tabular}
		\label{tab:supp:learned-reward-pathological:mean}
	\end{subtable}
	\begin{subtable}{\textwidth}
		\caption{95\% upper bound $\distance^{\mathrm{UP}}$ of approximate distance.}
		\centering
		\setlength{\tabcolsep}{3pt}
		\begin{tabular}{@{}ll@{\hspace{0.75em}}lllll@{\hspace{0.4em}}lllll@{\hspace{0.4em}}lll@{}}
			\toprule
			\textbf{Reward} & & \multicolumn{4}{c}{$\boldsymbol{1000 \times \canonicalizeddistance^{\mathrm{UP}}}$} & & \multicolumn{4}{c}{$\boldsymbol{1000 \times \nearestpoint^{\mathrm{UP}}}$} & & \multicolumn{2}{c}{$\boldsymbol{1000\mkern-4mu\times\mkern-4mu\returncorrelation^{\mathrm{UP}}}$} \\
			\cmidrule{3-6}\cmidrule{8-11}\cmidrule{13-14}
			\textbf{Function} & & $\boldsymbol{\permutedrandompolicy{}}$ & \textbf{Ind} & \textbf{Jail} & $\boldsymbol{\wrongpolicy{}}$ & & $\boldsymbol{\permutedrandompolicy{}}$ & \textbf{Ind} & \textbf{Jail} & $\boldsymbol{\wrongpolicy{}}$ & & \textbf{Jail} & $\boldsymbol{\wrongpolicy{}}$ \\
			\midrule
			\regressionmethod{} & & \num{129} & \num{398} & \num{706} & \num{206} &  & \num{106} & \num{593} & \num{2654} & \num{948} &  & \num{812} & \num{460} \\
			\preferencesmethod{} & & \num{148} & \num{433} & \num{464} & \num{349} &  & \num{132} & \num{635} & \num{705} & \num{265} &  & \num{376} & \num{335} \\
			\airlstateonlymethod{} & & \num{576} & \num{541} & \num{713} & \num{710} &  & \num{939} & \num{825} & \num{1047} & \num{1021} &  & \num{755} & \num{547} \\
			\airlstateactionmethod{} & & \num{774} & \num{628} & \num{559} & \num{669} &  & \num{1015} & \num{963} & \num{981} & \num{3012} &  & \num{432} & \num{756} \\
			\midrule
			\bettergoalmethod{} & & \num{85.9} & \num{0.09} & \num{0.05} & \num{0.09} &  & \num{3.25} & \num{0.46} & \num{28.6} & \num{45.3} &  & <0.01 & \num{0.02} \\
			\bottomrule
		\end{tabular}
		\label{tab:supp:learned-reward-pathological:upper}
	\end{subtable}
	\begin{subtable}{\textwidth}
		\centering
		\caption{Relative 95\% confidence interval $\distance^{\mathrm{REL}} = \left\vert\max\left(\frac{\mathrm{Upper}}{\mathrm{Mean}} - 1,1 - \frac{\mathrm{Lower}}{\mathrm{Mean}}\right)\right\vert$ in percent. The population mean is contained within $\pm D^{\mathrm{REL}}\%$ of the sample mean in Table~\ref{tab:supp:learned-reward-pathological:mean} with 95\% probability.}
		\setlength{\tabcolsep}{3pt}
		\begin{tabular}{@{}ll@{\hspace{0.75em}}lllll@{\hspace{0.4em}}lllll@{\hspace{0.4em}}lll@{}}
			\toprule
			\textbf{Reward} & & \multicolumn{4}{c}{$\boldsymbol{\canonicalizeddistance^{\mathrm{REL}}\%}$} & & \multicolumn{4}{c}{$\boldsymbol{\nearestpoint^{\mathrm{REL}}\%}$} & & \multicolumn{2}{c}{$\boldsymbol{\returncorrelation^{\mathrm{REL}}\%}$} \\
			\cmidrule{3-6}\cmidrule{8-11}\cmidrule{13-14}
			\textbf{Function} & & $\boldsymbol{\permutedrandompolicy{}}$ & \textbf{Ind} & \textbf{Jail} & $\boldsymbol{\wrongpolicy{}}$ & & $\boldsymbol{\permutedrandompolicy{}}$ & \textbf{Ind} & \textbf{Jail} & $\boldsymbol{\wrongpolicy{}}$ & & \textbf{Jail} & $\boldsymbol{\wrongpolicy{}}$ \\
			\midrule
			\regressionmethod{} & & \num{0.79} & \num{0.01} & \num{0.08} & \num{0.06} &  & \num{9.81} & \num{0.24} & \num{4.54} & \num{2.89} &  & \num{0.18} & \num{0.42} \\
			\preferencesmethod{} & & \num{0.76} & <0.01 & \num{0.16} & \num{0.07} &  & \num{16.9} & \num{0.35} & \num{3.35} & \num{11.7} &  & \num{0.47} & \num{0.48} \\
			\airlstateonlymethod{} & & \num{0.48} & <0.01 & \num{0.07} & \num{0.02} &  & \num{15.6} & \num{0.28} & \num{6.01} & \num{27.1} &  & \num{0.23} & \num{0.38} \\
			\airlstateactionmethod{} & & \num{0.38} & <0.01 & \num{0.13} & \num{0.03} &  & \num{18.2} & \num{0.22} & \num{2.47} & \num{12.6} &  & \num{0.45} & \num{0.23} \\
			\midrule
			\bettergoalmethod{} & & \num{103} & \num{50} & \num{80} & \num{83.3} &  & \num{131} & \num{85.4} & \num{56.4} & \num{19.7} &  & \num{0.54} & \num{0.55} \\
			\bottomrule
		\end{tabular}
		\label{tab:supp:learned-reward-pathological:relative-error}
	\end{subtable}
\end{table}

\begin{figure}
	\centering
	\vspace*{-2em}
	\begin{subfigure}{\textwidth}
		\includegraphics{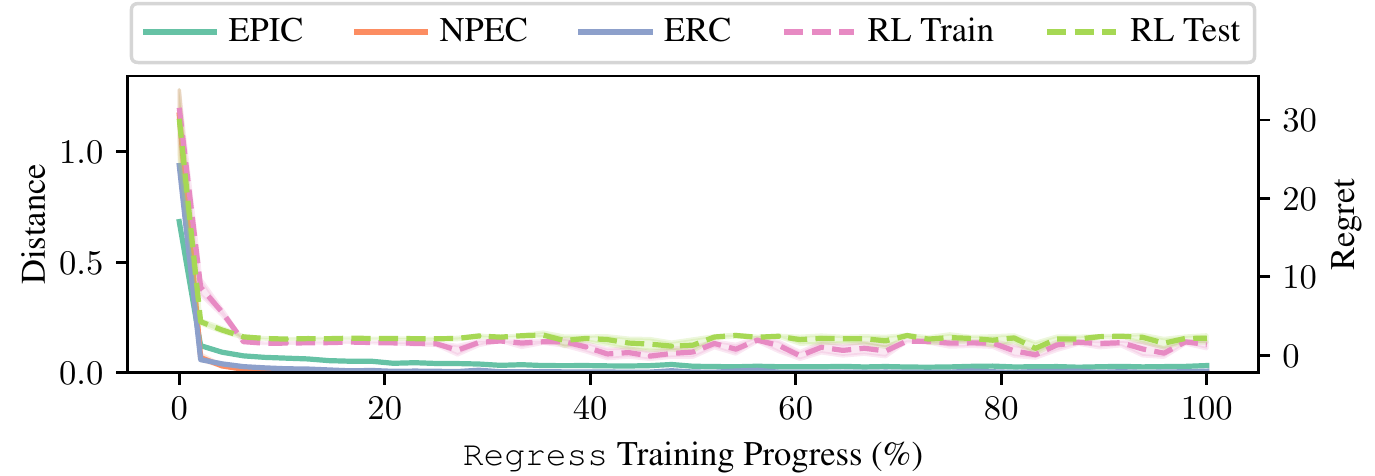}
		\caption{Comparisons of \regressionmethod{} using all distance algorithms.}
		\label{fig:supp:checkpoints-rewards:regress}
	\end{subfigure}
	\begin{subfigure}{\textwidth}
		\includegraphics{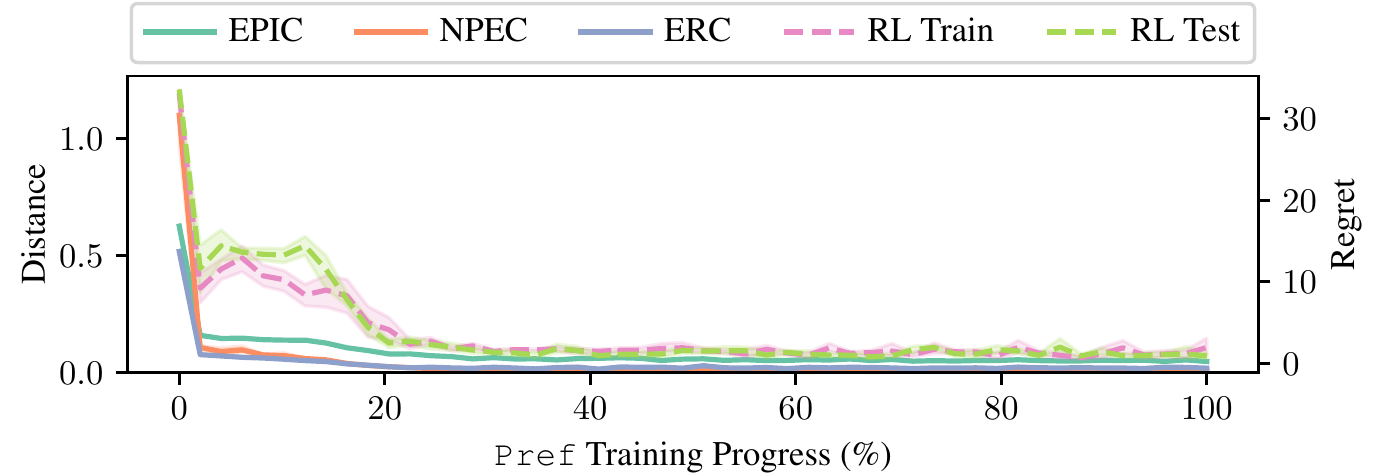}
		\caption{Comparisons of \preferencesmethod{} using all distance algorithms.}
		\label{fig:supp:checkpoints-rewards:prefs}
	\end{subfigure}
	\begin{subfigure}{\textwidth}
		\includegraphics{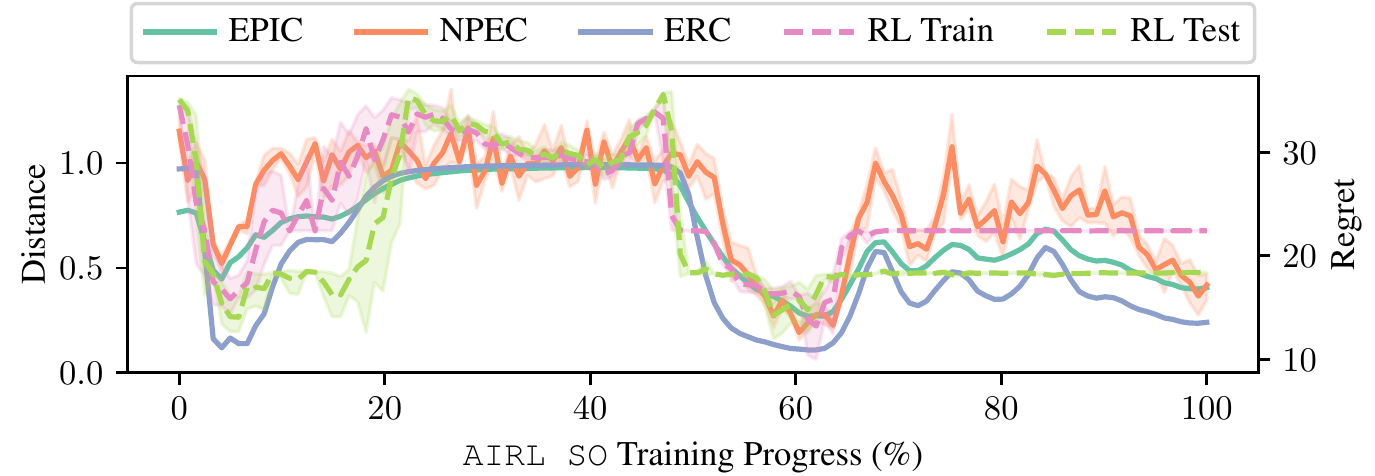}
		\caption{Comparisons of \airlstateonlymethod{} using all distance algorithms.}
		\label{fig:supp:checkpoints-rewards:airlso}
	\end{subfigure}
	\begin{subfigure}{\textwidth}
		\includegraphics{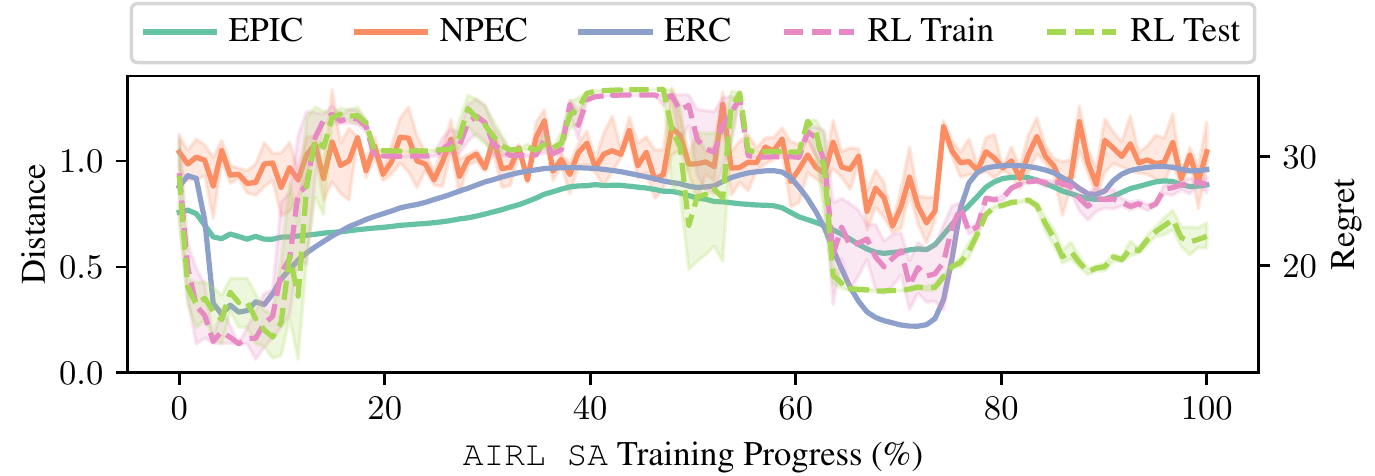}
		\caption{Comparisons of \airlstateactionmethod{} using all distance algorithms.}
		\label{fig:supp:checkpoints-rewards:airlsa}
	\end{subfigure}
	\caption{\checkpointscaption{}}
	\label{fig:supp:checkpoints-rewards}
\end{figure}

\begin{figure}
	\centering
	\vspace*{-3em}
	\begin{subfigure}{\textwidth}
		\includegraphics{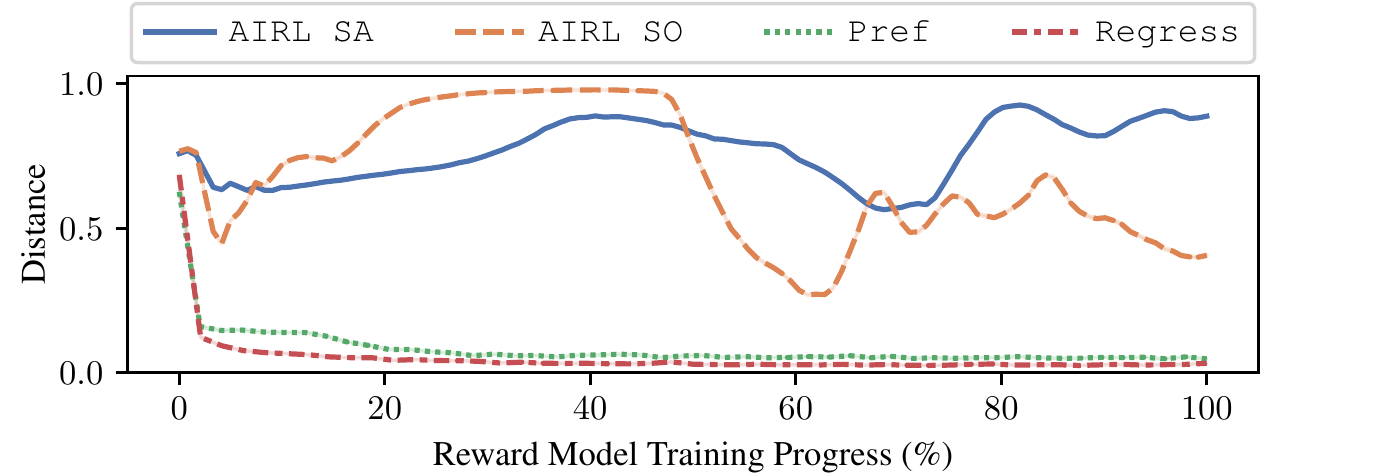}
		\caption{Comparisons using \canonicalizeddistanceabbrevonly{} on all reward models.}
		\label{fig:supp:checkpoints-algorithms:epic}
	\end{subfigure}
	\begin{subfigure}{\textwidth}
		\includegraphics{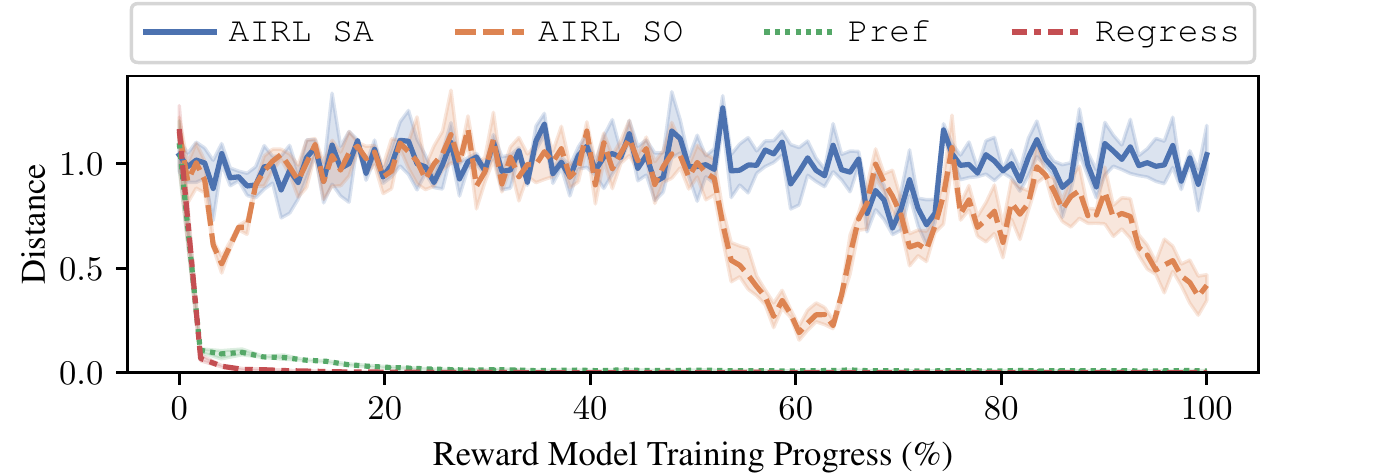}
		\caption{Comparisons using \nearestpointabbrevonly{} on all reward models.}
		\label{fig:supp:checkpoints-algorithms:npec}
	\end{subfigure}
	\begin{subfigure}{\textwidth}
		\includegraphics{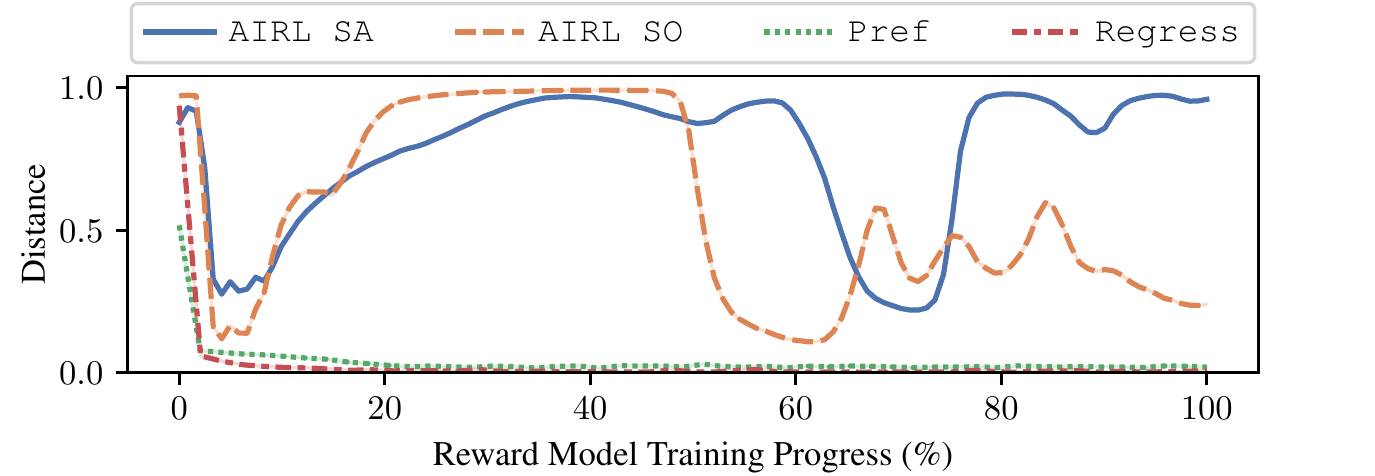}
		\caption{Comparisons using \returncorrelationabbrevonly{} on all reward models.}
		\label{fig:supp:checkpoints-algorithms:erc}
	\end{subfigure}
	\begin{subfigure}{\textwidth}
		\includegraphics{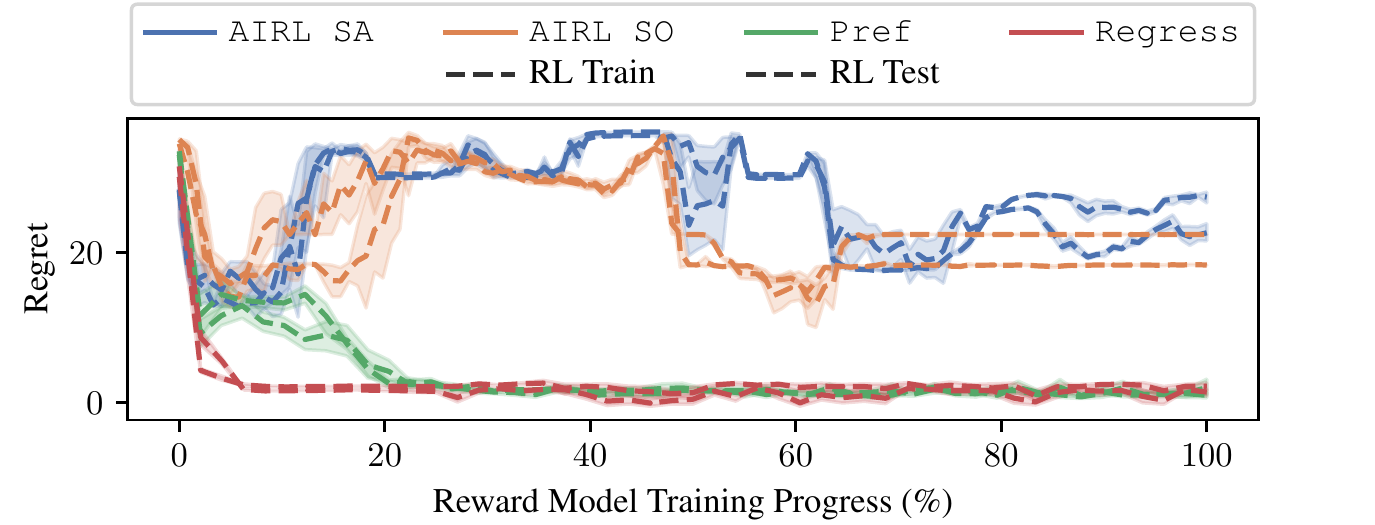}
		\caption{Comparisons using Episode Return on all reward models.}
		\label{fig:supp:checkpoints-algorithms:rl}
	\end{subfigure}
	\caption{\checkpointscaption{}}
	\label{fig:supp:checkpoints-algorithms}
\end{figure}

\clearpage
\subsection{Proofs}
\label{sec:supp:proofs}

\subsubsection{Background}
\label{sec:supp:proofs:background}

\rewardequivisequivalence*
\begin{proof}
	$\firstreward{} \equivreward \firstreward{}$ since choosing $\lambda = 1 > 0$ and $\potential(\state) = 0$, a bounded potential function, we have $\firstreward{}(\state,\action,\nextstate) = \lambda \firstreward{}(\state,\action,\nextstate) + \discount \potential(\nextstate) - \potential(\state)$ for all $\state,\nextstate \in \statespace$ and $a \in \actionspace$.
	
	Suppose $\firstreward{} \equivreward \secondreward{}$. Then there exists some $\lambda > 0$ and a bounded potential function $\potential:\statespace \to \mathbb{R}$ such that $\secondreward{}(\state,\action,\nextstate) = \lambda \firstreward{}(\state,\action,\nextstate) + \discount \potential(\nextstate) - \potential(\state)$ for all $\state,\nextstate \in \statespace$ and $a \in \actionspace$. Rearranging:
	\begin{equation}
	\firstreward{}(\state,\action,\nextstate) = \frac{1}{\lambda}\secondreward{}(\state,\action,\nextstate) + \discount \left(\frac{-1}{\lambda}\potential(\nextstate)\right) - \left(\frac{-1}{\lambda}\potential(\state)\right).
	\end{equation}
	Since $\frac{1}{\lambda} > 0$ and $\potential'(s) = \frac{-1}{\lambda}\potential(s)$ is a bounded potential function, it follows that $\secondreward{} \equivreward \firstreward{}$.
	
	Finally, suppose $\firstreward{} \equivreward \secondreward{}$ and $\secondreward{} \equivreward \thirdreward{}$. Then there exists some $\lambda_1, \lambda_2 > 0$ and bounded potential functions $\potential_1,\potential_2:\statespace \to \mathbb{R}$ such that for all $\state,\nextstate \in \statespace$ and $a \in \actionspace$:
	\begin{align}
	\secondreward{}(\state,\action,\nextstate) &= \lambda_1 \firstreward{}(\state,\action,\nextstate) + \discount \potential_1(\nextstate) - \potential_1(\state), \\
	\thirdreward{}(\state,\action,\nextstate) &= \lambda_2 \secondreward{}(\state,\action,\nextstate) + \discount \potential_2(\nextstate) - \potential_2(\state).
	\end{align}
	Substituting the expression for $\secondreward{}$ into the expression for $\thirdreward{}$:
	\begin{align}
	\thirdreward{}(\state,\action,\nextstate) &= \lambda_2\left(\lambda_1 \firstreward{}(\state,\action,\nextstate) + \discount \potential_1(\nextstate) - \potential_1(\state)\right) + \discount \potential_2(\nextstate) - \potential_2(\state) \\
	&= \lambda_1 \lambda_2 \firstreward{}(\state,\action,\nextstate) + \discount\left(\lambda_2 \potential_1(\nextstate) + \potential_2(\nextstate)\right) - \left(\lambda_2 \potential_1(\state) + \potential_2(\state)\right) \\
	&= \lambda \firstreward{}(\state,\action,\nextstate) + \discount\potential(\nextstate) - \potential(\state),
	\end{align}
	where $\lambda = \lambda_1 \lambda_2 > 0$ and $\potential(\state) = \lambda_2 \potential_1(\state) + \potential_2(\state)$ is bounded. Thus $\firstreward{} \equivreward \thirdreward{}$.
\end{proof}

\subsubsection{\canonicalizeddistancewithabbrev{}}
\label{sec:supp:proofs:canonicalizeddistance}

\canonicallyshapedrewardinvariant*
\begin{proof}
	Let $\state,\action,\nextstate \in \statespace \times \actionspace \times \statespace$. Then by substituting in the definition of $\reward{}'$ and using linearity of expectations:
	\begin{align}
	\canonical{\reward{}'}(\state,\action,\nextstate) &\triangleq \reward{}'(\state,\action,\nextstate) + \expectation\left[\discount\reward{}'(\nextstate, A, S') - \reward{}'(\state, A, S') - \discount \reward{}'(S, A, S')\right] \\
	&= \left(\reward{}(\state, \action, \nextstate) + \discount \potential(\nextstate) - \potential(\state)\right) \\
	&\quad+ \expectation\left[\discount\reward{}(\nextstate, A, S') + \discount^2 \potential(S') - \discount\potential(\nextstate)\right] \nonumber \\
	&\quad- \expectation\left[\reward{}(\state, A, S') + \discount \potential(S') - \potential(\state)\right] \nonumber \\
	&\quad- \expectation\left[\discount\reward{}(S, A, S') + \discount^2 \potential(S') - \discount\potential(S)\right] \nonumber \\
	&= \reward{}(\state, \action, \nextstate) + \expectation\left[\discount \reward{}(\nextstate, A, S') - \reward{}(\state, A, S') - \discount\reward{}(S, A, S') \right] \\
	&\quad+ \left(\discount \potential(\nextstate) - \potential(\state)\right) - \expectation\left[\discount \potential(\nextstate) - \potential(\state)\right] \nonumber \\
	&\quad+\expectation\left[\discount^2 \potential(S')  - \discount \potential(S')\right] - \expectation\left[\discount^2 \potential(S') - \discount \potential(S)\right] \nonumber \\
	&= \reward{}(\state, \action, \nextstate) + \expectation\left[\discount \reward{}(\nextstate, A, S') - \reward{}(\state, A, S') - \discount\reward{}(S, A, S') \right] \\
	&\triangleq \canonical{\reward}(\state,\action,\nextstate),
	\end{align}
	where the penultimate step uses $\expectation[\potential(S')] = \expectation[\potential(S)]$ since $S$ and $S'$ are identically distributed.
\end{proof}

\smoothnesscanonicalization*
\begin{proof}
Observe that:
\begin{equation}
\potential_{\statedistribution, \actiondistribution}(\reward)(\state, \action, \nextstate) = \expectation\left[\discount\reward{}(\nextstate, A, S') - \reward{}(\state, A, S') - \discount\reward{}(S, A, S')\right],
\end{equation}
where $S$ and $S'$ are random variables independently sampled from $\statedistribution$, and $A$ independently sampled from $\actiondistribution$.

Then:
\begin{equation}
\potential_{\statedistribution, \actiondistribution}(\reward + \nu) - \potential_{\statedistribution, \actiondistribution}(\reward) = \potential_{\statedistribution, \actiondistribution}(\nu).
\end{equation}
Now:
\begin{align}
\text{LHS} &\triangleq \left\Vert \potential_{\statedistribution, \actiondistribution}(\reward + \nu) - \potential_{\statedistribution, \actiondistribution}(\reward)\right\Vert_{\infty} \\
&= \max_{\state, \nextstate \in \statespace} \left|\expectation\left[\discount\nu(\nextstate, A, S') - \nu(\state, A, S') - \discount\nu(S, A, S')\right]\right| \\
&= \max_{\state, \nextstate \in \statespace} \left|\lambda\left(\discount\mathbb{I}[x = \nextstate]\actiondistribution(u)\statedistribution(x')\right.\right. \\
&- \left.\left.\mathbb{I}[x = \state]\actiondistribution(u)\statedistribution(x') - \discount\statedistribution(x)\actiondistribution(u)\statedistribution(x')\right)\right| \\
&= \max_{\state, \nextstate \in \statespace} \left|\lambda \actiondistribution(u)\statedistribution(x') \left(\discount\mathbb{I}[x = \nextstate] - \mathbb{I}[x = \state] - \discount\statedistribution(x)\right)\right| \\
&= \lambda \left(1 + \discount\statedistribution(x)\right) \actiondistribution(u)\statedistribution(x'),
\end{align}
where the final step follows by substituting $\state = x$ and $\nextstate \neq x$ (using $|\statespace| \geq 2$).
\end{proof}

\pearsondistanceproperties*
\begin{proof}
For a non-constant random variable $V$, define a standardized (zero mean and unit variance) version:
\begin{equation}
Z(V) = \frac{V - \expectation[V]}{\sqrt{\expectation{}\left[\left(V - \expectation[V]\right)^2\right]}}.
\end{equation}
The Pearson correlation coefficient on random variables $A$ and $B$ is equal to the expected product of these standardized random variables:
\begin{equation}
\rho(A,B) = \expectation\left[Z(A)Z(B)\right].
\end{equation}

Let $W$, $X$, $Y$ be random variables.

\textbf{Identity}. Have $\rho(X,X) = 1$, so $\pearsondistance{}(X,X) = 0$.

\textbf{Symmetry}. Have $\rho(X,Y) = \rho(Y,X)$ by commutativity of multiplication, so $\pearsondistance{}(X,Y) = \pearsondistance{}(Y,X)$.

\textbf{Triangle Inequality}. 
For any random variables $A,B$:
\begin{align}
\expectation\left[\left(Z(A) - Z(B)\right)^2\right] &= \expectation\left[Z(A)^2 - 2Z(A)Z(B) + Z(B)^2\right] \\
&= \expectation\left[Z(A)^2\right] + \expectation\left[Z(B)^2\right] - 2\expectation\left[Z(A)Z(B)\right] \\
&= 2 - 2\expectation\left[Z(A)Z(B)\right] \\
&= 2\left(1 - \rho(A,B)\right) \\
&= 4\pearsondistance{}(A,B)^2.
\end{align}
So:
\begin{align}
4\pearsondistance{}(W,Y)^2 &= \expectation\left[\left(Z(W) - Z(Y)\right)^2\right] \\
&= \expectation\left[\left(Z(W) - Z(X) + Z(X) - Z(Y)\right)^2\right] \\
&= \expectation\left[\left(Z(W) - Z(X)\right)^2\right] + \expectation\left[\left(Z(X) - Z(Y)\right)^2\right] \\
&\quad+ 2\expectation\left[\left(Z(W) - Z(X)\right)\left(Z(X) - Z(Y)\right)\right] \nonumber \\
&= 4\pearsondistance{}(W,X)^2 + 4\pearsondistance{}(X,Y)^2 + 8\expectation\left[\left(Z(W) - Z(X)\right)\left(Z(X) - Z(Y)\right)\right].
\end{align}
Since $\left<A,B\right> = \expectation[AB]$ is an inner product over $\mathbb{R}$, it follows by the Cauchy-Schwarz inequality that $\expectation[AB] \leq \sqrt{\expectation[A^2]\expectation[B^2]}$. So:
\begin{align}
\pearsondistance{}(W,Y)^2 &\leq \pearsondistance{}(W,X)^2 + \pearsondistance{}(X,Y)^2 + 2\pearsondistance{}(W,X)\pearsondistance{}(X,Y) \\
&= \left(\pearsondistance{}(W,X) + \pearsondistance{}(X,Y)\right)^2.
\end{align}
Taking the square root of both sides:
\begin{equation}
\pearsondistance{}(W,Y) \leq \pearsondistance{}(W,X) + \pearsondistance{}(X,Y),
\end{equation}
as required.

\textbf{Positive Affine Invariant and Bounded}
$\pearsondistance{}(aX + c, bY + d) = \pearsondistance{}(X, Y)$ is immediate from $\rho(X,Y)$ invariant to positive affine transformations.
Have $-1 \leq \rho(X,Y) \leq 1$, so $0 \leq 1 - \rho(X,Y) \leq 2$ thus $0 \leq \pearsondistance{}(X,Y) \leq 1$.
\end{proof}

\canonicalizeddistancepseudometric*
\begin{proof}
The result follows from $\pearsondistance{}$ being a pseudometric. Let $\firstreward{}$, $\secondreward{}$ and $\thirdreward{}$ be reward functions mapping from transitions $\statespace \times \actionspace \times \statespace$ to real numbers $\mathbb{R}$.

\textbf{Identity}. Have:
\begin{equation}
\canonicalizeddistance{}(\firstreward{},\firstreward{}) = \pearsondistance{}\left(\canonical{\firstreward{}}(S,A,S'),\canonical{\firstreward{}}(S,A,S')\right) = 0,
\end{equation}
since $\pearsondistance{}(X,X) = 0$.

\textbf{Symmetry}. Have:
\begin{align}
\canonicalizeddistance{}(\firstreward{},\secondreward{}) &= \pearsondistance{}\left(\canonical{\firstreward{}}(S,A,S'),\canonical{\secondreward{}}(S,A,S')\right) \\
&= \pearsondistance{}\left(\canonical{\secondreward{}}(S,A,S'),\canonical{\firstreward{}}(S,A,S')\right) \\
&= \canonicalizeddistance{}(\secondreward{},\firstreward{}),
\end{align}
since $\pearsondistance{}(X,Y) = \pearsondistance{}(Y,X)$.

\textbf{Triangle Inequality}. Have:
\begin{align}
\canonicalizeddistance{}(\firstreward{},\thirdreward{}) &= \pearsondistance{}\left(\canonical{\firstreward{}}(S,A,S'),\canonical{\thirdreward{}}(S,A,S')\right) \\
&\leq \pearsondistance{}\left(\canonical{\firstreward{}}(S,A,S'),\canonical{\secondreward{}}(S,A,S')\right) \\
&+ \pearsondistance{}\left(\canonical{\secondreward{}}(S,A,S'),\canonical{\thirdreward{}}(S,A,S')\right) \\
&= \canonicalizeddistance{}(\firstreward{},\secondreward{}) + \canonicalizeddistance{}(\secondreward{},\thirdreward{}),
\end{align}
since $\pearsondistance{}(X,Z) \leq \pearsondistance{}(X,Y) + \pearsondistance{}(Y,Z)$.
\end{proof}

\canonicalizeddistanceproperties*
\begin{proof}
Since $\canonicalizeddistance{}$ is defined in terms of $\pearsondistance{}$, the bounds $0 \leq \canonicalizeddistance{}(\firstreward', \secondreward')$ and $\canonicalizeddistance{}(\firstreward{}, \secondreward{}) \leq 1$ are immediate from the bounds in lemma~\ref{lemma:pearson-distance-properties}.

Since $\firstreward' \equivreward \firstreward{}$ and $\secondreward' \equivreward \secondreward{}$, we can write for $X \in \{A, B\}$:
\begin{align}
\reward^{\lambda}_X(\state, \action, \nextstate) &= \lambda_X \reward_X(\state, \action, \nextstate), \\
\reward'_X(\state, \action, \nextstate) &= \reward^{\lambda}_X(\state, \action, \nextstate) + \discount \potential_X(\nextstate) - \potential_X(\state),
\end{align}
for some scaling factor $\lambda_X > 0$ and potential function $\potential{}_X:\statespace \to \mathbb{R}$.

By proposition~\ref{prop:canonically-shaped-reward-invariant-shaping}:
\begin{equation}
\label{eq:canonicalized-distance-properties:canonical-potential-invariant}
\canonical{\reward'_X} = \canonical{\reward^{\lambda}_X}.
\end{equation}
Moreover, since $\canonical{\reward}$ is defined as an expectation over $\reward$ and expectations are linear:
\begin{equation}
\label{eq:canonicalized-distance-properties:canonical-linear}
\canonical{\reward^{\lambda}_X} = \lambda_X \canonical{\reward_X}.
\end{equation}

Unrolling the definition of $\canonicalizeddistance{}$ and applying this result gives:
\begin{align}
\canonicalizeddistance{}(\firstreward', \secondreward') &= \pearsondistance{}\left(\canonical{\firstreward'}(S,A,S'),\canonical{\secondreward'}(S,A,S')\right) \\
&= \pearsondistance{}\left(\lambda_A \canonical{\firstreward}(S,A,S'), \lambda_B \canonical{\secondreward}(S,A,S')\right) & \text{eqs.~\ref{eq:canonicalized-distance-properties:canonical-potential-invariant} and~\ref{eq:canonicalized-distance-properties:canonical-linear}} \nonumber \\
&= \pearsondistance{}\left(\canonical{\firstreward}(S,A,S'), \canonical{\secondreward}(S,A,S')\right) & \text{lemma~\ref{lemma:pearson-distance-properties}} \nonumber \\
&= \canonicalizeddistance{}(\firstreward, \secondreward). \nonumber \qedhere
\end{align}

\end{proof}

\subsubsection{\nearestpointlongnamewithabbrev{}}
\label{sec:supp:proofs:nearestpoint}

\begin{prop}
\label{prop:direct-metric}
\begin{enumerate*}[label=(\arabic*)]
\item $\directdistance{}$ is a metric in $L^p$ space, where functions $f$ and $g$ are identified if they agree almost everywhere on $\transitiondataset{}$.
\item $\directdistance{}$ is a pseudometric if functions are identified only if they agree at all points.
\end{enumerate*}
\end{prop}
\begin{proof}
\begin{enumerate*}[label=(\arabic*)]
\item $\directdistance{}$ is a metric in the $L^p$ space since $L^p$ is a norm in the $L^p$ space, and $d(x,y) = \left\lVert x - y \right\rVert$ is always a metric.
\item As $f = g$ at all points implies $f = g$ almost everywhere, certainly $\directdistance{}(\reward{}, \reward{}) = 0$.
	  Symmetry and triangle inequality do not depend on identity so still hold.
\end{enumerate*}
\end{proof}

\begin{prop}[Properties of $\nearestpoint^U$]
\label{prop:unnormalized-nearest-point-properties}
Let $\firstreward{}, \secondreward{}, \thirdreward:\statespace \times \actionspace \times \statespace \to \mathbb{R}$ be bounded reward functions, and $\lambda \geq 0$.
Then \emph{$\nearestpoint^U$}:
\begin{itemize}
	\item \textbf{Is invariant under $\equivreward$ in source}:\\ $\nearestpoint^U(\firstreward{}, \thirdreward{}) = \nearestpoint^U(\secondreward{}, \thirdreward{})$ if $\firstreward{} \equivreward{} \secondreward{}$.
	\item \textbf{Invariant under scale-preserving $\equivreward$ in target}:\\ $\nearestpoint^U(\firstreward{}, \secondreward{}) = \nearestpoint^U(\firstreward{}, \thirdreward{})$ if $\secondreward{} - \thirdreward{} \equivreward{} \zeroreward{}$.
	\item \textbf{Scalable in target}:\\ $\nearestpoint^U(\firstreward{}, \lambda \secondreward{}) = \lambda \nearestpoint^U(\firstreward{}, \secondreward{})$.
	\item \textbf{Bounded}:\\ $\nearestpoint^U(\firstreward{}, \secondreward{}) \leq \nearestpoint^U(\zeroreward{}, \secondreward{})$.
\end{itemize}
\end{prop}
\begin{proof} We will show each case in turn.

\textbf{Invariance under $\equivreward$ in source}

If $\firstreward{} \equivreward \secondreward{}$, then:
\begin{align}
\nearestpoint^U(\firstreward{}, \thirdreward{}) &\triangleq \inf_{\reward \equivreward{} \firstreward{}} \directdistance(\reward{}, \thirdreward{}) \\
&= \inf_{\reward{} \equivreward{} \secondreward{}} \directdistance(\reward{}, \thirdreward{}) \\
&\triangleq \nearestpoint^U(\secondreward{}, \thirdreward{}), \\
\end{align}
since $\reward{} \equivreward \firstreward{}$ if and only if $\reward{} \equivreward \secondreward{}$ as $\equivreward$ is an equivalence relation.

\textbf{Invariance under scale-preserving $\equivreward$ in target}

If $\secondreward{} - \thirdreward{} \equivreward \zeroreward{}$, then we can write $\secondreward{}(s,a,s') - \thirdreward{}(s,a,s') = \gamma \potential(s') - \potential(s)$ for some potential function $\potential : \statespace \to \mathbb{R}$. Define $f(\reward{})(s,a,s') = \reward{}(s,a,s') - \gamma \potential(s') + \potential(s)$. Then for any reward function $\reward{} : \statespace \times \actionspace \times \statespace \to \mathbb{R}$:
\begin{align}
\directdistance(\reward{}, \secondreward{}) &\triangleq \left(\underset{s,a,s' \sim \transitiondataset{}}{\expectation{}}\left[\left\lvert \reward{}(s,a,s') - \secondreward{}(s,a,s')\right\rvert^p\right]\right)^{1/p} \nonumber \\
&= \left(\underset{s,a,s' \sim \transitiondataset{}}{\expectation{}} \left[\left\lvert \reward{}(s,a,s') - \left(\thirdreward{}(s,a,s') + \gamma \potential(s') - \potential(s)\right)\right\rvert^p\right]\right)^{1/p} \nonumber \\
&= \left(\underset{s,a,s' \sim \transitiondataset{}}{\expectation{}} \left[\left\lvert \left(\reward{}(s,a,s') - \gamma \potential(s') + \potential(s)\right) - \thirdreward{}(s,a,s') \right\rvert^p\right]\right)^{1/p} \nonumber \\
&= \left(\underset{s,a,s' \sim \transitiondataset{}}{\expectation{}} \left[\left\lvert f(\reward{})(s,a,s') - \thirdreward{}(s,a,s') \right\rvert^p\right]\right)^{1/p} \nonumber \\
&\triangleq \directdistance(f(\reward{}), \thirdreward{}) \label{eq:divergence-invariant-target:norm-swap},
\end{align}
Crucially, note $f(\reward{})$ is a bijection on the equivalence class $[\reward{}]$. Now, substituting this into the expression for the \nearestpointabbrev{}:
\begin{align}
\nearestpoint^U(\firstreward{}, \secondreward{}) &\triangleq \inf_{\reward{} \equivreward{} \firstreward{}} \directdistance(\reward{}, \secondreward{}) \nonumber \\
&= \inf_{\reward{} \equivreward{} \firstreward{}} \directdistance(f(\reward{}), \thirdreward{}) & \text{eq.~\ref{eq:divergence-invariant-target:norm-swap}} \nonumber \\
&= \inf_{f(\reward{}) \equivreward{} \firstreward{}} \directdistance(f(R), \thirdreward{}) & f\text{ bijection on }[R] \nonumber \\
&= \inf_{\reward{} \equivreward{} \firstreward{}} \directdistance(\reward{}, \thirdreward{}) & f\text{ bijection on }[R] \nonumber \\
&\triangleq \nearestpoint^U(\firstreward{}, \thirdreward{}).
\end{align}

\textbf{Scalable in target}

First, note that $\directdistance{}$ is absolutely scalable in both arguments:
\begin{align}
\directdistance(\lambda \firstreward{}, \lambda \secondreward{}) &\triangleq \left(\underset{s,a,s' \sim \transitiondataset{}}{\expectation{}}\left[\left\lvert \lambda \firstreward{}(s,a,s') - \lambda \secondreward{}(s,a,s')\right\rvert^p\right]\right)^{1/p} \nonumber \\
&= \left(\underset{s,a,s' \sim \transitiondataset{}}{\expectation{}}\left[\left\lvert \lambda \right\rvert^{p} \left\lvert \firstreward{}(s,a,s') - \secondreward{}(s,a,s')\right\rvert^p\right]\right)^{1/p} \nonumber & \lvert \cdot \rvert \text{ absolutely scalable} \\
&= \left(\left\lvert \lambda \right\rvert^{p} \underset{s,a,s' \sim \transitiondataset{}}{\expectation{}}\left[\left\lvert \firstreward{}(s,a,s') - \secondreward{}(s,a,s')\right\rvert^p\right]\right)^{1/p} \nonumber & \expectation{} \text{ linear} \\
&= \left\lvert \lambda \right\rvert \left(\underset{s,a,s' \sim \transitiondataset{}}{\expectation{}}\left[\left\lvert \firstreward{}(s,a,s') - \secondreward{}(s,a,s')\right\rvert^p\right]\right)^{1/p} \nonumber \\
&\triangleq \left\lvert \lambda \right\rvert \directdistance(\firstreward{}, \secondreward{}). \label{eq:lp-absolutely-scalable}
\end{align}
Now, for $\lambda > 0$, applying this to $\nearestpoint^U$:
\begin{align}
\nearestpoint^U(\firstreward{}, \lambda \secondreward{}) &\triangleq \inf_{\reward{} \equivreward{} \firstreward{}} \directdistance(\reward{}, \lambda \secondreward{}) \\
&= \inf_{\reward{} \equivreward{} \firstreward{}} \directdistance(\lambda \reward{}, \lambda \secondreward{}) & \reward{} \equivreward{} \lambda \reward{} \\
&= \inf_{\reward{} \equivreward{} \firstreward{}} \lambda \directdistance(\reward{}, \secondreward{}) & \text{eq.~\ref{eq:lp-absolutely-scalable}} \\
&= \lambda \inf_{\reward{} \equivreward{} \firstreward{}} \directdistance(\reward{}, \secondreward{}) \\
&\triangleq \lambda \nearestpoint^U(\firstreward{}, \secondreward{}).
\end{align}
In the case $\lambda = 0$, then:
\begin{align}
\nearestpoint^U(\firstreward{}, \zeroreward{}) &\triangleq \inf_{\reward{} \equivreward{} \firstreward{}} \directdistance(\reward{}, \zeroreward{}) \\
&= \inf_{\reward{} \equivreward{} \firstreward{}} \directdistance\left(\frac{1}{2} \reward{}, \zeroreward{}\right) & \reward{} \equivreward{} \frac{1}{2} \reward{} \\
&= \inf_{\reward{} \equivreward{} \firstreward{}}\frac{1}{2} \directdistance(\reward{}, \zeroreward{}) \\
&= \frac{1}{2}\inf_{\reward{} \equivreward{} \firstreward{}} \directdistance(\reward{}, \zeroreward{}) \\
&= \frac{1}{2}\nearestpoint^U(\firstreward{}, \zeroreward{}).
\end{align}
Rearranging, we have:
\begin{equation}
\nearestpoint^U(\firstreward{}, \zeroreward{}) = 0.
\end{equation}

\textbf{Bounded}

Let $d \triangleq \nearestpoint(\zeroreward{}, \secondreward)$.
Then for any $\epsilon > 0$, there exists some potential function $\potential\::\:\statespace \to \mathbb{R}$ such that the \directdistancename{} distance of the potential shaping $\reward{}(\state,\action,\nextstate) \triangleq \gamma \potential(\state) - \potential(\state)$ from $\secondreward{}$ satisfies:
\begin{equation}
\label{eq:divergence-bounded:direct-first}
\directdistance(\reward{}, \secondreward) \leq d + \epsilon.
\end{equation}

Let $\lambda \in [0,1]$.
Define:
\begin{equation}
\reward{}'_{\lambda}(\state,\action,\nextstate) \triangleq \lambda \firstreward(\state,\action,\nextstate) + \reward{}(\state,\action,\nextstate).
\end{equation}
Now:
\begin{align}
\directdistance{}(\reward{}'_\lambda, \reward) &\triangleq \left(\underset{s,a,s' \sim \transitiondataset{}}{\expectation{}} \left[\left\lvert\reward'_{\lambda}(s,a,s') -  \reward{}(s,a,s')\right\rvert^p\right]\right)^{1/p} \\
&= \left(\underset{s,a,s' \sim \transitiondataset{}}{\expectation{}} \left[\left\lvert \lambda \firstreward{}(s,a,s')\right\rvert^p\right]\right)^{1/p} \\
&= \left(\lvert \lambda \rvert^p \underset{s,a,s' \sim \transitiondataset{}}{\expectation{}} \left[\left\lvert \firstreward{}(s,a,s')\right\rvert^p\right]\right)^{1/p} \\
&= \lvert \lambda \rvert \left(\underset{s,a,s' \sim \transitiondataset{}}{\expectation{}} \left[\left\lvert \firstreward{}(s,a,s')\right\rvert^p\right]\right)^{1/p} \\
&= \lvert \lambda \rvert \directdistance{}(\firstreward{}, \zeroreward{}).
\end{align}
Since $\firstreward{}$ is bounded, $\directdistance{}(\firstreward{}, \zeroreward{})$ must be finite, so:
\begin{equation}
\lim_{\lambda \to 0^{+}} \directdistance{}(\reward{}'_\lambda, \reward) = 0.
\end{equation}

It follows that for any $\epsilon > 0$ there exists some $\lambda > 0$ such that:
\begin{equation}
\label{eq:divergence-bounded:direct-second}
\directdistance(\reward{}'_{\lambda}, \reward) \leq \epsilon.
\end{equation}

Note that $\firstreward \equivreward R'_{\lambda}$ for all $\lambda > 0$. So:
\begin{align}
\nearestpoint(\firstreward, \secondreward) &\leq \directdistance(R'_{\lambda}, \secondreward) \\
&\leq \directdistance(R'_{\lambda}, R) + \directdistance(R, \secondreward) & \text{prop.~\ref{prop:direct-metric}} \\
&\leq \epsilon + (d + \epsilon) & \text{eq.~\ref{eq:divergence-bounded:direct-first} and eq.~\ref{eq:divergence-bounded:direct-second}} \\
&= d + 2\epsilon.
\end{align}

Since $\epsilon > 0$ can be made arbitrarily small, it follows:
\begin{equation}
\nearestpoint(\firstreward, \secondreward) \leq d \triangleq \nearestpoint(\zeroreward{}, \secondreward),
\end{equation}
completing the proof.
\end{proof}

\nearestpointproperties*
\begin{proof}
We will first prove $\nearestpoint{}$ is a premetric, and then prove it is invariant and bounded.

\textbf{Premetric}

First, we will show that $\nearestpoint{}$ is a premetric.

\emph{Respects identity: $\nearestpoint{}(\firstreward{}, \firstreward{}) = 0$}

If $\nearestpoint^U(\zeroreward{}, \firstreward{}) = 0$ then $\nearestpoint{}(\firstreward{}, \firstreward{}) = 0$ as required.
Suppose from now on that $\nearestpoint^U{}(\firstreward{}, \firstreward{}) \neq 0$.
It follows from prop~\ref{prop:direct-metric} that $\directdistance(\firstreward{}, \firstreward{}) = 0$.
Since $X \equivreward{} X$, $0$ is an upper bound for $\nearestpoint^U(\firstreward{}, \firstreward{})$.
By prop~\ref{prop:direct-metric} $\directdistance$ is non-negative, so this is also a lower bound for $\nearestpoint^U(\firstreward{}, \firstreward{})$.
So $\nearestpoint^U(\firstreward{}, \firstreward{}) = 0$ and:
\begin{equation}
\nearestpoint{}(\firstreward{}, \firstreward{}) = \frac{\nearestpoint^U(\firstreward{}, \firstreward{})}{\nearestpoint^U(\zeroreward{}, \firstreward{})} = \frac{0}{\nearestpoint^U(\zeroreward{}, \firstreward{})} = 0.
\end{equation}

\emph{Well-defined: $\nearestpoint{}(\firstreward{}, \secondreward{}) \geq 0$}

By prop~\ref{prop:direct-metric}, it follows that $\directdistance(R,\secondreward{}) \geq 0$ for all reward functions $R:\statespace \times \actionspace \times \statespace$.
Thus $0$ is a lower bound for $\{\directdistance(R,\secondreward{}) \mid R:\statespace \times \actionspace \times \statespace\}$, and thus certainly a lower bound for $\{\directdistance(R,Y) \mid R \equivreward{} X\}$ for any reward function $X$.
Since the infimum is the \emph{greatest} lower bound, it follows that for any reward function $X$:
\begin{equation}
\nearestpoint^U(X, \secondreward{}) \triangleq \inf_{R \equivreward{} X} \directdistance(R, \secondreward{}) \geq 0.
\end{equation}

In the case that $\nearestpoint^U(\zeroreward{}, \secondreward{}) = 0$, then $\nearestpoint(\firstreward{}, \secondreward{}) = 0$ which is non-negative.
From now on, suppose that $\nearestpoint^U(\zeroreward{}, \secondreward{}) \neq 0$.
The quotient of a non-negative value with a positive value is non-negative, so:
\begin{equation}
\nearestpoint(\firstreward{}, \secondreward{}) = \frac{\nearestpoint^U(\firstreward{}, \secondreward{})}{\nearestpoint^U(\zeroreward{}, \secondreward{})} \geq 0.
\end{equation}

\textbf{Invariant and Bounded}

Since $\secondreward' \equivreward \secondreward$, we have $\secondreward' - \lambda \secondreward \equivreward \zeroreward{}$ for some $\lambda > 0$.
By proposition~\ref{prop:unnormalized-nearest-point-properties},  $\nearestpoint^U$ is invariant under scale-preserving $\equivreward$ in target and scalable in target.
That is, for any reward $\reward{}$:
\begin{equation}
\label{eq:nearest-point:target-reward-rescale}
\nearestpoint^U(\reward{}, \secondreward') = \nearestpoint^U(\reward{}, \lambda \secondreward{}) = \lambda \nearestpoint^U(\reward{}, \secondreward{}).
\end{equation}

In particular, $\nearestpoint^U(\zeroreward{}, \secondreward{}') = \lambda \nearestpoint^U(\zeroreward{}, \secondreward{})$. As $\lambda > 0$, it follows that $\nearestpoint^U(\zeroreward{}, \secondreward{}') = 0 \iff \nearestpoint^U(\zeroreward{}, \secondreward{}) = 0$.

Suppose $\nearestpoint^U(\zeroreward{}, \secondreward{}) = 0$. Then $\nearestpoint(\reward{}, \secondreward{}) = 0 = \nearestpoint(\reward{}, \secondreward{}')$ for any reward $\reward{}$, so the result trivially holds.
From now on, suppose $\nearestpoint^U(\zeroreward{}, \secondreward{}) \neq 0$.

By proposition~\ref{prop:unnormalized-nearest-point-properties}, $\nearestpoint^U$ is invariant to $\equivreward$ in source.
That is, $\nearestpoint^U(\firstreward{}, \secondreward{}) = \nearestpoint^U(\firstreward', \secondreward{})$, so:
\begin{equation}
\nearestpoint(\firstreward', \secondreward{}) = \frac{\nearestpoint^U(\firstreward', \secondreward{})}{\nearestpoint^U(\zeroreward{}, \secondreward{})} = \frac{\nearestpoint^U(\firstreward{}, \secondreward{})}{\nearestpoint^U(\zeroreward{}, \secondreward{})} = \nearestpoint(\firstreward{}, \secondreward{}).
\end{equation}

By eq.~\eqref{eq:nearest-point:target-reward-rescale}:
\begin{equation}
\nearestpoint(\firstreward{}, \secondreward') = \frac{\lambda \nearestpoint^U(\firstreward{}, \secondreward{})}{\lambda \nearestpoint^U(\zeroreward{}, \secondreward{})} =  \frac{\nearestpoint^U(\firstreward{}, \secondreward{})}{\nearestpoint^U(\zeroreward{}, \secondreward{})} = \nearestpoint(\firstreward{}, \secondreward{}).
\end{equation}

Since $\nearestpoint$ is a premetric it is non-negative.
By the boundedness property of proposition~\ref{prop:unnormalized-nearest-point-properties}, $\nearestpoint^U(R, \secondreward{}) \leq \nearestpoint^U(\zeroreward{}, \secondreward{})$, so:
\begin{equation}
\nearestpoint(\firstreward{}, \secondreward) = \frac{\nearestpoint^U(\firstreward{}, \secondreward{})}{\nearestpoint^U(\zeroreward{}, \secondreward{})} \leq 1,
\end{equation}
which completes the proof.
\end{proof}

Note when $\directdistance$ is a metric, then $\nearestpoint(X,Y) = 0$ if and only if $X = Y$.

\begin{prop}
\label{prop:nearest-point-not-pseudometric}
$\nearestpoint{}$ is \emph{not} symmetric in the undiscounted case.
\end{prop}
\begin{proof}
We will provide a counterexample showing that $\nearestpoint{}$ is not symmetric.

Choose the state space $\statespace$ to be binary $\{0,1\}$ and the actions $\actionspace$ to be the singleton $\{0\}$.
Choose the \transitiondatasetname{} $\transitiondataset$ to be uniform on $\state \overset{0}{\rightarrow} \state$ for $\state \in \statespace$.
Take $\discount = 1$, i.e. undiscounted.
Note that as the successor state is always the same as the start state, potential shaping has no effect on $\directdistance{}$, so WLOG we will assume potential shaping is always zero.

Now, take $\firstreward{}(\state) = 2\state$ and $\secondreward{}(\state) = 1$. Take $p = 1$ for the \directdistancename{} distance.
Observe that $\directdistance{}(\zeroreward{}, \firstreward{}) = \frac{1}{2}\left(|0| + |2|\right) = 1$ and $\directdistance{}(\zeroreward{}, \secondreward{}) = \frac{1}{2}\left(|1| + |1|\right) = 1$.
Since potential shaping has no effect, $\nearestpoint^U(\zeroreward{}, \reward{}) = \directdistance{}(\zeroreward, \reward{})$ and so $\nearestpoint(\zeroreward{}, \firstreward{}) = 1$ and $\nearestpoint(\zeroreward{}, \secondreward{}) = 1$.

Now:
\begin{align}
\nearestpoint^U(\firstreward{}, \secondreward{}) &= \inf_{\lambda > 0} \directdistance{}(\lambda \firstreward{}, \secondreward{}) \\
&= \inf_{\lambda > 0} \frac{1}{2}\left(|1| + |2\lambda - 1|\right) \\
&= \frac{1}{2},
\end{align}
with the infimum attained at $\lambda = \frac{1}{2}$. But:
\begin{align}
\nearestpoint^U(\secondreward{}, \firstreward{}) &= \inf_{\lambda > 0} \directdistance{}(\lambda \secondreward{}, \firstreward{}) \\
&= \inf_{\lambda > 0} \frac{1}{2}f(\lambda) \\
&= \frac{1}{2}\inf_{\lambda > 0} f(\lambda),
\end{align}
where:
\begin{equation}
f(\lambda) = |\lambda| + |2 - \lambda|,\quad\quad \lambda > 0.
\end{equation}
Note that:
\begin{equation}
f(\lambda) = \begin{cases}
2 & \lambda \in (0, 2], \\
2\lambda - 2 & \lambda \in (2, \infty).
\end{cases}
\end{equation}
So $f(\lambda) \geq 2$ on all of its domain, thus:
\begin{equation}
\nearestpoint^U(\secondreward{}, \firstreward{}) = 1.
\end{equation}

Consequently:
\begin{equation}
\nearestpoint{}(\firstreward{}, \secondreward{}) = \frac{1}{2} \neq 1 = \nearestpoint{}(\secondreward{}, \firstreward{}),
\end{equation}
so $\nearestpoint{}$ is not symmetric.
\end{proof}

\subsection{Full Normalization Variant of \canonicalizeddistanceabbrevonly{}}
\label{sec:direct-distance-epic}

Previously, we used the Pearson distance $\pearsondistance{}$ to compare the canonicalized rewards.
Pearson distance is naturally invariant to scaling.
An alternative is to explicitly normalize the canonicalized rewards, and then compare them using any metric over functions.

\begin{defn}[Normalized Reward]
\ag{Should we also subtract mean? The zero mean does not hold when the transition dataset is not. For this would need to specialize to e.g. $L^P$ norm. Mean only right thing when $p=2$ in which case back at Pearson; for $p=1$ median minimizes distance.}
Let $\reward{}:\statespace \times \actionspace \times \statespace \to \mathbb{R}$ be a bounded reward function.
Let $\lVert\cdot\rVert$ be a norm on the vector space of reward functions over the real field.
Then the normalized $R$ is:
\begin{equation}
\normalized{\reward{}}(\state,\action,\nextstate) = \frac{\reward{}(\state,\action,\nextstate)}{\lVert\reward{}\rVert}
\end{equation}
\end{defn}

Note that $\normalized{\left(\lambda \reward{}\right)} = \normalized{\reward{}}$ for any $\lambda > 0$ as norms are absolutely homogeneous.

We say a reward is \emph{standardized} if it has been canonicalized and then normalized.

\begin{defn}[Standardized Reward]
Let $\reward{}:\statespace \times \actionspace \times \statespace \to \mathbb{R}$ be a bounded reward function.
Then the standardized $R$ is:
\begin{equation}
\reward{}^S = \normalized{\left(\canonical{\reward}\right)}.
\end{equation}
\end{defn}

Now, we can define a pseudometric based on the direct distance between the standardized rewards.

\begin{defn}[\canonicalizeddistancedirectlongnameonly{}]
Let $\transitiondataset{}$ be some \transitiondatasetname{} over transitions $\state \overset{a}{\to} \nextstate$.
Let $\statedistribution$ and $\actiondistribution$ be some distributions over states $\statespace$ and $\actionspace$ respectively.
Let $S,A,S'$ be random variables jointly sampled from $\transitiondataset{}$.
The \emph{Direct Distance Standardized Reward pseudometric} between two reward functions $\firstreward{}$ and $\secondreward{}$ is the \directdistancename{} distance between their standardized versions over $\transitiondataset{}$:
\begin{equation}
\label{eq:direct-distance-normalized}
\canonicalizeddistancedirect{}(\firstreward{}, \secondreward{}) = \frac{1}{2}\directdistance{}\left(\firstreward^S(S,A,S'), \secondreward^S(S,A,S')\right),
\end{equation}
where the normalization step, $\normalized{\reward{}}$, uses the $L^p$ norm.
\end{defn}

For brevity, we omit the proof that $\canonicalizeddistancedirect{}$ is a pseudometric, but this follows from $\directdistance{}$ being a pseudometric in a similar fashion to theorem~\ref{thm:canonicalized-distance-pseudometric}. Note it additionally is invariant to equivalence classes, similarly to \canonicalizeddistanceabbrevonly{}.

\begin{theorem}
Let $\firstreward{}$, $\firstreward{}'$, $\secondreward{}$ and $\secondreward{}'$ be reward functions mapping from transitions $\statespace \times \actionspace \times \statespace$ to real numbers $\mathbb{R}$ such that $\firstreward{} \equivreward \firstreward{}'$ and $\secondreward{} \equivreward \secondreward{}'$.
Then:
\begin{equation}
0 \leq \canonicalizeddistancedirect{}(\firstreward', \secondreward') = \canonicalizeddistancedirect{}(\firstreward, \secondreward) \leq 1.
\end{equation}
\end{theorem}
\begin{proof}
The invariance under the equivalence class follows from $\reward{}^S$ being invariant to potential shaping and scale in $\reward{}$. The non-negativity follows from $\directdistance{}$ being a pseudometric. The upper bound follows from the rewards being normalized to norm $1$ and the triangle inequality:
\begin{align}
\canonicalizeddistancedirect{}(\firstreward, \secondreward) &= \frac{1}{2}\lVert \firstreward^S - \secondreward^S \rVert \\
&\leq \frac{1}{2}\left(\lVert \firstreward^S \rVert + \lVert \secondreward^S \rVert\right) \\
&= \frac{1}{2}\left(1 + 1\right) \\
&= 1. \nonumber \qedhere
\end{align}
\end{proof}

Since both \canonicalizeddistancedirectabbrevonly{} and \canonicalizeddistanceabbrevonly{} are pseudometrics and invariant on equivalent rewards, it is interesting to consider the connection between them.
In fact, under the $L^2$ norm, then \canonicalizeddistancedirectabbrevonly{} recovers \canonicalizeddistanceabbrevonly{}.
First, we will show that canonical shaping centers the reward functions.

\begin{lemma}[The Canonically Shaped Reward is Mean Zero]
\label{lemma:canonically-shaped-mean-zero}
Let $\reward{}$ be a reward function mapping from transitions $\statespace \times \actionspace \times \statespace$ to real numbers $\mathbb{R}$. Then:
\begin{equation}
\expectation\left[\canonical{\reward{}}(S, A, S')\right] = 0.
\end{equation}
\end{lemma}
\begin{proof}
Let $X$, $U$ and $X'$ be random variables that are independent of $S$, $A$ and $S'$ but identically distributed.
\begin{align}
\mathrm{LHS} &\triangleq \expectation\left[\canonical{\reward{}}(S, A, S')\right] \\
&= \expectation\left[\reward{}(S,A,S') + \discount\reward{}(S', U, X') - \reward{}(S, U, X') - \discount\reward{}(X, U, X')\right] \\
&= \expectation\left[\reward{}(S,A,S')\right] + \discount\expectation\left[\reward{}(S', U, X')\right] - \expectation\left[\reward{}(S, U, X')\right] - \discount\expectation\left[\reward{}(X, U, X')\right] \\
&= \expectation\left[\reward{}(S,U,X')\right] + \discount\expectation\left[\reward{}(X, U, X')\right] - \expectation\left[\reward{}(S, U, X')\right] - \discount\expectation\left[\reward{}(X, U, X')\right] \\
&= 0,
\end{align}
where the penultimate step follows since $A$ is identically distributed to $U$, and $S'$ is identically distributed to $X'$ and therefore to $X$.
\end{proof}

\begin{theorem}
$\canonicalizeddistancedirect$ with $p=2$ is equivalent to $\canonicalizeddistance$.
Let $\firstreward{}$ and $\secondreward{}$ be reward functions mapping from transitions $\statespace \times \actionspace \times \statespace$ to real numbers $\mathbb{R}$. Then:
\begin{equation}
\canonicalizeddistancedirect(\firstreward, \secondreward) = \canonicalizeddistance(\firstreward, \secondreward).
\end{equation}
\end{theorem}
\begin{proof}
Recall from the proof of lemma~\ref{lemma:pearson-distance-properties} that:
\begin{align}
\pearsondistance{}(U,V) &= \frac{1}{2}\sqrt{\expectation\left[\left(Z(U) - Z(V)\right)^2\right]} \\
&= \frac{1}{2}\left\lVert Z(U) - Z(V)\right\rVert_2,
\end{align}
where $\lVert \cdot \rVert_2$ is the $L^2$ norm (treating the random variables as functions on a measure space) and $Z(U)$ is a centered (zero-mean) and rescaled (unit variance) random variable. By lemma~\ref{lemma:canonically-shaped-mean-zero}, the canonically shaped reward functions are already centered under the joint distribution $\statedistribution \times \actiondistribution \times \statedistribution$, and normalization by the $L^2$ norm also ensures they have unit variance. Consequently:
\begin{align}
\canonicalizeddistance{}(\firstreward, \secondreward) &= \pearsondistance{}\left(\canonical{\firstreward{}}(S,A,S'),\canonical{\secondreward{}}(S,A,S')\right) \\
&= \frac{1}{2}\left\lVert \normalized{\left(\canonical{\firstreward{}}(S,A,S')\right)} - \normalized{\left(\canonical{\secondreward{}}(S,A,S')\right)}\right\rVert_2 \\
&= \frac{1}{2}\left\lVert \firstreward^S(S,A,S') - \secondreward^S(S,A,S')\right\rVert_2 \\
&= \frac{1}{2}\directdistance{}\left(\firstreward^S(S,A,S'), \secondreward^S(S,A,S')\right) \\
&= \canonicalizeddistancedirect{}(\firstreward, \secondreward),
\end{align}
completing the proof.
\end{proof}

\subsection{Regret Bound}

In this section, we derive an upper bound on the regret in terms of the \canonicalizeddistanceabbrevonly{} distance.
Specifically, given two reward functions $\firstreward$ and $\secondreward$ with optimal policies $\policy_A^*$ and $\policy_B^*$, we show that the regret (under reward $\firstreward$) of using policy $\policy_B^*$ instead of a policy $\policy_A^*$ is bounded by a function of $\canonicalizeddistance(\firstreward, \secondreward)$.
First, in section~\ref{sec:regret-bound:discrete} we derive a bound for MDPs with finite state and action spaces.
In section~\ref{sec:regret-bound:lipschitz} we then present an alternative bound for MDPs with arbitrary state and action spaces and Lipschitz reward functions.
Finally, in section~\ref{sec:regret-bound:limit} we show that in both cases the regret tends to $0$ as $\canonicalizeddistance(\firstreward, \secondreward) \to 0$.

\subsubsection{Discrete MDPs}
\label{sec:regret-bound:discrete}

We start in lemma~\ref{lemma:L1-bounded-by-L2} by showing that $L^2$ distance upper bounds $L^1$ distance.
Next, in lemma~\ref{lemma:direct-distance-regret-bound} we show regret is bounded by the $L^1$ distance between reward functions using an argument similar to~\cite{singh:1994}.
Then in lemma~\ref{lemma:standardized-to-original-regret-bound} we relate regret bounds for standardized rewards $\reward^S$ to the original reward $\reward$.
Finally, in theorem~\ref{thm:epic-regret-bound-discrete} we use section~\ref{sec:direct-distance-epic} to express $\canonicalizeddistance$ in terms of the $L^2$ distance on standardized rewards, deriving a bound on regret in terms of the \canonicalizeddistanceabbrevonly{} distance.

\begin{lemma}
\label{lemma:L1-bounded-by-L2}
Let $(\Omega, \mathcal{F}, P)$ be a probability space and $f:\Omega \to \mathbb{R}$ a measurable function whose absolute value raised to the $n$-th power for $n \in \{1,2\}$ has a finite expectation. Then the $L^1$ norm of $f$ is bounded above by the $L^2$ norm:
\begin{equation}
\lVert f \rVert_1 \leq \lVert f \rVert_2.
\end{equation}
\end{lemma}
\begin{proof}
Let $X$ be a random variable sampled from $P$, and consider the variance of $f(X)$:
\begin{align}
\expectation\left[\left(|f(X)| - \expectation\left[|f(X)|\right]\right)^2\right] &= \expectation\left[|f(X)|^2 - 2|f(X)|\expectation\left[|f(X)|\right] + \expectation\left[|f(X)|\right]^2\right] \\
&= \expectation\left[|f(X)|^2\right] - 2\expectation\left[|f(X)|\right]\expectation\left[|f(X)|\right] + \expectation\left[|f(X)|\right]^2 \\
&= \expectation\left[|f(X)|^2\right] - \expectation\left[|f(X)|\right]^2 \\
&\geq 0.
\end{align}
Rearranging terms, we have
\begin{align}
\lVert f \rVert_2^2 = \expectation\left[|f(X)|^2\right] \geq \expectation\left[|f(X)|\right]^2 = \lVert f \rVert_1^2.
\end{align}
Taking the square root of both sides gives:
\begin{equation}
\lVert f \rVert_1 \leq \lVert f \rVert_2,
\end{equation}
as required.
\end{proof}

\begin{lemma}
\label{lemma:direct-distance-regret-bound}
Let $M$ be an \mdpnorliteral{} with finite state and action spaces $\statespace$ and $\actionspace$.
Let $\firstreward{}, \secondreward{}:\statespace \times \actionspace \times \statespace \to \mathbb{R}$ be rewards.
Let $\policy_A^*$ and $\policy_B^*$ be policies optimal for rewards $\firstreward$ and $\secondreward$ in $M$.
Let $\transitiondataset_{\policy}(t, \state_t, \action_t, \state_{t+1})$ denote the distribution over trajectories that policy $\policy$ induces in $M$ at time step $t$.
Let $\transitiondataset{}(\state,\action,\nextstate)$ be the (stationary) \transitiondatasetname{} over transitions $\statespace \times \actionspace \times \statespace$ used to compute $\canonicalizeddistance$.
Suppose that there exists some $K > 0$ such that $K\transitiondataset(\state_t, \action_t, \nextstate_{t+1}) \geq \transitiondataset_{\policy}(t, \state_t, \action_t, \nextstate_{t+1})$ for all time steps $t \in \mathbb{N}$, triples $\state_t,\action_t,\state_{t+1} \in \statespace \times \actionspace \times \statespace$ and policies $\pi \in \{\policy_A^*, \policy_B^*\}$.
Then the regret under $\firstreward$ from executing $\policy_B^*$ optimal for $\secondreward$ instead of $\policy_A^*$ is at most:
\begin{equation}
\policyreturn{}_{\firstreward{}}(\policy_A^*) - \policyreturn{}_{\firstreward{}}(\policy_B^*) \leq \frac{2K}{1 - \discount} \directdistance[1]{}(\firstreward{}, \secondreward{}).
\end{equation}
\end{lemma}
\begin{proof}
Noting $\policyreturn{}_{\firstreward{}}(\policy)$ is maximized when $\policy = \policy_A^*$, it is immediate that
\begin{align}
\policyreturn{}_{\firstreward{}}(\policy_A^*) - \policyreturn{}_{\firstreward{}}(\policy_B^*) &= \left\lvert \policyreturn{}_{\firstreward{}}(\policy_A^*) - \policyreturn{}_{\firstreward{}}(\policy_B^*)\right\rvert \\
&= \left\lvert \left(\policyreturn{}_{\firstreward{}}(\policy_A^*) - \policyreturn{}_{\secondreward{}}(\policy_B^*)\right) + \left(\policyreturn{}_{\secondreward{}}(\policy_B^*) -  \policyreturn{}_{\firstreward{}}(\policy_B^*)\right) \right\rvert \\
\label{eq:direct-distance-regret-bound:key-ineq}
&\leq \left\lvert \policyreturn{}_{\firstreward{}}(\policy_A^*) - \policyreturn{}_{\secondreward{}}(\policy_B^*)\right\rvert + \left\lvert \policyreturn{}_{\secondreward{}}(\policy_B^*) -  \policyreturn{}_{\firstreward{}}(\policy_B^*)\right\rvert.
\end{align}
We will show that both these terms are bounded above by $\frac{K}{1 - \discount}\directdistance[1]{}(\firstreward{}, \secondreward{})$, from which the result follows.

First, we will show that for policy $\policy \in \{\policy_A^*, \policy_B^*\}$:
\begin{equation}
\left\lvert \policyreturn{}_{\firstreward{}}(\policy) - \policyreturn{}_{\secondreward{}}(\policy)\right\rvert \leq \frac{K}{1 - \discount}\directdistance[1]{}(\firstreward{}, \secondreward{}).
\end{equation}
Let $T$ be the horizon of $M$. This may be infinite ($T = \infty$) when $\discount < 1$; note since $\statespace \times \actionspace \times \statespace$ is bounded, so are $\firstreward, \secondreward$ so the discounted infinite returns $\policyreturn{}_{\firstreward}(\policy), \policyreturn{}_{\secondreward}(\policy)$ converge (as do their differences). Writing $\tau = (\state_0, \action_0, \state_1, \action_1, \cdots)$, we have for any policy $\policy$:
\begin{align}
\Delta &\triangleq \left\lvert \policyreturn{}_{\firstreward}(\policy) - \policyreturn{}_{\secondreward}(\policy) \right\rvert \\
&= \left\lvert\expectation_{\tau \sim \transitiondataset_{\policy}}\left[\sum_{t=0}^{T} \discount^t \left(\firstreward(\state_t, \action_t, \state_{t+1}) - \secondreward(\state_t, \action_t, \state_{t+1})\right)\right]\right\rvert \\
&\leq \underset{\tau \sim \transitiondataset_{\policy}}{\expectation}\left[\sum_{t=0}^{T} \discount^t \left\lvert \firstreward(\state_t, \action_t, \state_{t+1}) - \secondreward(\state_t, \action_t, \state_{t+1})\right\rvert\right] \\
&= \sum_{t=0}^{T} \discount^t \underset{\state_t, \action_t, \state_{t+1} \sim \transitiondataset_{\policy}}{\expectation}\left[\left\lvert \firstreward(\state_t, \action_t, \state_{t+1}) - \secondreward(\state_t, \action_t, \state_{t+1})\right\rvert\right] \\
&= \sum_{t=0}^{T} \discount^t \sum_{\state_t, \action_t, \state_{t+1} \in \statespace \times \actionspace \times \statespace}  \transitiondataset_{\policy}(t, \state_t, \action_t, \state_{t+1})\left\lvert \firstreward(\state_t, \action_t, \state_{t+1}) - \secondreward(\state_t, \action_t, \state_{t+1})\right\rvert. \label{eq:direct-distance-regret-bound:bound-before-pi-assumption}
\end{align}
Let $\policy \in \{\policy_A^*, \policy_B^*\}$. By assumption, $\transitiondataset_{\policy}(t, \state_t, \action_t, \nextstate_{t+1}) \leq K\transitiondataset(\state_t, \action_t, \nextstate_{t+1})$, so:
\begin{align}
\Delta &\leq K\sum_{t=0}^{T} \discount^t \sum_{\state_t, \action_t, \state_{t+1} \in \statespace \times \actionspace \times \statespace}  \transitiondataset(\state_t, \action_t, \state_{t+1})\left\lvert \firstreward(\state_t, \action_t, \state_{t+1}) - \secondreward(\state_t, \action_t, \state_{t+1})\right\rvert \\
&= K \sum_{t=0}^{T} \discount^t \directdistance[1]{}(\firstreward{}, \secondreward{}) \\
&= \frac{K}{1 - \discount} \directdistance[1]{}(\firstreward{}, \secondreward{}),
\end{align}
as required.

In particular, substituting $\policy = \policy_B^*$ gives:
\begin{equation}
\label{eq:direct-distance-regret-bound:1st-ineq}
\left\lvert \policyreturn{}_{\secondreward{}}(\policy_B^*) - \policyreturn{}_{\firstreward{}}(\policy_B^*)\right\rvert = \left\lvert \policyreturn{}_{\firstreward{}}(\policy_B^*) - \policyreturn{}_{\secondreward{}}(\policy_B^*)\right\rvert \leq \frac{K}{1 - \discount} \directdistance[1]{}(\firstreward{}, \secondreward{}).
\end{equation}
Rearranging gives:
\begin{equation}
\policyreturn{}_{\firstreward{}}(\policy_B^*) \geq \policyreturn{}_{\secondreward{}}(\policy_B^*) - \frac{K}{1 - \discount} \directdistance[1]{}(\firstreward{}, \secondreward{}).
\end{equation}
So certainly:
\begin{equation}
\label{eq:direct-distance:symmetric-bound-one}
\policyreturn{}_{\firstreward{}}(\policy_A^*) = \max_{\policy} \policyreturn{}_{\firstreward{}}(\policy) \geq \policyreturn{}_{\secondreward{}}(\policy_B^*) - \frac{K}{1 - \discount} \directdistance[1]{}(\firstreward{}, \secondreward{}).
\end{equation}
By a symmetric argument, substituting $\policy = \policy_A^*$ gives:
\begin{equation}
\label{eq:direct-distance:symmetric-bound-two}
\policyreturn{}_{\secondreward{}}(\policy_B^*) = \max_{\policy} \policyreturn{}_{\secondreward{}}(\policy) \geq \policyreturn{}_{\firstreward{}}(\policy_A^*) - \frac{K}{1 - \discount} \directdistance[1]{}(\firstreward{}, \secondreward{}).
\end{equation}
Eqs.~\ref{eq:direct-distance:symmetric-bound-one} and~\ref{eq:direct-distance:symmetric-bound-two} respectively give $\policyreturn{}_{\secondreward{}}(\policy_B^*) - \policyreturn{}_{\firstreward{}}(\policy_A^*) \leq \frac{K}{1 - \discount}$ and $\policyreturn{}_{\firstreward{}}(\policy_A^*) - \policyreturn{}_{\secondreward{}}(\policy_B^*) \leq \frac{K}{1 - \discount}$. Combining these gives:
\begin{equation}
\label{eq:direct-distance-regret-bound:2nd-ineq}
\left\lvert \policyreturn{}_{\firstreward{}}(\policy_A^*) - \policyreturn{}_{\secondreward{}}(\policy_B^*)\right\rvert \leq \frac{K}{1 - \discount} \directdistance[1]{}(\firstreward{}, \secondreward{}).
\end{equation}
Substituting inequalities~\ref{eq:direct-distance-regret-bound:1st-ineq} and~\ref{eq:direct-distance-regret-bound:2nd-ineq} into eq.~\ref{eq:direct-distance-regret-bound:key-ineq} yields the required result. \qedhere
\end{proof}

Note that if $\transitiondataset = \uniformtransitiondataset$, uniform over $\statespace \times \actionspace \times \statespace$, then $K \leq |\statespace|^2 |\actionspace|$.

\begin{lemma}
\label{lemma:standardized-to-original-regret-bound}
Let $M$ be an \mdpnorliteral{} with state and action spaces $\statespace$ and $\actionspace$.
Let $\firstreward{}, \secondreward{}:\statespace \times \actionspace \times \statespace \to \mathbb{R}$ be bounded rewards.
Let $\policy_A^*$ and $\policy_B^*$ be policies optimal for rewards $\firstreward$ and $\secondreward$ in $M$.
Suppose the regret under the standardized reward $\firstreward^S$ from executing $\policy_B^*$ instead of $\policy_A^*$ is upper bounded by some $U \in \mathbb{R}$:
\begin{equation}
\label{eq:standardized-to-original-regret-bound:d-assumption}
\policyreturn_{\firstreward^S}(\policy_A^*) - \policyreturn_{\firstreward^S}(\policy_B^*) \leq U.
\end{equation}
Then the regret under the original reward $\firstreward$ is bounded by:
\begin{equation}
\policyreturn_{\firstreward}(\policy_A^*) - \policyreturn_{\firstreward}(\policy_B) \leq 4U\lVert\firstreward{}\rVert_2.
\end{equation}
\end{lemma}
\begin{proof}
Recall that
\begin{equation}
\reward{}^S = \frac{\canonical{\reward}}{\lVert \canonical{\reward} \rVert_2},
\end{equation}
where $\canonical{\reward}$ is simply $\reward$ shaped with some (bounded) potential $\potential$.
It follows that:
\begin{align}
\policyreturn_{\reward{}^S}(\policy) &= \frac{1}{\lVert \canonical{\reward} \rVert_2} \policyreturn_{\canonical{\reward}}(\policy) \\
&= \frac{1}{\lVert \canonical{\reward} \rVert_2}\left(\policyreturn_{\reward}(\policy) - \expectation_{\state_0 \sim \initialstatedist}\left[\potential(\state_0)\right]\right),
\end{align}
where $\state_0$ depends only on the initial state distribution $\initialstatedist$. \footnote{In the finite-horizon case, there is also a term $\discount^T \potential(s_T)$, where $s_T$ is the fixed terminal state. Since $s_T$ is fixed, it also cancels in eq.~\ref{eq:standardized-to-original-regret-bound:canceling}. This term can be neglected in the discounted infinite-horizon case as $\discount^T \potential(s_T) \to 0$ as $T \to \infty$ for any bounded $\potential$.}
Since $\state_0$ does not depend on $\policy$, the terms cancel when taking the difference in returns:
\begin{equation}
\label{eq:standardized-to-original-regret-bound:canceling}
\policyreturn_{\firstreward^S}(\policy_A^*) - \policyreturn_{\firstreward^S}(\policy_B^*) = \frac{1}{\lVert \canonical{\firstreward} \rVert_2} \left(\policyreturn_{\firstreward}(\policy_A^*) - \policyreturn_{\firstreward}(\policy_B^*)\right).
\end{equation}
Combining this with eq~\ref{eq:standardized-to-original-regret-bound:d-assumption} gives
\begin{equation}
\label{eq:standardized-to-original-regret-bound:bound-with-canon-norm}
\policyreturn_{\firstreward}(\policy_A^*) - \policyreturn_{\firstreward}(\policy_B^*) \leq U \lVert \canonical{\firstreward} \rVert_2.
\end{equation}
Finally, we will bound $\lVert \canonical{\firstreward} \rVert_2$ in terms of $\lVert \firstreward \rVert_2$, completing the proof. Recall:
\begin{equation}
\canonical{\reward{}}(\state, \action, \nextstate) = \reward{}(\state,\action,\nextstate) + \expectation\left[\discount\reward{}(\nextstate, A, S') - \reward{}(\state, A, S') - \discount\reward{}(S, A, S')\right],
\end{equation}
where $S$ and $S'$ are random variables independently sampled from $\statedistribution$ and $A$ sampled from $\actiondistribution$. By the triangle inequality on the $L^2$ norm and linearity of expectations, we have:
\begin{equation}
\lVert\canonical{\reward{}} \rVert_2 \leq \lVert \reward{} \rVert_2 + \discount \lVert f \rVert_2  + \lVert g \rVert_2 + \discount|c|,
\end{equation}
where $f(\state,\action,\nextstate) = \expectation\left[\reward{}(\nextstate, A, S')\right]$, $g(\state,\action,\nextstate) = \expectation\left[\reward{}(\state, A, S')\right]$ and $c = \expectation\left[\reward{}(S, A, S')\right]$. Letting $X'$ be a random variable sampled from $\statedistribution$ independently from $S$ and $S'$, have
\begin{align}
\lVert f \rVert_2^2 &= \expectation_{X'}\left[\expectation\left[\reward{}(X', A, S')\right]^2\right] \\
&\leq \expectation_{X'}\left[\expectation\left[\reward{}(X', A, S')^2\right]\right] \\
&= \expectation\left[\reward{}(X', A, S')^2\right] \\
&= \lVert \reward{} \rVert_2^2.
\end{align}
So $\lVert f \rVert_2 \leq \lVert \reward{} \rVert_2$ and, by an analogous argument, $\lVert g \rVert_2 \leq \lVert \reward{} \rVert_2$. Similarly
\begin{align}
|c| &= \left\lvert \expectation\left[\reward{}(S, A, S')\right] \right\rvert \\
&\leq \expectation\left[\lvert\reward{}(S, A, S')\rvert\right] \\
&= \lVert \reward{} \rVert_1 \\
&\leq \lVert \reward{} \rVert_2 & \text{lemma~\ref{lemma:L1-bounded-by-L2}}.
\end{align}
Combining these results, we have
\begin{equation}
\label{eq:standardized-to-original-regret-bound:canon-bound}
\lVert\canonical{\reward{}} \rVert_2 \leq 4\lVert \reward{} \rVert_2.
\end{equation}
Substituting eq.~\ref{eq:standardized-to-original-regret-bound:canon-bound} into eq.~\ref{eq:standardized-to-original-regret-bound:bound-with-canon-norm} yields:
\begin{equation}
\policyreturn_{\firstreward}(\policy_A^*) - \policyreturn_{\firstreward}(\policy_B^*) \leq 4U\lVert\firstreward{}\rVert_2,
\end{equation}
as required.
\end{proof}

\epicregretbounddiscrete*
\begin{proof}
Recall from section~\ref{sec:direct-distance-epic} that:
\begin{equation}
\canonicalizeddistance{}(\firstreward{}, \secondreward{}) = \frac{1}{2}\left\lVert \firstreward^S(S,A,S') - \secondreward^S(S,A,S')\right\rVert_2.
\end{equation}
Applying lemma~\ref{lemma:L1-bounded-by-L2} we obtain:
\begin{equation}
\label{eq:epic-regret-bound-discrete:ell1-vs-epic}
\directdistance[1]{}(\firstreward^S{}, \secondreward^S{}) = \left\lVert \firstreward^S(S,A,S') - \secondreward^S(S,A,S')\right\rVert_1 \leq 2\canonicalizeddistance{}(\firstreward{}, \secondreward{}).
\end{equation}

Note that $\policy_A^*$ is optimal for $\firstreward^S$ and $\policy_B^*$ is optimal for $\secondreward^S$ since the set of optimal policies for $\reward^S$ is the same as for $\reward$. Applying lemma~\ref{lemma:direct-distance-regret-bound} and eq.~\ref{eq:epic-regret-bound-discrete:ell1-vs-epic} gives
\begin{equation}
\label{eq:epic-regret-bound-discrete:policy-return-1st}
\policyreturn_{\firstreward^S}(\policy_A^*) - \policyreturn_{\firstreward^S}(\policy_B^*) \leq \frac{2K}{1 - \discount}\directdistance[1]{}(\firstreward^S{}, \secondreward^S{}) \leq \frac{4K}{1 - \discount}\canonicalizeddistance{}(\firstreward{}, \secondreward{}).
\end{equation}
Since $\statespace \times \actionspace \times \statespace$ is bounded, $\firstreward$ and $\secondreward$ must be bounded, so we can apply lemma~\ref{lemma:standardized-to-original-regret-bound}, giving:
\begin{equation}
\policyreturn_{\firstreward}(\policy_A^*) - \policyreturn_{\firstreward}(\policy_B^*) \leq \frac{16K\lVert \firstreward{} \rVert_2}{1 - \discount} \canonicalizeddistance{}(\firstreward{}, \secondreward{}),
\end{equation}
completing the proof.
\end{proof}

\subsection{Lipschitz Reward Functions}
\label{sec:regret-bound:lipschitz}

In this section, we generalize the previous results to MDPs with continuous state and action spaces.
The challenge is that even though the spaces may be continuous, the distribution $\transitiondataset_{\policy^*}$ induced by an optimal policy $\policy^*$ may only have support on some measure zero set of transitions $B$.
However, the expectation over a continuous distribution $\transitiondataset$ is unaffected by the reward at any measure zero subset of points.
Accordingly, the reward can be varied \emph{arbitrarily} on transitions $B$ -- causing arbitrarily small or large regret -- while leaving the EPIC distance fixed.
To rule out this pathological case, we assume the rewards are Lipschitz smooth.
This guarantees that if the expected difference between rewards is small on a given region, then all points in this region have bounded reward difference.

We start by defining a relaxation of the Wasserstein distance $\relaxedwassersteindistance{}$ in definition~\ref{defn:relaxed-wasserstein-distance}.
In lemma~\ref{lemma:lipschitz-wasserstein-bound} we then bound the expected value under distribution $\mu$ in terms of the expected value under alternative distribution $\nu$ plus $\relaxedwassersteindistance{}(\mu, \nu)$.
Next, in lemma~\ref{lemma:direct-distance-regret-bound-lipschitz} we bound the regret in terms of the $L^1$ distance between the rewards plus $\relaxedwassersteindistance{}$; this is analogous to lemma~\ref{lemma:direct-distance-regret-bound} in the finite case.
Finally, in theorem~\ref{thm:epic-regret-bound-lipschitz} we use the previous results to bound the regret in terms of the \canonicalizeddistanceabbrevonly{} distance plus $\relaxedwassersteindistance{}$.

\begin{defn}
\label{defn:relaxed-wasserstein-distance}
Let $S$ be some set and let $\mu, \nu$ be probability measures on $S$ with finite first moment.
We define the \emph{relaxed Wasserstein distance} between $\mu$ and $\nu$ by:
\begin{equation}
\relaxedwassersteindistance{}(\mu, \nu) \triangleq \inf_{p \in \Gamma_{\wassersteinrelaxationfactor}(\mu, \nu)} \int \left\Vert x - y\right\Vert dp(x,y),
\end{equation}
where $\Gamma_{\wassersteinrelaxationfactor}(\mu, \nu)$ is the set of probability measures on $S \times S$ satisfying for all $x, y \in S$:
\begin{align}
\label{eq:lipschitz-wasserstein-bound:x-invariant}
\int_S p(x,y) dy &= \mu(x), \\
\label{eq:lipschitz-wasserstein-bound:y-invariant}
\int_S p(x,y) dx &\leq \wassersteinrelaxationfactor{} \nu(y).
\end{align}
\end{defn}

Note that $\relaxedwassersteindistance[1]$ is equal to the (unrelaxed) Wasserstein distance (in the $\ell_1$ norm).

\begin{lemma}
\label{lemma:lipschitz-wasserstein-bound}
Let $S$ be some set and let $\mu, \nu$ be probability measures on $S$.
Let $f\::\:S \to \mathbb{R}$ be an $L$-Lipschitz function on the $\ell_1$ norm $\left\Vert \cdot \right\Vert_1$.
Then, for any $\wassersteinrelaxationfactor{} \geq 1$:
\begin{equation}
\expectation_{X \sim \mu}\left[|f(X)|\right] \leq \wassersteinrelaxationfactor{} \expectation_{Y \sim \nu}\left[|f(Y)|\right] + L \relaxedwassersteindistance{}(\mu, \nu).
\end{equation}
\end{lemma}
\begin{proof}
Let $p \in \Gamma_{\wassersteinrelaxationfactor}(\mu, \nu)$. Then:
\begin{align}
\expectation_{X \sim \mu}\left[|f(X)|\right] &\triangleq \int |f(x)| d\mu(x) & \text{definition of }\expectation \\
&= \int |f(x)| dp(x, y) & \mu\text{ is a marginal of }p \\
&\leq \int |f(y)| + L\left\Vert x - y\right\Vert dp(x, y) & f\ L\text{-Lipschitz} \\
&= \int |f(y)| dp(x,y) + L \int \left\Vert x - y\right\Vert dp(x, y) \\
&= \int |f(y)| \int p(x,y) dx dy + L \int \left\Vert x - y\right\Vert dp(x, y) \\
&\leq \int |f(y)| \wassersteinrelaxationfactor{}\nu(y) dy  + L \int \left\Vert x - y\right\Vert dp(x, y) & \text{eq.~\ref{eq:lipschitz-wasserstein-bound:y-invariant}} \\
&= \wassersteinrelaxationfactor{}\expectation_{Y \sim \nu}\left[|f(Y)|\right] + L \int \left\Vert x - y\right\Vert dp(x, y) & \text{definition of }\expectation.
\end{align}
Since this holds for all choices of $p$, we can take the infimum of both sides, giving:
\begin{align}
\expectation_{X \sim \mu}\left[|f(X)|\right] &\leq \wassersteinrelaxationfactor{}\expectation_{Y \sim \nu}\left[|f(Y)|\right] + L \inf_{p \in \Gamma_{\wassersteinrelaxationfactor{}}(\mu, \nu)} \int \left\Vert x - y\right\Vert dp(x, y) \\
&= \wassersteinrelaxationfactor{}\expectation_{Y \sim \nu}\left[|f(Y)|\right] + L \relaxedwassersteindistance{}(\mu, \nu). \qedhere
\end{align}
\end{proof}

\begin{lemma}
\label{lemma:direct-distance-regret-bound-lipschitz}
Let $M$ be an \mdpnorliteral{} with state and action spaces $\statespace$ and $\actionspace$.
Let $\firstreward{}, \secondreward{}:\statespace \times \actionspace \times \statespace \to \mathbb{R}$ be $L$-Lipschitz, bounded rewards on the $\ell_1$ norm $\left\Vert \cdot \right\Vert_1$.
Let $\policy_A^*$ and $\policy_B^*$ be policies optimal for rewards $\firstreward$ and $\secondreward$ in $M$.
Let $\transitiondataset_{\policy,t}(\state_t, \action_t, \state_{t+1})$ denote the distribution over trajectories that policy $\policy$ induces in $M$ at time step $t$.
Let $\transitiondataset{}(\state,\action,\nextstate)$ be the (stationary) \transitiondatasetname{} over transitions $\statespace \times \actionspace \times \statespace$ used to compute $\canonicalizeddistance$.
Let $\wassersteinrelaxationfactor \geq 1$, and let $B_{\wassersteinrelaxationfactor}(t) = \max_{\policy \in \{\policy_A^*, \policy_B^*\}} \relaxedwassersteindistance{}\left(\transitiondataset_{\policy, t}, \transitiondataset\right)$.
Then the regret under $\firstreward$ from executing $\policy_B^*$ optimal for $\secondreward$ instead of $\policy_A^*$ is at most:
\begin{equation}
\policyreturn{}_{\firstreward{}}(\policy_A^*) - \policyreturn{}_{\firstreward{}}(\policy_B^*) \leq \frac{2\wassersteinrelaxationfactor{}}{1 - \discount} \directdistance[1]{}(\firstreward{}, \secondreward{}) + 4L \sum_{t=0}^{\infty} \discount^t B_{\wassersteinrelaxationfactor{}}(t).
\end{equation}
\end{lemma}
\begin{proof}
By the same argument as lemma~\ref{lemma:direct-distance-regret-bound} up to eq.~\ref{eq:direct-distance-regret-bound:bound-before-pi-assumption}, we have for any policy $\policy$:
\begin{equation}
\left\lvert \policyreturn{}_{\firstreward}(\policy) - \policyreturn{}_{\secondreward}(\policy) \right\rvert \leq \sum_{t=0}^{\infty} \discount^t \directdistance[1][\transitiondataset_{\policy, t}](\firstreward, \secondreward).
\end{equation}
Let $f(\state, \action, \nextstate) = \firstreward{}(\state, \action, \nextstate) - \secondreward{}(\state, \action, \nextstate)$, and note $f$ is $2L$-Lipschitz and bounded since $\firstreward$ and $\secondreward$ are both $L$-Lipschitz and bounded.
Now, by lemma~\ref{lemma:lipschitz-wasserstein-bound}, letting $\mu = \transitiondataset_{\policy, t}$ and $\nu = \transitiondataset$, we have:
\begin{equation}
\directdistance[1][\transitiondataset_{\policy, t}](\firstreward, \secondreward) \leq \wassersteinrelaxationfactor{} \directdistance[1]{}(\firstreward, \secondreward) + 2L\relaxedwassersteindistance{}(\transitiondataset_{\policy, t}, \transitiondataset).
\end{equation}
So, for $\pi \in \{\policy_A^*, \policy_B^*\}$, it follows that
\begin{equation}
\left\lvert \policyreturn{}_{\firstreward}(\policy) - \policyreturn{}_{\secondreward}(\policy) \right\rvert \leq \frac{\wassersteinrelaxationfactor{}}{1 - \gamma} \directdistance[1]{}(\firstreward, \secondreward) + 2L \sum_{t=0}^{\infty} \discount^t B_{\wassersteinrelaxationfactor{}}(t).
\end{equation}
By the same argument as for eq.~\ref{eq:direct-distance-regret-bound:1st-ineq} to~\ref{eq:direct-distance-regret-bound:2nd-ineq} in lemma~\ref{lemma:direct-distance-regret-bound}, it follows that
\begin{equation}
\policyreturn{}_{\firstreward{}}(\policy_A^*) - \policyreturn{}_{\firstreward{}}(\policy_B^*) \leq \frac{2\wassersteinrelaxationfactor{}}{1 - \gamma}	\directdistance[1]{}(\firstreward, \secondreward) + 4L \sum_{t=0}^{\infty} \discount^t B_{\wassersteinrelaxationfactor{}}(t),
\end{equation}
completing the proof.
\end{proof}

\begin{theorem}
\label{thm:epic-regret-bound-lipschitz}
Let $M$ be an \mdpnorliteral{} with state and action spaces $\statespace$ and $\actionspace$.
Let $\firstreward{}, \secondreward{}:\statespace \times \actionspace \times \statespace \to \mathbb{R}$ be bounded, $L$-Lipschitz rewards on the $\ell_1$ norm $\left\Vert \cdot \right\Vert_1$.
Let $\policy_A^*$ and $\policy_B^*$ be policies optimal for rewards $\firstreward$ and $\secondreward$ in $M$.
Let $\transitiondataset_{\policy}(t, \state_t, \action_t, \state_{t+1})$ denote the distribution over trajectories that policy $\policy$ induces in $M$ at time step $t$.
Let $\transitiondataset{}(\state,\action,\nextstate)$ be the (stationary) \transitiondatasetname{} over transitions $\statespace \times \actionspace \times \statespace$ used to compute $\canonicalizeddistance$.
Let $\wassersteinrelaxationfactor{} \geq 1$, and let $B_{\wassersteinrelaxationfactor{}}(t) = \max_{\policy \in {\policy_A^*, \policy_B^*}} \relaxedwassersteindistance{}\left(\transitiondataset_{\policy, t}, \transitiondataset\right)$.
Then the regret under $\firstreward$ from executing $\policy_B^*$ optimal for $\secondreward$ instead of $\policy_A^*$ is at most:
\begin{equation}
\policyreturn{}_{\firstreward{}}(\policy_A^*) - \policyreturn{}_{\firstreward{}}(\policy_B^*) \leq 16\left\Vert \firstreward \right\Vert_2\left(\frac{\wassersteinrelaxationfactor{}}{1 - \discount} \canonicalizeddistance{}(\firstreward{}, \secondreward{}) + L \sum_{t=0}^{\infty} \discount^t B_{\wassersteinrelaxationfactor{}}(t)\right).
\end{equation}
\end{theorem}
\begin{proof}
The proof for theorem~\ref{thm:epic-regret-bound-discrete} holds in the general setting up to eq.~\ref{eq:epic-regret-bound-discrete:ell1-vs-epic}.
Applying lemma~\ref{lemma:direct-distance-regret-bound-lipschitz} to eq.~\ref{eq:epic-regret-bound-discrete:ell1-vs-epic} gives
\begin{align}
\label{eq:epic-regret-bound-lipschitz:policy-return-1st}
\policyreturn_{\firstreward^S}(\policy_A^*) - \policyreturn_{\firstreward^S}(\policy_B^*) &\leq \frac{2\wassersteinrelaxationfactor{}}{1 - \discount}\directdistance[1]{}(\firstreward^S{}, \secondreward^S{}) + 4L \sum_{t=0}^{\infty} \discount^t B_{\wassersteinrelaxationfactor{}}(t) \\
&\leq \frac{4\wassersteinrelaxationfactor{}}{1 - \discount}\canonicalizeddistance{}(\firstreward{}, \secondreward{}) + 4L \sum_{t=0}^{\infty} \discount^t B_{\wassersteinrelaxationfactor{}}(t).
\end{align}
Applying lemma~\ref{lemma:standardized-to-original-regret-bound} yields
\begin{equation}
\policyreturn{}_{\firstreward{}}(\policy_A^*) - \policyreturn{}_{\firstreward{}}(\policy_B^*) \leq 16\left\Vert \firstreward \right\Vert_2\left(\frac{\wassersteinrelaxationfactor{}}{1 - \discount} \canonicalizeddistance{}(\firstreward{}, \secondreward{}) + L \sum_{t=0}^{\infty} \discount^t B_{\wassersteinrelaxationfactor{}}(t)\right),
\end{equation}
as required.
\end{proof}

\subsection{Limiting Behavior of Regret}
\label{sec:regret-bound:limit}

The regret bound for finite MDPs, Theorem~\ref{thm:epic-regret-bound-discrete}, directly implies that, as \canonicalizeddistanceabbrevonly{} distance tends to $0$, the regret also tends to $0$.
By contrast, our regret bound in theorem~\ref{thm:epic-regret-bound-lipschitz} for (possibly continuous) MDPs with Lipschitz reward functions includes the relaxed Wasserstein distance $\relaxedwassersteindistance{}$ as an additive term.
At first glance, it might therefore appear possible for the regret to be positive even with a zero \canonicalizeddistanceabbrevonly{} distance.
However, in this section we will show that in fact the regret tends to $0$ as $\canonicalizeddistance(\firstreward, \secondreward) \to 0$ in the Lipschitz case as well as the finite case.

We show in lemma~\ref{lemma:lipschitz-limit} that if the expectation of a non-negative function over a totally bounded measurable metric space $M$ tends to zero under one distribution with adequate support, then it also tends to zero under all other distributions.
For example, taking $M$ to be a hypercube in Euclidean space with the Lebesque measure satisfies these assumptions.
We conclude in theorem~\ref{thm:epic-regret-bound-lipschitz-limit} by showing the regret tends to $0$ as the \canonicalizeddistanceabbrevonly{} distance tends to $0$.

\begin{lemma}
\label{lemma:lipschitz-limit}
Let $M = (S, d)$ be a totally bounded metric space, where $d(x, y) = \left\Vert x - y\right\Vert$.
Let $(S, A, \mu)$ be a measure space on $S$ with the Borel $\sigma$-algebra $A$ and measure $\mu$.
Let $p, q \in \Delta(S)$ be probability density functions on $S$.
Let $\delta > 0$ such that $p(s) \geq \delta$ for all $s \in S$.
Let $f_n:S \to \mathbb{R}$ be a sequence of $L$-Lipschitz functions on norm $\left\Vert \cdot \right\Vert$.
Suppose $\lim_{n \to \infty} \expectation_{X \sim p}\left[|f_n(X)|\right] = 0$.
Then $\lim_{n \to \infty} \expectation_{Y \sim q}\left[|f_n(Y)|\right] = 0$.
\end{lemma}
\begin{proof}
Since $M$ is totally bounded, for each $r > 0$ there exists a finite collection of open balls in $S$ of radius $r$ whose union contains $M$.
Let $B_r(c) = \{s \in S \mid \left\Vert s - c\right\Vert < r\}$, the open ball of radius $r$ centered at $c$.
Let $C(r)$ denote some finite collection of $Q(r)$ open balls:
\begin{equation}
C(r) = \left\{B_r(c_{r,n}) \mid n \in \{1,\cdots,Q(r)\}\right\},
\end{equation}
such that $\bigcup_{B \in C(r)} B = S$.

It is possible for some balls $B_r(c_n)$ to have measure zero, $\mu(B_r(c_n)) = 0$, such as if $S$ contains an isolated point $c_n$.
Define $P(r)$ to be the subset of $C(r)$ with positive measure:
\begin{equation}
P(r) = \left\{B \in C(r) \mid \mu(B) > 0\right\},
\end{equation}
and let $p_{r,1},\cdots,p_{r,Q'(r)}$ denote the centers of the balls in $P$.
Since $P(r)$ is a finite collection, it must have a minimum measure:
\begin{equation}
\alpha(r) = \min_{B \in P} \mu(B).
\end{equation}
Moreover, by construction of $P$, $\alpha(r) > 0$.

Let $S'(r)$ be the union only over balls of positive measure:
\begin{equation}
S'(r) = \bigcup_{B \in P(r)} B.
\end{equation}
Now, let $D(r) = S \setminus S'(r)$, comprising the (finite number of) measure zero balls in $C(r)$.
Since measures are countably additive, it follows that $D(r)$ is itself measure zero: $\mu\left(D(r)\right) = 0$.
Consequently:
\begin{equation}
\label{eq:lipschitz-limit:integration-equal}
\int_S g(s) d\mu = \int_{S'(r)} g(s) d\mu,
\end{equation}
for any measurable function $g:S \to \mathbb{R}$.

Since $\lim_{n \to \infty} \expectation_{X \sim p}\left[|f_n(X)|\right] = 0$, for all $r > 0$ there exists some $N_r \in \mathbb{N}$ such that for all $n \geq N_r$:
\begin{equation}
\label{eq:lipschitz-limit:expectation-x-lower-bound}
\expectation_{X \sim p}\left[|f_n(X)|\right] < \delta Lr \alpha(r).
\end{equation}

By Lipschitz continuity, for any $s, s' \in S$:
\begin{equation}
|f_n(s')| \geq |f_n(s)| - L\left\Vert s' - s\right\Vert.
\end{equation}
In particular, since any point $s \in S'(r)$ is at most $r$ distance from some ball center $p_{r,i}$, then $|f_n(p_{r,i})| \geq |f_n(s)| - Lr$.
So if there exists $s \in S'(r)$ such that $|f_n(s)| \geq 3Lr$, then there must exist a ball center $p_{r,i}$ with $|f_n(p_{r,i})| \geq 2Lr$.
Then for any point $s' \in B_r(p_{r,i})$:
\begin{equation}
\label{eq:lipschitz-limit:lower-bound-on-ball}
|f_n(s')| \geq |f_n(p_{r,i})| - Lr \geq Lr.
\end{equation}
Now, we have:
\begin{align}
\expectation_{X \sim p}\left[|f_n(X)|\right] &\triangleq \int_{s \in S} |f_n(s)| p(s) d\mu(s) \\
&= \int_{s \in S'(r)} |f_n(s)| p(s) d\mu(s) & \text{eq.~\ref{eq:lipschitz-limit:integration-equal}} \\
&\geq \int_{s \in B_{r}(p_{r,i})} |f_n(s)| p(s) d\mu(s) & \text{non-negativity of }|f_n(s)| \\
&\geq \delta \int_{s \in B_{r}(p_{r,i})} |f_n(s)| d\mu(s) & p(s) \geq \delta \\
&\geq \delta \cdot Lr \int_{s \in B_{r}(p_{r,i})} 1 d\mu(s) & \text{eq.~\ref{eq:lipschitz-limit:lower-bound-on-ball}} \\
&= \delta Lr \mu(B_{r}(p_{r,i})) & \text{integrating w.r.t. }\mu \\
&\geq \delta Lr \alpha(r) & \alpha(r)\text{ minimum of }\mu(B_{r}(p_{r,i})). \label{eq:lipschitz-limit:expectation-y-bound}
\end{align}
But this contradicts eq.~\ref{eq:lipschitz-limit:expectation-x-lower-bound}, and so can only hold if $n < N_r$.
It follows that for all $n \geq N_r$ and $s \in S'(r)$, we have $|f_n(s)| < 3Lr$, and so in particular:
\begin{equation}
\expectation_{Y \sim q}\left[|f_n(Y)|\right] < 3Lr.
\end{equation}

Let $\epsilon > 0$. Choose $r = \frac{\epsilon}{3L}$. Then for all $n \geq N_r$, $\expectation_{Y \sim q}\left[|f_n(Y)|\right] < \epsilon$.
It follows that:
\begin{equation}
\lim_{n \to \infty} \expectation_{Y \sim q}\left[|f_n(Y)|\right] = 0,
\end{equation}
completing the proof.
\end{proof}

\begin{theorem}
\label{thm:epic-regret-bound-lipschitz-limit}
Let $M$ be an \mdpnorliteral{} with state and action spaces $\statespace$ and $\actionspace$.
Let $\firstreward{}, \secondreward{}:\statespace \times \actionspace \times \statespace \to \mathbb{R}$ be bounded rewards on some norm $\left\Vert \cdot \right\Vert$ on $\statespace \times \actionspace \times \statespace$.
Let $\policy_A^*$ and $\policy_B^*$ be policies optimal for rewards $\firstreward$ and $\secondreward$ in $M$.
Let $\transitiondataset_{\policy}(t, \state_t, \action_t, \state_{t+1})$ denote the distribution over trajectories that policy $\policy$ induces in $M$ at time step $t$.
Let $\transitiondataset{}(\state,\action,\nextstate)$ be the (stationary) \transitiondatasetname{} over transitions $\statespace \times \actionspace \times \statespace$ used to compute $\canonicalizeddistance$.

Suppose that either:
\begin{enumerate}
	\item \emph{Discrete}: $\statespace$ and $\actionspace$ are \emph{discrete}. Moreover, suppose that there exists some $K > 0$ such that $K\transitiondataset(\state_t, \action_t, \nextstate_{t+1}) \geq \transitiondataset_{\policy}(t, \state_t, \action_t, \nextstate_{t+1})$ for all time steps $t \in \mathbb{N}$, triples $\state_t,\action_t,\state_{t+1} \in \statespace \times \actionspace \times \statespace$ and policies $\pi \in \{\policy_A^*, \policy_B^*\}$.
	\item \emph{Lipschitz}: $(\statespace \times \actionspace \times \statespace,d)$ is a totally bounded measurable metric space where $d(x,y) = \left\Vert x - y\right\Vert$. Moreover, $\firstreward{}$ and $\secondreward{}$ are \emph{$L$-Lipschitz} on $\left\Vert \cdot \right\Vert$. Furthermore, suppose there exists some $\delta > 0$ such that $\transitiondataset(\state,\action,\nextstate) \geq \delta$ for all $\state, \action, \nextstate \in \statespace \times \actionspace \times \statespace$, and that $\transitiondataset_{\policy}(t, \state_t, \action_t, \state_{t+1})$ is a non-degenerate probability density function (i.e. no single point has positive measure).
\end{enumerate}
Then as $\canonicalizeddistance{}(\firstreward{}, \secondreward{}) \to 0$, $\policyreturn{}_{\firstreward{}}(\policy_A^*) - \policyreturn{}_{\firstreward{}}(\policy_B^*) \to 0$.
\end{theorem}
\begin{proof}
In case \emph{(1) Discrete}, by theorem~\ref{thm:epic-regret-bound-discrete}:
\begin{equation}
\policyreturn{}_{\firstreward{}}(\policy_A^*) - \policyreturn{}_{\firstreward{}}(\policy_B^*) \leq \frac{16K\lVert\firstreward{}\rVert_2}{1 - \discount} \canonicalizeddistance{}(\firstreward{}, \secondreward{}).
\end{equation}
Moreover, by optimality of $\policy_A^*$ we have $0 \leq \policyreturn{}_{\firstreward{}}(\policy_A^*) - \policyreturn{}_{\firstreward{}}(\policy_B^*)$. So by the squeeze theorem, as $\canonicalizeddistance{}(\firstreward{}, \secondreward{}) \to 0$, $\policyreturn{}_{\firstreward{}}(\policy_A^*) - \policyreturn{}_{\firstreward{}}(\policy_B^*) \to 0$.

From now on, suppose we are in case \emph{(2) Lipschitz}.
By the same argument as lemma~\ref{lemma:direct-distance-regret-bound} up to eq.~\ref{eq:direct-distance-regret-bound:bound-before-pi-assumption}, we have for any policy $\policy$:
\begin{equation}
\left\lvert \policyreturn{}_{\firstreward^S}(\policy) - \policyreturn{}_{\secondreward^S}(\policy) \right\rvert \leq \sum_{t=0}^{\infty} \discount^t \directdistance[1][\transitiondataset_{\policy, t}](\firstreward^S, \secondreward^S).
\end{equation}
Applying lemma~\ref{lemma:standardized-to-original-regret-bound} we have:
\begin{equation}
\left\lvert \policyreturn{}_{\firstreward}(\policy) - \policyreturn{}_{\secondreward}(\policy) \right\rvert \leq 4\left\Vert \firstreward \right\Vert_2 \sum_{t=0}^{\infty} \discount^t \directdistance[1][\transitiondataset_{\policy, t}](\firstreward^S, \secondreward^S).
\end{equation}
By equation~\ref{eq:epic-regret-bound-discrete:ell1-vs-epic}, we know that $\directdistance[1](\firstreward^S{}, \secondreward^S{}) \to 0$ as $\canonicalizeddistance{}(\firstreward{}, \secondreward{}) \to 0$.
By lemma~\ref{lemma:lipschitz-limit}, we know that $\directdistance[1][\transitiondataset_{\policy, t}]{}(\firstreward^S, \secondreward^S) \to 0$ as $\directdistance[1][\transitiondataset{}](\firstreward^S{}, \secondreward^S{}) \to 0$.
So we can conclude that as $\canonicalizeddistance{}(\firstreward{}, \secondreward{}) \to 0$:
\begin{equation}
\left\lvert \policyreturn{}_{\firstreward}(\policy) - \policyreturn{}_{\secondreward}(\policy) \right\rvert \to 0,
\end{equation}
as required.
\end{proof}

\end{document}